\documentclass{article} %

\usepackage[margin=1in]{geometry}
\usepackage{times}
\usepackage{natbib}
\setcitestyle{authoryear,round,citesep={;},aysep={,},yysep={;}}
\usepackage{mystyle}
\usepackage[belowskip=0pt]{caption}
\usepackage[utf8]{inputenc} %
\usepackage[T1]{fontenc}    %
\usepackage{xurl}            %

\usepackage{hyperref}
\usepackage{booktabs}       %

\usepackage{amsfonts}       %
\usepackage{nicefrac}       %
\usepackage{xcolor}         %
\usepackage{comment}
\usepackage{subcaption}
\allowdisplaybreaks

\usepackage{doi}

\newcommand{\human}{\text{\normalfont\scshape h}}
\newcommand{\llm}{\text{\normalfont\scshape l}}

\newcommand{\df}{\mathrm{df}}

\newcommand{\largecdot}{\text{\raisebox{-0.5ex}{\scalebox{1.5}{$\cdot$}}}}

\newcommand{\row}{\operatorname{row}}
\newcommand{\col}{\operatorname{col}}
\newcommand{\logit}{\operatorname{logit}}

\newcommand{\methHuman}{\textnormal{Human}}
\newcommand{\methPooled}{\textnormal{Pooled}}
\newcommand{\methConsensus}{\textnormal{DIAL-Anc}}
\newcommand{\methAda}{\textnormal{DIAL-Ada}}
\newcommand{\methDIALmu}{\methAda}   %
\newcommand{\methDIALW}{\methAda\ensuremath{\,(W)}}
\newcommand{\methDIALnoPos}{\textnormal{DIAL-noPos}}
\newcommand{\methAtC}{\textnormal{AtC}}

\newcommand{\titletext}{
DIAL: Position-Debiased LLM Judges with\\ Adaptive Human Preference Calibration
}

\title{\titletext}

\author{
Zesheng Cai\textsuperscript{1} \quad Yingqi Fan\textsuperscript{1} \quad Sichang Chen\textsuperscript{2} \quad Jin-Hong Du\textsuperscript{1}\thanks{Corresponding author: \texttt{jinhongd@hku.hk}. Code and data are available at \url{https://github.com/JinHongDu-Lab/DIAL}.}\\[1ex]
\textsuperscript{1}The University of Hong Kong \qquad \textsuperscript{2}Sun Yat-sen University
}
\date{}

\hypersetup{
  pdftitle={DIAL: Position-Debiased LLM Judges with Adaptive Human Preference Calibration},
  pdfauthor={Zesheng Cai, Yingqi Fan, Sichang Chen, Jin-Hong Du}
}

\begin{document}

\maketitle

\begin{abstract}
Large language models (LLMs) as a judge enable scalable evaluation, but their judgments can be sensitive to response order and, even after removing such position effects, can still diverge systematically from human preferences.
We introduce DIAL, a unified framework that combines abundant LLM comparisons with limited human comparisons to separate judge-specific position effects, learn shared structure in position-debiased LLM preferences, and adaptively calibrate that structure toward the human preference target.
Theoretically, we study three aspects of DIAL: (i) identification of latent LLM preferences, position effects, and human calibration; (ii) adaptive estimation that balances LLM anchoring against limited human evidence; and (iii) fixed-weight uncertainty quantification for the calibrated human preference.
Empirically, we evaluate position debiasing and human alignment separately in controlled simulations and on three human-preference benchmarks, showing that DIAL remains robust to unbalanced response order, achieves strong human-aligned rankings with limited labels, and adapts toward human evidence when LLM information is imperfect.
Our real-data study collects over $410$K judgments from $21$ LLM judges in both display orders, providing a resource for future studies of LLM-judge bias, heterogeneity, and human alignment.
\end{abstract}

\section{Introduction}
\label{sec:introduction}

LLMs are increasingly used as scalable judges of model-generated responses, providing a practical complement to costly human evaluation \citep{zheng2023mtbench}.
However, their pairwise judgments can be sensitive to response display order, so observed preferences may partly reflect presentation instead of response quality \citep{wang2024fair,shi2025judging}.
Moreover, LLM and human evaluations can differ systematically \citep{chen2024humans,gao2025reevaluating}, so removing position effects alone does not make an LLM judge human-aligned.
Reliable human-targeted LLM evaluation must therefore address two distinct but coupled challenges: separating systematic position effects from underlying LLM preferences, and accounting for discrepancies between LLM and human preferences when human judgments are limited (\Cref{fig:overview}(a)).
\Cref{fig:overview}(b) illustrates both effects empirically across LLM judges.

Existing work has studied these challenges largely along separate lines.
Position-bias studies document order sensitivity \citep{wang2024fair,shi2025judging} and mitigate it via response swapping, verdict correction, or judge training \citep{wang2024fair,yang2025any,zhou2026toward,yang2026fairjudge}, while human-supervised evaluation uses limited human labels to learn evaluation criteria, calibrate LLM-derived features, or improve human-centered evaluation \citep{liu2024hdeval,kola2026saja,boyeau2025autoeval,xie2026atc}.
What remains less understood is how to connect stages in multi-judge evaluation: recover position-debiased LLM preferences, learn shared structure across judges, and calibrate it toward latent human preferences using limited human comparisons.

We introduce \textbf{DIAL} (\textbf{D}ebiasing and human-preference \textbf{I}nformed \textbf{A}lignment for \textbf{L}LM judges), a unified framework for position debiasing and human-preference calibration in multi-judge evaluation.
DIAL separates judge-specific position effects from latent LLM preferences, learns shared structure across the recovered position-debiased preferences, and calibrates that structure toward a latent human-preference target using limited human comparisons.
The resulting representation reduces unrestricted item-level human-preference estimation to calibration along learned low-dimensional preference directions.
Adaptive weighting controls reliance on LLM evidence, interpolating between human-only and LLM-anchored estimation; \Cref{fig:overview}(c) summarizes the DIAL framework.

\begin{figure}[!t]
    \centering
    \includegraphics[width=\linewidth]{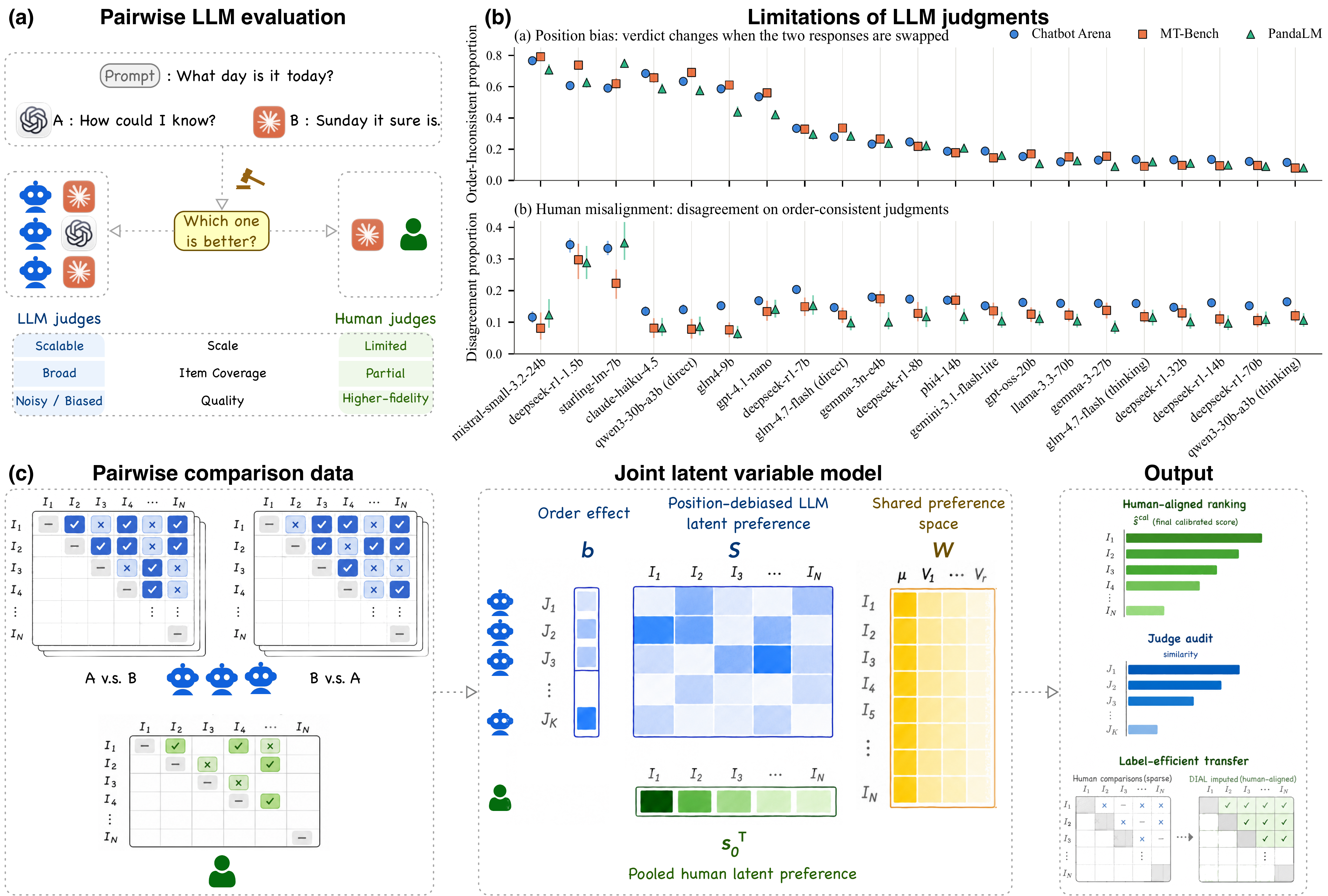}
\caption{Overview of DIAL framework.
(a) Pairwise evaluation combines abundant, scalable LLM judgments with limited, higher-fidelity human judgments.
(b) Empirical examples illustrate two distinct challenges: position bias and human-preference misalignment.
(c) DIAL jointly models LLM and human comparisons to separate judge-specific position effects from latent LLM preferences, learn shared preference structure across judges, and adaptively calibrate toward human preference.}
    \label{fig:overview}
\end{figure}

\paragraph{Contributions.}
Our contributions are threefold.
(i) We characterize identification of judge-specific position effects and latent LLM preferences via a graph-cycle criterion, and identify the human target through the learned LLM preference space (\Cref{prop:ident-LLM,prop:human_identification}).
(ii) We characterize the first-order approximation-estimation tradeoff of LLM calibration (\Cref{thm:population-risk-tradeoff}), derive fixed-weight inference and a local-misalignment shrinkage operator (\Cref{thm:weighted-asymptotic-normality} and \Cref{prop:local-misalignment}), and develop GACV for tuning the human-LLM weight using leave-one-human-out loss (\Cref{thm:gacv}).
(iii) Across synthetic and three real benchmarks, DIAL remains robust to display imbalance and improves human-label efficiency under limited human supervision, including regimes where unrestricted human-only BTL estimation is unstable or nonfinite.
Adaptive updating becomes especially valuable when the learned LLM representation is data-limited or imperfect.
Our study also contributes over $410$K judgments from $21$ LLM judges across three benchmarks in both display orders.

\section{Related Work}\label{sec:related-work}

\paragraph{Position Bias and Paired-Comparison Models.}
LLM judges exhibit systematic biases including position, verbosity, and self-enhancement biases \citep{zheng2023mtbench}: pairwise verdicts can change when the two responses are swapped \citep{wang2024fair}, and order sensitivity varies across judges, tasks, and comparison difficulty \citep{shi2025judging}.
Existing mitigation either corrects individual verdicts at evaluation time \citep{yang2025any} or modifies the judge through training \citep{zhou2026toward,yang2026fairjudge}.
The Bradley-Terry-Luce (BTL) model \citep{bradley1952rank,luce1959individual} can incorporate within-pair order effects \citep{davidson1977order,augustin2004order}, treating display order as a comparison-varying covariate, so that identification becomes a property of the augmented comparison design, as in ranking models with comparison-varying covariates \citep{dong2026dynamic} and recent work on identifiability of LLM-judge biases \citep{xu2026debias}.
When position bias is not of concern, user-specific accuracy \citep{jin2020rank}, judge-specific discrimination \citep{xu2026judgeaware} and consensus-disagreement structure \citep{yu2026heterogeneous} are studied for multiple~judges.

\paragraph{Human Preference Aggregation and Calibration.}
Classical latent-variable approaches model annotator noise and reliability \citep{dawid1979maximum,chen2013crowdbt}, while empirical studies document systematic differences between human and LLM judges \citep{chen2024humans,gao2025reevaluating} and substantial uncertainty within human evaluation itself \citep{elangovan2025beyond}.
Recent work uses limited human labels to learn evaluation criteria \citep{liu2024hdeval}, to calibrate features extracted from LLM evaluations \citep{kola2026saja}, or to combine abundant automatic labels with a smaller human sample when estimating aggregate evaluation functionals such as win rates, benchmark scores, or ranking metrics \citep{boyeau2025autoeval,chen2026efficient,lee2026correctly,divekar2026precise}.
DIAL instead targets latent human preferences and asks how a position-debiased multi-LLM-judge representation can improve estimation from limited human comparisons, separating order-effect estimation from human alignment and propagating uncertainty from both stages.
See \Cref{app:related-comparison} for more detailed comparisons with related methods.

\section{Position-debiased LLM preference representation}
\label{sec:llm-position}

We write $s_{\human}\in\RR^N$ for pooled-human scores and $s_k\in\RR^N$, $k=1,\ldots,K$, for the LLM judges' scores over the $N$ items, and collect the LLM scores and order effects in $S=[s_1,\ldots,s_K]^\top\in\RR^{K\times N}$ and $b=(b_1,\ldots,b_K)^\top\in\RR^K$.
Because BTL scores are invariant to common additive shifts, we fix centered representatives throughout.

\begin{condition}[Score normalization]
\label[condition]{cond:score-normalization}
    For $k\in\{\human,1,\ldots,K\}$, the latent scores satisfy $\one_N^\top s_k=0$.
\end{condition}

\subsection{Judge-specific order-effect BTL and position debiasing}

For each unordered pair $\{i,j\}$, fix the canonical orientation $i<j$ and let $A\in\{-1,1\}$ indicate whether item $i$ is displayed first ($A=1$) or second ($A=-1$).
For judge $k$, let $\Omega_k=\{(i,j,a): i<j,\ a\in\{-1,1\},\ n_{kij}^{(a)}>0\}$ with $m_k=|\Omega_k|$, let $\mathcal{G}_k=([N],\mathcal{E}_k)$ be the comparison graph in which $\{i,j\}\in\mathcal{E}_k$ whenever $(i,j,a)\in\Omega_k$ for at least one $a$, and let $Y_{kij}^{(a)}$ be the number of times judge $k$ prefers item $i$ among $n_{kij}^{(a)}$ comparisons under order $A=a$.
For $(i,j,a)\in\Omega_k$, assume
\begin{equation}
Y_{kij}^{(a)}
\sim
\mathrm{Binomial}\left(n_{kij}^{(a)},p_{kij}^{(a)}\right),
\qquad 
\logit\left(p_{kij}^{(a)}\right)
=
s_{ki}-s_{kj}+ab_k,
\label{eq:binomial-model-llm}
\end{equation}
where \(p_{kij}^{(a)}\) is the order-specific preference probability and $\sigma(x)=(1+e^{-x})^{-1}$ denotes the logistic function.
Conditional on the comparison design, the individual LLM comparison outcomes underlying these counts are mutually independent across judges, pairs, display orders, and repetitions.
Here, $b_k>0$ indicates preference for the first-displayed response and $b_k<0$ preference for the second; exponentiating the score and order-effect terms recovers the classical within-pair order-effect BTL parameterization \citep{davidson1977order,augustin2004order}, here applied judge-specifically to the response-order sensitivity documented for LLM evaluators \citep{wang2024fair,shi2025judging}.

Fix an arbitrary enumeration $\Omega_k=\{(i_t,j_t,a_t)\}_{t=1}^{m_k}$.
The item-difference design matrix $D_k\in\{-1,0,1\}^{m_k\times N}$ has rows $(D_k)_{t\largecdot}=(e_{i_t}-e_{j_t})^\top$, and the order-indicator vector $z_k\in\{-1,1\}^{m_k}$ has entries $(z_k)_t=a_t$.

\begin{condition}[LLM comparison design]
\label[condition]{cond:llm-design}
    For every $k=1,\ldots,K$, the augmented comparison design $[D_k,z_k]$ has rank $N$.
\end{condition}

\begin{proposition}[Identification of latent scores and order effect]
\label[proposition]{prop:ident-LLM}
    Under \eqref{eq:binomial-model-llm} and \Cref{cond:score-normalization}, for any judge $k$ the parameter $(s_k,b_k)$ is identifiable if and only if $\rank([D_k,z_k])=N$, equivalently if and only if $\mathcal{G}_k$ is connected and $z_k\notin\col(D_k)$.
    Consequently, under \Cref{cond:llm-design}, $(S,b)$ is identifiable.
\end{proposition}

This rank characterization specializes augmented paired-comparison identification with comparison-varying covariates \citep{dong2026dynamic} to the observed order indicator; unlike unrestricted LLM-judge bias terms that can be confounded with item effects \citep{xu2026debias}, $b_k$ is separable from $s_k$ exactly when $z_k\notin\col(D_k)$.
Because $\col(D_k)$ is the space of potential differences on the comparison multigraph, $z_k\notin\col(D_k)$ holds if and only if some cycle has nonzero signed order sum (\Cref{cor:cycle-criterion}), which is checkable in one pass; a pair shown under both orders is the length-two cycle, and the criterion also covers designs in which every pair appears in one order only.
\Cref{cond:llm-design} is an identification condition on the comparison support and does not guarantee a finite empirical BTL estimator, which may fail under separation \citep{ford1957solution,hunter2004mm}; asymptotically, the same rank condition applies to the limiting support $\{(i,j,a):\rho_{kij}^{(a)}>0\}$ with within-judge limiting sampling proportions $\rho_{kij}^{(a)}:=\lim n_{kij}^{(a)}/n_k$ and sample size $n_k:=\sum_{(i',j',a')\in\Omega_k}n_{ki'j'}^{(a')}$.

After separating the order effect, we define the position-debiased preference probability $p^{\mathrm{deb}}_{kij}:=\sigma(s_{ki}-s_{kj})$, so that judge $k$'s position-debiased ranking is read directly from $s_k$.
This differs from the order-averaged probability $p^{\mathrm{avg}}_{kij}:=\{p_{kij}^{(1)}+p_{kij}^{(-1)}\}/2$, whose log-odds are a nonlinear odd function of $s_{ki}-s_{kj}$ whenever $b_k\neq0$, so they are generally not representable by a single BTL score vector and BTL fitting to order-averaged data is generically misspecified (\Cref{app:debias-vs-average}).
We therefore carry the recovered judge-specific $s_k$, rather than probability-averaged judgments, into the cross-judge representation and subsequent human calibration.

\subsection{Structured representation across LLM judges}
\label{subsec:structured-llm}

Following \citet{yu2026heterogeneous}, we adopt its consensus-disagreement decomposition for the position-debiased LLM score matrix.

\begin{condition}[Structured LLM representation]
\label[condition]{cond:llm-structure}
    The score matrix $S$ admits the representation
    \begin{equation}
    S
    =
    \gamma\mu^\top+UV^\top,
    \label{eq:S-llm}
    \end{equation}
    where $\gamma\in\RR^K$, $\mu\in\RR^N$, $U\in\RR^{K\times r}$, and $V\in\RR^{N\times r}$ satisfy
    \begin{enumerate}[(i), ref=\roman*]
        \item \label{eq:structured-llm-item-constraints} \(\one_N^\top\mu=0\), \(V^\top\one_N=0\), \(\mu^\top V=0\), \(N^{-1}\|\mu\|_2^2=1\) and \(N^{-1}V^\top V=I_r\).

        \item \label{eq:structured-llm-judge-constraints} \(\one_K^\top\gamma>0\), \(\one_K^\top U=0^\top\) and $\rank(U)=r$.
    \end{enumerate}
\end{condition}

Here $\mu$ is the common LLM preference direction, $\gamma_k$ is judge $k$'s loading on it, and $UV^\top$ is rank-$r$ heterogeneous variation whose item-side directions, the columns of $V$, are those along which the judges systematically disagree.
These restrictions are compatible with \Cref{cond:score-normalization} and require $r\leq K-1$ and $r+1\leq N-1$.
Define $W:=[\mu,V]\in\RR^{N\times(r+1)}$, so that \eqref{eq:structured-llm-item-constraints} reads $W^\top\one_N=\zero_{r+1}$ and $N^{-1}W^\top W=I_{r+1}$: fixing the scale on the item side rather than the judge side orthogonalizes the basis of $\col(W)$, gives $\|Wc\|_2^2=N\|c\|_2^2$, and confines the item side to a compact set.
For human calibration in \Cref{sec:human-alignment}, let $M=\mu\in\RR^{N\times1}$ with $d=1$, or $M=W\in\RR^{N\times(r+1)}$ with $d=r+1$, corresponding to consensus-only and full-space calibration, respectively.

\Cref{prop:ident-LLM} identifies $(S,b)$, and the constraints of \Cref{cond:llm-structure} then give $\mu=\sqrt N\,S^\top\one_K/\|S^\top\one_K\|_2$, $\gamma=N^{-1}S\mu$, $UV^\top=S-\gamma\mu^\top$, and $\col(V)=\operatorname{row}(UV^\top)$, so $\mu$, $\gamma$, $UV^\top$, $\col(V)$, and $\col(W)$ are all identified.
The factors $U$ and $V$ themselves are not, because $UV^\top=(UQ)(VQ)^\top$ for every orthogonal $Q\in\RR^{r\times r}$; only the subspace $\col(V)$ is invariant.
Conditioning on $\col(W)$, the shared structure reduces judge $k$'s design requirement from $\rank([D_k,z_k])=N$ to $\rank([D_kW,z_k])=r+2$ (\Cref{prop:ident-lowrank}).
Since $\rank(W)=r+1$ and $\col(W)=\operatorname{span}(\mu)\oplus\col(V)$, both calibration choices are defined by the position-debiased LLM representation.

\section{Human preference alignment with limited supervision}
\label{sec:human-alignment}

\subsection{Human calibration and identification}
\label{subsec:human-model}

For pooled human judgments, recall that $s_{\human}\in\RR^N$ denotes the centered latent human preference score vector.
For each pair $i<j$, let $Y_{\human ij}$ be the number of human comparisons preferring item $i$ among $n_{\human ij}$ comparisons, define $n_{\human}:=\sum_{i<j}n_{\human ij}$ and $\Omega_{\human}:=\{(i,j):i<j,\ n_{\human ij}>0\}$, and let $\mathcal G_{\human}$ denote the corresponding comparison graph.
We assume
\begin{equation}
Y_{\human ij}
\sim
\mathrm{Binomial}\!\left(n_{\human ij},p_{\human ij}\right),
\qquad
\logit\!\left(p_{\human ij}\right)=s_{\human i}-s_{\human j},
\qquad
(i,j)\in\Omega_{\human}.
\label{eq:binomial-model-human}
\end{equation}
Recent human-supervised LLM evaluation learns evaluation criteria or calibrates LLM-derived features from limited human labels \citep{liu2024hdeval,kola2026saja}; here the target is instead the latent pairwise human score $s_{\human}$.
Treating human annotations as noisy observations of this target, for a calibration basis $M$ defined in \Cref{subsec:structured-llm}, we impose the working restriction
\begin{equation}
s_{\human}=Mc_{\human},
\qquad
c_{\human}\in\RR^d.
\label{eq:human-subspace-model}
\end{equation}
This reduces human estimation from $N-1$ free score coordinates to $d$ calibration coordinates.
The restriction is an alignment assumption; $Mc_{\human}$ is basis-invariant, although for $M=W$ its coordinate $c_{\human}$ depends on the normalized basis for $\col(V)$.

Enumerate $\Omega_{\human}=\{(i_t,j_t)\}_{t=1}^{m_{\human}}$, $m_{\human}=|\Omega_{\human}|$, and let $D_{\human}\in\{-1,0,1\}^{m_{\human}\times N}$ have rows $(e_{i_t}-e_{j_t})^\top$.

\begin{condition}[Human comparison design]
\label[condition]{cond:human-design}
    $\rank(D_{\human}M)=d$.
\end{condition}

\begin{proposition}[Identification under limited human comparisons]
\label[proposition]{prop:human_identification}
    Under model \eqref{eq:binomial-model-llm}, \eqref{eq:binomial-model-human}, and \eqref{eq:human-subspace-model}, \Cref{cond:score-normalization,cond:llm-design,cond:llm-structure,cond:human-design}, $c_{\human}$ is identifiable relative to the chosen basis and $s_{\human}=Mc_{\human}$ is identifiable and basis-invariant.
    The rank condition is equivalent to $\ker(D_{\human})\cap\col(M)=\{0\}$; connectivity of $\mathcal G_{\human}$ is sufficient but not necessary.
\end{proposition}

Unlike methods combining automatic and human labels to estimate aggregate evaluation functionals \citep{boyeau2025autoeval,chen2026efficient,lee2026correctly,divekar2026precise}, \Cref{prop:human_identification} targets the latent item-level preference through the learned calibration space $\col(M)$.
Condition~\ref{cond:human-design} asks the human comparisons to distinguish only its $d$ directions, so $m_{\human}\ge d$ is necessary but connectivity or coverage of every item is not; \Cref{thm:incomplete-human-coverage} extends the risk and inference results to incomplete population coverage.

\paragraph{Likelihoods.}
The normalized human loss for a centered score $s$ is
\begin{equation}
\ell_{\human}(s)
:=
\frac{1}{n_{\human}}
\sum_{(i,j)\in\Omega_{\human}}
\left[
n_{\human ij}\log\{1+\exp(s_i-s_j)\}
-
Y_{\human ij}(s_i-s_j)
\right].
\label{eq:human-full-loss}
\end{equation}
Equivalently $\ell_{\human}(s)=n_{\human}^{-1}\sum_{t=1}^{n_{\human}}h_t(s)$, the average over individual comparisons, with $h_t(s)=\log\{1+\exp(s_{I_t}-s_{J_t})\}-Z_t(s_{I_t}-s_{J_t})$ for the pair $(I_t,J_t)$ and outcome $Z_t\in\{0,1\}$.
Let $n_{\llm}:=\sum_{k=1}^Kn_k$ be the total LLM sample size and define $\eta_{kij}^{(a)}(S,b):=S_{ki}-S_{kj}+ab_k$.
The normalized LLM negative log-likelihood is
\begin{equation}
\ell_{\llm}(S,b)
=
\frac{1}{n_{\llm}}
\sum_{k=1}^K
\sum_{(i,j,a)\in\Omega_k}
\left[
n_{kij}^{(a)}
\log\!\left\{
1+\exp\!\left(\eta_{kij}^{(a)}(S,b)\right)
\right\}
-
Y_{kij}^{(a)}
\eta_{kij}^{(a)}(S,b)
\right].
\label{eq:llm-negloglik}
\end{equation}

Calibration evaluates the same loss at $s=Mc$.
When a finite minimizer exists, define the unrestricted human-only estimator by
\begin{equation}
\widehat s_0\in\argmin\nolimits_{\one_N^\top s=0}\ell_{\human}(s).
\label{eq:human-only-estimator}
\end{equation}

\subsection{Population value of LLM calibration}
\label{subsec:population-calibration}

To characterize the tradeoff that the weighted estimator of \Cref{subsec:weighted-estimator} interpolates, we temporarily relax \eqref{eq:human-subspace-model} and allow $s_{\human}$ to lie outside $\col(M)$.

For the population analysis, we treat the $n_{\human}$ human comparisons as independent draws in which pair $(i,j)$ is sampled with probability $\rho_{ij}$, with $\sum_{i<j}\rho_{ij}=1$, and the outcome follows \eqref{eq:binomial-model-human} conditional on the sampled pair.
Define $\Omega_{\human}^{\infty}:=\{(i,j):i<j,\ \rho_{ij}>0\}$ and the corresponding comparison graph $\cG_{\human}^{\infty}$.
The human comparison risk is $R_{\human}(s):=\EE\{h_t(s)\}$ under this sampling scheme.
When $\cG_{\human}^{\infty}$ is connected, $s_{\human}$ minimizes $R_{\human}$ over centered score vectors.

Define the approximation error of the calibration space by
\begin{equation}
\Delta_M
:=
\inf_{c\in\RR^d}
\left\{
R_{\human}(Mc)-R_{\human}(s_{\human})
\right\}.
\label{eq:population-alignment-gaps}
\end{equation}
For fixed sampling design $\rho$, $\Delta_M$ depends on $M$ only through $\col(M)$ and is nonnegative, vanishing if and only if $s_{\human}\in\col(M)$ when $\cG_{\human}^\infty$ is connected; $\Delta_M$ is design-dependent through $\rho$.

To isolate this tradeoff from estimation of the LLM representation, treat $M$ as known at its population value and define the oracle anchored estimator $\widetilde s_\infty:=M\widetilde c_\infty$, where $\widetilde c_\infty\in\argmin_{c\in\RR^d}\ell_{\human}(Mc)$.
For expected-risk statements, we use the measurable extension of \Cref{lem:fixed-space-existence}, setting a fixed-space estimator to $\zero_N$ on its exponentially rare fallback event.
The $(N-1)/(2n_{\human})$ and $d/(2n_{\human})$ terms below are the classical first-order likelihood-risk complexity terms underlying AIC \citep{akaike1974new}.

\begin{theorem}[Human-risk tradeoff from LLM calibration]
\label[theorem]{thm:population-risk-tradeoff}
    Under the sampling scheme above, let $N$ and $M$ be fixed, let $\cG_{\human}^{\infty}$ be connected, and assume the human probabilities are bounded away from $0$ and $1$ on $\Omega_{\human}^{\infty}$.
    If $s_{\human}\in\col(M)$, the first-order expected excess risks of $\widehat s_0$ and $\widetilde s_\infty$ are $(N-1)/(2n_{\human})$ and $d/(2n_{\human})$, respectively.
    More generally, let $s_{\human}=s_{\human}^{(n_{\human})}$ vary with $n_{\human}$ while keeping $\{\rho_{ij}\}$ fixed, and suppose its induced probabilities
    $p_{\human ij}^{(n_{\human})}:=\sigma(s_{\human i}^{(n_{\human})}-s_{\human j}^{(n_{\human})})$
    lie in $[\varepsilon,1-\varepsilon]$ on $\Omega_{\human}^{\infty}$ for some fixed $\varepsilon>0$, with
    $n_{\human}\Delta_M\to\delta_M<\infty$.
    Then
    \begin{equation}
        n_{\human}\EE\!\left\{
        R_{\human}(\widehat s_0)-R_{\human}(s_{\human})
        \right\}\to\frac{N-1}{2},
        \qquad
        n_{\human}\EE\!\left\{
        R_{\human}(\widetilde s_\infty)-R_{\human}(s_{\human})
        \right\}\to\delta_M+\frac{d}{2}.
        \label{eq:population-risk-anchored}
    \end{equation}
    Thus, anchoring has a smaller first-order risk exactly when $\delta_M<(N-1-d)/2$.
\end{theorem}

For $M=\mu$, $d=1$ gives consensus-only calibration; for $M=W$, $d=r+1$ recovers full-space calibration.
A fixed $\Delta_M>0$ eventually dominates as $n_{\human}$ grows, motivating adaptive weighting.
The above analysis treats $M$ as fixed; the analysis also extends to misspecification of \eqref{eq:human-subspace-model} and staged-oracle equivalence; see \Cref{thm:fixed-space-tradeoff} and \Cref{cor:staged-oracle-equivalence}.

\section{DIAL: Estimation and adaptive calibration}
\label{sec:estimation}

To utilize both the human and LLM comparisons, we combine the two likelihoods \eqref{eq:human-full-loss} and \eqref{eq:llm-negloglik}.
For fixed rank $r\leq\min\{K-1,N-2\}$, let
$\Theta_r$ denote the parameter space
$\theta=(\gamma,\mu,U,V,b)$ satisfying \Cref{cond:llm-structure}.
For $\theta\in\Theta_r$, write $S_\theta:=\gamma\mu^\top+UV^\top$ and $W_\theta:=[\mu,V]$, and set $M_\theta:=\mu$ for $M=\mu$ and $M_\theta:=W_\theta$ for $M=W$.

\subsection{Weighted likelihood estimator}
\label{subsec:weighted-estimator}

Using the human loss in \eqref{eq:human-full-loss}, define for $\lambda>0$
\begin{equation}
Q_\lambda(s,S,b)
:=
\ell_{\human}(s)+\lambda\ell_{\llm}(S,b).
\label{eq:weighted-loss}
\end{equation}
The factorized estimator is
\begin{flalign}
&(\methAda{})\qquad
(\widehat c_\lambda,\widehat\theta_\lambda)
\in
\argmin_{c\in\RR^d,\ \theta\in\Theta_r}
Q_\lambda(M_\theta c,S_\theta,b),
\qquad
\widehat s_\lambda
:=
M_{\widehat\theta_\lambda}\widehat c_\lambda.&
\label{eq:weighted-joint-estimator}
\end{flalign}
Because the two likelihoods are normalized by $n_{\human}$ and $n_{\llm}$, $\lambda$ controls their relative weight, with the ordinary joint likelihood corresponding to $\lambda=n_{\llm}/n_{\human}$.
Under exact calibration, that choice is efficient; choosing $\lambda$ away from $n_{\llm}/n_{\human}$ trades exact-model efficiency for robustness when the calibration restriction is only approximate.
We treat the score $\widehat s_\lambda$ as the human-preference estimate.

We use the unrestricted human BTL estimator $\widehat s_0$ in \eqref{eq:human-only-estimator} as the $\lambda=0$ endpoint.
For $\lambda=\infty$, define the staged LLM-anchored endpoint by
\begin{flalign}
&(\methConsensus{})\quad
\widehat\theta_\infty
\in
\argmin_{\theta\in\Theta_r}\ell_{\llm}(S_\theta,b),
\qquad
\widehat c_\infty
\in
\argmin_{c\in\RR^d}
\ell_{\human}\!\left(M_{\widehat\theta_\infty}c\right),
\quad
\widehat s_\infty
:=
M_{\widehat\theta_\infty}\widehat c_\infty.&
\label{eq:staged-calibrated-estimator}
\end{flalign}
Thus, $\lambda=0$ gives human-only estimation, $\lambda=\infty$ gives LLM-first calibration, and finite $\lambda$ jointly estimates the shared representation using both data sources.
Relative to human-supervised calibration of LLM-derived features \citep{kola2026saja}, finite $\lambda$ additionally lets the human comparisons update the LLM representation itself, so $\widehat s_\lambda$ is not restricted to a fixed estimated space.
Under approximate calibration, this can bias the LLM parameters.
For the fixed calibration choice $M\in\{\mu,W\}$, write $\mathcal C_\mu(S):=\operatorname{span}(S^\top\one_K)$ and $\mathcal C_W(S):=\operatorname{row}(S)$.
For inference and tuning, we use $Q_\lambda$ on the regular score space:
\begin{equation}
    \mathcal M_r
    :=
    \left\{
    (s,S,b):
    S\one_N=\zero_K,\
    \rank(S)=r+1,\
    S^\top\one_K\neq\zero_N,\
    s\in\mathcal C_M(S),\
    b\in\RR^K
    \right\},
    \label{eq:reduced-parameter-space}
\end{equation}
Regular factorized and score-space fits have the same objective and calibrated score (\Cref{lem:factor-reduced-equivalence}).

\subsection{Inference at a fixed calibration weight}
\label{subsec:asymptotic-normality}

Fix $\lambda\in(0,\infty)$, $N,K,r$, and a calibration choice $M\in\{\mu,W\}$.
Suppose the human comparisons are independent draws under the sampling scheme of \Cref{subsec:population-calibration}, with connected population graph $\cG_{\human}^\infty$, independently of the LLM sample.
Let $n_{\human},n_{\llm}\to\infty$ with $n_{\human}/n_{\llm}\to\tau\in[0,\infty)$, and suppose $n_k/n_{\llm}\to\kappa_k>0$ and $n_{kij}^{(a)}/n_k\to\rho_{kij}^{(a)}$, where each limiting within-judge support satisfies \Cref{cond:llm-design}.
Write $R_{\llm}$ for the limiting expected LLM loss as in \eqref{eq:R-llm} and use smooth local coordinates $\zeta$ on $\mathcal M_r$, with true coordinate $\zeta_0$.
Define
\[
\mathcal I_{\human}
:=
\nabla_\zeta^2R_{\human}\{s(\zeta)\}\big|_{\zeta=\zeta_0},
\qquad
\mathcal I_{\llm}
:=
\nabla_\zeta^2R_{\llm}\{S(\zeta),b(\zeta)\}\big|_{\zeta=\zeta_0},
\]
together with $\mathcal H_\lambda:=\mathcal I_{\human}+\lambda\mathcal I_{\llm}$, $\mathcal J_{\lambda,\tau}:=\mathcal I_{\human}+\lambda^2\tau\mathcal I_{\llm}$, and Jacobian $\dot s_0:=\frac{\partial s(\zeta_0)}{\partial\zeta^\top}\in\RR^{N\times\dim\zeta}$.
The covariance below has the standard sandwich form from likelihood and M-estimation theory \citep{white1982maximum}; its two-sample form propagates both human and LLM sampling variation through $\mathcal J_{\lambda,\tau}$.

\begin{theorem}[Fixed-weight inference for the calibrated preference]
\label[theorem]{thm:weighted-asymptotic-normality}
    Under model \eqref{eq:binomial-model-llm}, \eqref{eq:binomial-model-human}, and \eqref{eq:human-subspace-model}, \Cref{cond:score-normalization,cond:llm-design,cond:llm-structure}, and the sampling regime above,
    \begin{equation}
    \sqrt{n_{\human}}
    \left(
    \widehat s_\lambda-s_{\human}
    \right)
    \dto
    \mathcal N
    \left(
    0,\Sigma_{\lambda,s,\tau}
    \right),
    \qquad
    \Sigma_{\lambda,s,\tau}
    :=
    \dot s_0
    \mathcal H_\lambda^{-1}
    \mathcal J_{\lambda,\tau}
    \mathcal H_\lambda^{-1}
    \dot s_0^\top.
    \label{eq:calibrated-asymptotic-normality}
    \end{equation}
    The covariance is invariant to the smooth local coordinate used for $\mathcal M_r$.
    The empirical sandwich estimator $\widehat\Sigma_{\lambda,s}$ defined in \eqref{eq:weighted-empirical-sandwich} satisfies $\widehat\Sigma_{\lambda,s}\pto\Sigma_{\lambda,s,\tau}$, so $n_{\human}^{-1}\widehat\Sigma_{\lambda,s}$ consistently estimates the first-order covariance of $\widehat s_\lambda$.
    If $n_{\human}/n_{\llm}\to0$, the asymptotic linearization and sandwich consistency are uniform over $\lambda\in[\lambda_0,\infty)$ for every $\lambda_0>0$; \Cref{cor:uniform-lambda} also gives the $\lambda\to\infty$ staged limit.
\end{theorem}

For $\tau>0$, both human and LLM sampling variation contribute through $\mathcal J_{\lambda,\tau}$; when $\tau=0$, LLM sampling variation vanishes at the human $\sqrt{n_{\human}}$ scale while its information remains through $\mathcal H_\lambda$.
Since $\one_N^\top s(\zeta)\equiv0$, the limit law is supported on the centered subspace and $\Sigma_{\lambda,s,\tau}\one_N=\zero_N$.
For any centered contrast $v$ with $v^\top\Sigma_{\lambda,s,\tau}v>0$, including pairwise contrasts $v=e_i-e_j$ when nondegenerate, we studentize by $(v^\top\widehat\Sigma_{\lambda,s}v/n_{\human})^{1/2}$.
These intervals concern exact calibration at a fixed $\lambda$; \Cref{thm:incomplete-human-coverage} relaxes human-graph connectivity, and \Cref{prop:local-misalignment} and \Cref{cor:finite-weight-risk} give the local-misalignment limit.

\subsection{Adaptive choice of the LLM weight}
\label{subsec:adaptive-lambda}

Let $\overline\Lambda\subset[0,\infty]$ be a fixed finite candidate set containing $0$ and $\infty$, and write $\Lambda:=\overline\Lambda\setminus\{0,\infty\}$.
Following the GACV principle of approximating leave-one-out likelihood loss without refitting every deletion \citep{xiang1996generalized} and related GCV-based tuning without sample splitting \citep{du2023subsample}, we select $\lambda$ by approximating leave-one-human-comparison-out predictive loss while retaining all LLM comparisons in every candidate fit.

For each candidate fit, let $\widehat\zeta_\lambda$ be its local-coordinate estimate and $q_\lambda(\zeta)$ its full-sample criterion, and define $\widehat H_\lambda:=\nabla_\zeta^2q_\lambda(\widehat\zeta_\lambda)$ and $\widehat J_\lambda:=n_{\human}^{-1}\sum_t(g_{t,\lambda}-\bar g_\lambda)(g_{t,\lambda}-\bar g_\lambda)^\top$, where $g_{t,\lambda}:=\nabla_\zeta h_t\{s(\zeta)\}|_{\zeta=\widehat\zeta_\lambda}$ and $\bar g_\lambda:=n_{\human}^{-1}\sum_t g_{t,\lambda}$.
Here, $\zeta$ denotes local coordinates on $\mathcal M_r$ for finite $\lambda$, on the centered score space at $\lambda=0$, and on $c$ with $M_{\widehat\theta_\infty}$ fixed at $\lambda=\infty$.
Define
\begin{equation}
\operatorname{GACV}(\lambda)
:=
\ell_{\human}(\widehat s_\lambda)
+
\frac{1}{n_{\human}-1}
\operatorname{tr}
\left(
\widehat H_\lambda^{-1}\widehat J_\lambda
\right),
\qquad
\widehat\lambda
\in
\argmin_{\lambda\in\overline\Lambda}
\operatorname{GACV}(\lambda).
\label{eq:gacv}
\end{equation}
For comparison, write $\operatorname{CV}_{\mathrm{loo}}(\lambda):=n_{\human}^{-1}\sum_{t=1}^{n_{\human}}h_t(\widehat s_\lambda^{(-t)})$, where $\widehat s_\lambda^{(-t)}$ is the leave-one-out (LOO) fit obtained after deleting human comparison $t$ while retaining all LLM comparisons.

\begin{theorem}[GACV tuning]
\label[theorem]{thm:gacv}
    Suppose that $N$, $K$, and $r$ are fixed, the calibration choice $M\in\{\mu,W\}$ is fixed, the human comparisons are independent draws under the sampling scheme of \Cref{thm:weighted-asymptotic-normality}, independently of the LLM sample, and $\overline\Lambda$ is fixed and finite.
    Under \Cref{ass:weighted-local-regularity} (verified under the correctly specified regime by \Cref{lem:gacv-regularity}), as $n_{\human}\to\infty$, conditionally on the LLM comparisons, $\operatorname{GACV}(\lambda)$ is invariant to the chosen smooth local coordinates and
    \begin{equation}
    \max_{\lambda\in\overline\Lambda}
    \left|
    \operatorname{GACV}(\lambda)
    -
    \operatorname{CV}_{\mathrm{loo}}(\lambda)
    \right|
    =
    \Op(n_{\human}^{-2}).
    \label{eq:gacv-loo-equivalence}
    \end{equation}
\end{theorem}

Here, GACV tunes DIAL-Ada \eqref{eq:weighted-joint-estimator} without a separate human validation sample or a rate restriction between $n_{\human}$ and $n_{\llm}$; \Cref{cor:gacv-loo-optimality} gives its near-LOO selection guarantee.
The same criterion can select rank and weight jointly, while \Cref{prop:calibration-test} tests the calibration restriction and compares DIAL-Anc versus unrestricted human models by AIC.

\section{Experiments}
\label{sec:experiments}

We evaluate the two aspects of the proposed DIAL pipeline separately.
Synthetic experiments first analyze recovery of order effects and position-debiased LLM preferences, and human-preference calibration under limited supervision and LLM misspecification.
Real-data experiments study the same mechanisms against held-out human preferences.

\paragraph{Compared methods.}
We compare (i) \methHuman{}, the unrestricted centered human BTL fit ($\lambda=0$);
(ii) \methPooled{}, a pooled LLM BTL fit that ignores judge identity and display order and is rescaled using the human comparisons;
(iii) \methDIALmu{} and \methConsensus{}, the adaptive estimator \eqref{eq:weighted-joint-estimator} tuned by GACV and anchored estimator \eqref{eq:staged-calibrated-estimator} within the consensus space;
(iv) \methDIALW{}, the adaptive estimator \eqref{eq:weighted-joint-estimator} within $W$-space;
(v) \methDIALnoPos{}, the order-agnostic ablation of \methDIALmu{};
and (vi) \methAtC{}, a stage-matched AtC baseline aggregating the same human comparisons by BTL and isotonic-calibrating the position-debiased LLM consensus.

\subsection{Synthetic experiments}
\label{subsec:synthetic}

We use $N=10$, $K=4$, and $r=1$, with structured LLM scores from \eqref{eq:binomial-model-llm} and consensus-aligned target $s_{\human}=\mu$.
The canonical item is displayed first with probability $0.75$.
In \Cref{fig:simu-main}, the top row varies $n_{\human}$ at $n_{\llm}=20{,}000$ under exact specification, while the bottom row varies $n_{\llm}$ at $n_{\human}=800$ under judge-pair overdispersion that misspecifies the LLM likelihood.
\Cref{app:simulation} gives the full design and additional larger-scale, misalignment, and heterogeneous-position robustness experiments (\Cref{fig:simu-items20,fig:simu-misaligned,fig:simu-pos-heterogeneity}).

\begin{figure}[!t]
    \centering
    \includegraphics[width=0.95\linewidth]{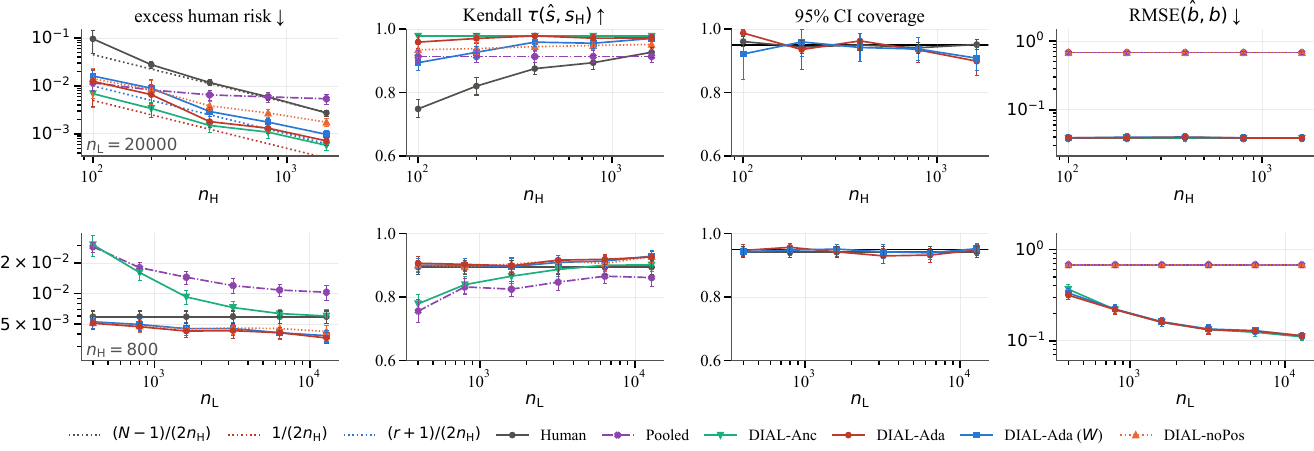}
    \caption{Synthetic evaluation with $N=10$, $K=4$, $r=1$, $s_{\human}=\mu$, and canonical-first display probability $0.75$; curves show means $\pm1.96$ Monte Carlo standard errors over $50$ replications.
    Columns report excess human risk, Kendall's $\tau$, $95\%$ interval coverage, and $\operatorname{RMSE}(\widehat b,b)$.
    Top: $n_{\human}$ varies at $n_{\llm}=20{,}000$ under correct specification; dotted lines are the first-order risk benchmarks for \methHuman{}, calibration within $\operatorname{span}(\mu)$, and calibration within $\operatorname{col}(W)$.
    Bottom: $n_{\llm}$ varies at $n_{\human}=800$ under judge-pair overdispersion that misspecifies the LLM likelihood while leaving the human calibration model correctly specified.}
    \label{fig:simu-main}
\end{figure}

\paragraph{Limited human supervision.}
Under exact specification, \methHuman{} and \methConsensus{} closely follow their corresponding first-order risk benchmarks, with anchoring substantially improving both risk and ranking when human supervision is scarce.
\methDIALmu{} remains close to the anchored endpoint, while \methDIALW{} pays additional variance for estimating an unnecessary disagreement direction.
The order-agnostic \methPooled{} and \methDIALnoPos{} deteriorate under the unbalanced display.

\paragraph{Adaptive weighting and position recovery.}
Under pair-level LLM misspecification, the anchored estimator plateaus, whereas \methDIALmu{} adapts toward the human data and remains competitive with or better than both endpoints, closely tracking the oracle candidate.
Its ranking accuracy is correspondingly more stable as the LLM budget varies.
For order-aware fits, order-effect estimation improves steadily with additional LLM comparisons, while order-agnostic methods retain substantial error; post-selection cell-clustered intervals remain empirically near nominal coverage.

\subsection{Real-data experiments}
\label{subsec:real-data}

We evaluate $21$ LLM judges ($18$ open-weight via Ollama and $3$ commercial API models) on Chatbot Arena \citep{chiang2024chatbot}, MT-Bench \citep{zheng2023mtbench}, and PandaLM \citep{wang2024pandalm}; after preprocessing, the benchmarks contain $20$, $6$, and $5$ candidate models, respectively, with each LLM comparison queried in both display orders.
We use disjoint human calibration and test sets and $r=1$ throughout; position robustness evaluations report excess held-out human log loss, while calibration experiments report Kendall's $\tau$ against the held-out human ranking (\Cref{app:sec:real,app:hyperparameters}).

\begin{figure}[!t]
    \centering
    \includegraphics[width=0.95\linewidth]{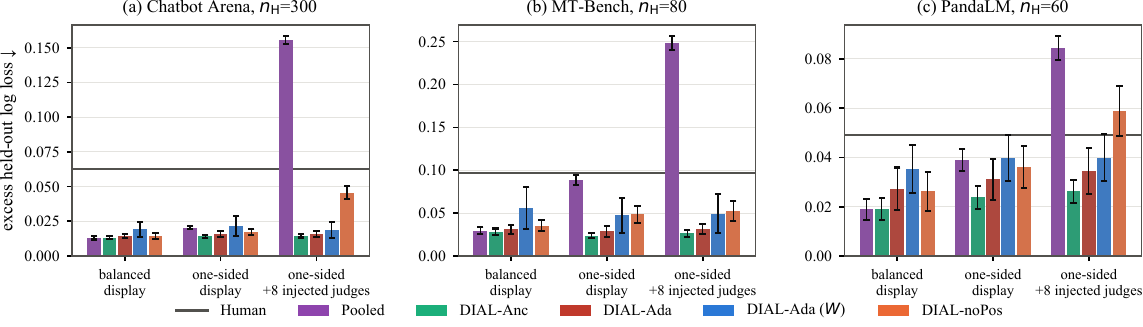}
    \caption{Real-data position robustness evaluation: excess held-out human log loss (zero is the test-pool BTL fit), averaged over up to $50$ splits with $\pm1.96$ Monte Carlo standard errors.
    Panels show Chatbot Arena, MT-Bench, and PandaLM using the same five order-sensitive judges: \texttt{Mistral-Small 3.2}, \texttt{Claude Haiku 4.5}, \texttt{Qwen3-30B} without reasoning, \texttt{GLM4-9B}, and \texttt{DeepSeek-R1-1.5B}, under balanced, one-sided ($5\%$ swapped), and one-sided plus $8$ injected pure-position judges.}
    \label{fig:real-main}
\end{figure}

\paragraph{Order-aware DIAL remains stable under position perturbations.}
Under one-sided display, \methPooled{} degrades markedly, and adding pure-position judges further degrades both \methPooled{} and \methDIALnoPos{}, whereas the order-aware DIAL fits remain stable across datasets (\Cref{fig:real-main}).
Judge-level order-effect diagnostics are reported in \Cref{fig:real-diagnostics}(a).

\begin{figure}[!t]
    \centering
    \includegraphics[width=0.95\linewidth]{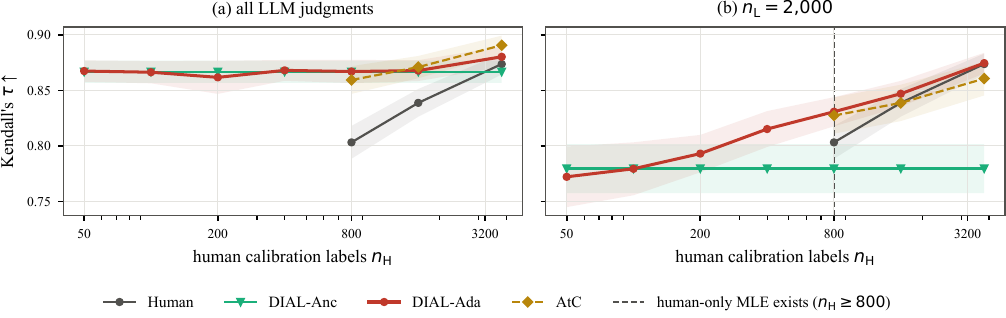}
    \caption{Human-label efficiency and adaptive calibration on Chatbot Arena: Kendall's $\tau$ with the held-out human ranking as $n_{\human}$ varies, using (a) all available judgments of the $21$ LLM judges and (b) $n_{\llm}=2{,}000$.
    Curves show means over up to $50$ splits with $\pm1.96$ Monte Carlo standard errors; \methHuman{} and \methAtC{} are shown only where the pooled-human BTL fit is finite in all splits, and the dashed vertical line marks the smallest such budget ($n_{\human}=800$).}
    \label{fig:real-calibration}
\end{figure}

\paragraph{Human efficiency and adaptive weighting.}
With all 21 LLM judges, \methConsensus{} and \methDIALmu{} achieve strong held-out ranking agreement even at small human-label budgets, where the unrestricted \methHuman{} fit is not finite in every split (\Cref{fig:real-calibration}(a)).
Where its rank-aggregation stage is finite, \methAtC{} gains $0.017$--$0.056$ in $\tau$ over \methHuman{} but stays within $0.013$ of \methDIALmu{}, measurably exceeding it only at the largest budget ($n_\human=3200$).
With fewer LLM judgments, \methConsensus{} remains nearly flat as human supervision grows, whereas \methDIALmu{} improves toward \methHuman{} by updating the imperfect LLM consensus (\Cref{fig:real-calibration}(b)).
Appendix results show that \methDIALmu{} remains more stable than fixed anchoring under anti-consensus corruption (\Cref{fig:real-curves}(c,f)), diagnose subgroup-specific failures of the global calibration space (\Cref{fig:real-diagnostics}(b)), and find similar human-label gains across smaller judge panels, with adaptive updating becoming more valuable as LLM data become scarce (\Cref{fig:panel-human-budget,fig:panel-llm-budget}).

\section{Discussion}
\label{sec:discussion}

DIAL treats position debiasing and human-preference alignment as distinct but complementary: removing display-order effects recovers cleaner LLM preferences, while limited human comparisons determine how those preferences should inform the human target.
When the learned LLM preference structure is informative, it can substantially reduce human-label requirements; when it is imperfect, adaptive weighting shifts reliance toward the human data, as reflected in both theory and experiments.
Key extensions include heterogeneous human preferences, richer position and tie models, post-selection inference, and cost-aware human comparison design.

\bibliographystyle{iclr2027_conference}
\bibliography{references}

\clearpage

\appendix
\counterwithin{theorem}{section}
\counterwithin{equation}{section}

\renewcommand{\thesection}{\Alph{section}}
\setcounter{table}{0}
\renewcommand{\thetable}{\thesection\arabic{table}}
\setcounter{figure}{0} 
\renewcommand\thefigure{\thesection\arabic{figure}}
\setcounter{algorithm}{0}
\renewcommand{\thealgorithm}{\thesection.\arabic{algorithm}}
\renewcommand{\theHalgorithm}{\thesection.\arabic{algorithm}}

\begin{center}
\Large
{\bf Appendix}
\end{center}

This serves as an appendix to the main paper.
Below, we provide an outline for the appendix along with a summary of the notation used in the main paper and the appendix.

\paragraph{Organization.}
The content of the appendix is organized as follows.

\begin{table}[!ht]
\centering
\small
\begin{tabularx}{\textwidth}{@{}l l P@{}}
    \toprule
    \multicolumn{2}{c}{\textbf{Appendix}} & \textbf{Content} \\
    \midrule

    \Cref{app:related-comparison}
    & 
    & Detailed comparisons to related work. \\
    \midrule \addlinespace[0.5ex]

    \multirow{3}{*}{\Cref{app:proofs}}
    & \ref{app:proof:prop:ident-LLM}
    & Identification of position-debiased LLM scores and order effects (\Cref{prop:ident-LLM}). \\
    & \ref{app:debias-vs-average}
    & Debiased versus order-averaged probabilities. \\
    & \ref{app:stage-one}
    & Estimation of the position-debiased representation and judge diagnostics. \\
    \midrule \addlinespace[0.5ex]

    \multirow{4}{*}{\Cref{app:population-calibration-proofs}}
    & \ref{app:proof-prop:human_identification}
    & Identification and fixed-space risk preliminaries (\Cref{prop:human_identification}). \\
    & \ref{app:proof-thm:population-risk-tradeoff}
    & Proof of \Cref{thm:population-risk-tradeoff}. \\
    & \ref{app:incomplete-human-coverage}
    & Calibration under incomplete human comparison coverage. \\
    & \ref{app:human-calibration-additional}
    & Additional results for human calibration (\Cref{cor:staged-oracle-equivalence}). \\
    \midrule \addlinespace[0.5ex]

    \multirow{4}{*}{\Cref{app:gacv}}
    & \ref{app:gacv-reduced}
    & Reduced score-space representation and regularity. \\
    & \ref{app:weighted-asymptotic-normality}
    & Proof of \Cref{thm:weighted-asymptotic-normality}. \\
    & \ref{app:gacv-proof}
    & Proof of \Cref{thm:gacv}. \\
    & \ref{app:gacv-additional}
    & Additional results for adaptive calibration (\Cref{cor:gacv-loo-optimality} and \Cref{prop:calibration-test}). \\
    \midrule \addlinespace[0.5ex]

    \multirow{3}{*}{\Cref{app:experiments}}
    & \ref{app:dial-algorithm}
    & DIAL estimation procedure and adaptive GACV selection. \\
    & \ref{app:hac}
    & Variance estimation for the DIAL intervals (cluster-robust HAC sandwich). \\
    & \ref{app:hyperparameters}
    & Hyperparameter selection (rank $r$, weight $\lambda$). \\
    \midrule \addlinespace[0.5ex]

    \multirow{2}{*}{\Cref{app:simulation}}
    & \ref{app:simulation-design}
    & Simulation design, estimation, and metrics. \\
    & \ref{app:simulation-additional}
    & Larger-scale, misaligned-target, and heterogeneous-position simulations. \\
    \midrule \addlinespace[0.5ex]

    \multirow{4}{*}{\Cref{app:sec:real}}
    & \ref{app:subsec:real-evaluation}
    & Datasets and collection. \\
    & \ref{app:subsec:real-protocol}
    & Protocol and estimators. \\
    & \ref{app:subsec:real-position}
    & Position robustness evaluations and diagnostics. \\
    & \ref{app:subsec:real-efficiency}
    & Human-label efficiency and adaptive calibration. \\

    \bottomrule
\end{tabularx}
\end{table}
\addcontentsline{toc}{part}{\appendixname}

\paragraph{Notation.}
For $d\geq1$, let $[d]:=\{1,\ldots,d\}$, let $e_j$ denote the $j$th standard basis vector in $\RR^d$, and let $\zero_d$ and $\one_d$ denote the zero and all-ones vectors; dimension subscripts are omitted when clear.
For a set $\cS$, $|\cS|$ denotes its cardinality and $\ind_{\{\cdot\}}$ the indicator function.
For a matrix $A$, $\col(A)$, $\row(A)$, and $\rank(A)$ denote its column space, row space, and rank, while $\|A\|_{\mathrm F}$ and $\|A\|_{\mathrm{op}}$ denote its Frobenius and operator norms.
We write $\sigma(x):=(1+e^{-x})^{-1}$ and $\logit(p):=\log\{p/(1-p)\}$ for the logistic function and its inverse.
We use $o,\cO$ and $\op,\Op$ for deterministic and stochastic order notation, and $\pto$ and $\dto$ for convergence in probability and distribution.
Finally, $\operatorname{KL}(p\,\|\,q)$ denotes the Kullback-Leibler divergence from $p$ to $q$.

\clearpage

\section{Detailed Comparison with Closely Related Methods}
\label{app:related-comparison}

This section separates DIAL's contribution from established components: order-effect paired-comparison modeling and consensus-disagreement structure are prior ingredients, while DIAL links the resulting position-debiased multi-judge representation to limited-supervision human calibration and adaptive weighting.

\paragraph{Position bias and its identification.}
The closest work on position bias is \citet{wang2024fair}, which shows that pairwise LLM judgments can change substantially under response-order permutations and develops calibration-based mitigation.
Our position model instead treats display order as a comparison-varying covariate in a paired-comparison likelihood and separates the latent score contrast \(s_{ki}-s_{kj}\) from the nuisance order effect \(ab_k\).
The resulting position-debiased preference is therefore defined by \(s_k\), not by averaging judgments across display orders.
As shown in \Cref{app:debias-vs-average}, probability-scale order averaging retains dependence on \(b_k\) and generally differs from the position-debiased probability, though the two induce the same pairwise ordering under our model.
Balanced Position Calibration \citep{wang2024fair} instead queries both display orders and aggregates the resulting evaluations.
Under balanced querying, \methDIALnoPos{} provides our closest order-agnostic comparison using both display orders, although its BTL aggregation differs from BPC's score averaging; BPC itself requires the swapped judgment and is unavailable in the one-sided stress designs.

The identifiability results of \citet{xu2026debias} are relevant because they characterize when bias terms can or cannot be separated from unrestricted item effects.
Their negative results apply when the bias contribution is confounded with item quality, and they develop additional comparison designs that supply the missing identifying information.
DIAL instead treats the order indicator as a comparison-varying design covariate and identifies the order effect when its design column is linearly independent of the item-difference design, as formalized by \Cref{cond:llm-design}.

\paragraph{Structured preferences across LLM judges.}
Judge heterogeneity has also been modeled independently of position bias.
\citet{jin2020rank} model user-specific accuracy levels, while \citet{xu2026judgeaware} introduce judge-specific discrimination parameters around a common latent ranking, capturing heterogeneity in how strongly judges respond to the shared ranking signal.
Most directly, \citet{yu2026heterogeneous} proposes the consensus-plus-heterogeneity decomposition \(S=\gamma\mu^\top+UV^\top\), separating a consensus preference direction, judge-specific sensitivity to that direction, and structured residual disagreement.
DIAL adopts this structured representation for the \emph{position-debiased} LLM score matrix.
After estimating and removing order effects, DIAL calibrates within either \(\operatorname{span}(\mu)\) or the full item-side space \(\col(W)=\operatorname{span}(\mu)\oplus\col(V)\).
Accordingly, HJA targets interpretable decomposition and uncertainty quantification of heterogeneous judge preferences, whereas DIAL carries the learned representation into a separate human-alignment problem with population calibrated target $Mc_M^\star$, where $c_M^\star\in\argmin_cR_{\human}(Mc)$.

\paragraph{Human-supervised calibration and statistical evaluation.}
Several methods use human judgments to improve LLM-based evaluation:
HD-Eval \citep{liu2024hdeval} learns human-aligned evaluation criteria and aggregation rules, while SAJA \citep{kola2026saja} uses a small human-labeled sample to calibrate features extracted from LLM evaluations toward human-aligned scores.
In contrast, DIAL formulates the human target as a latent score and asks how limited pairwise human comparisons identify and estimate that score through a preference representation learned from LLM judges.

\paragraph{Aggregated evaluation functional estimation.}
A complementary statistical literature uses abundant automatic labels together with a smaller human reference sample to estimate evaluation quantities more efficiently.
\citet{boyeau2025autoeval}, \citet{chen2026efficient}, and \citet{lee2026correctly} study statistically principled estimation and uncertainty quantification for targets such as aggregate performance measures, benchmark scores, pairwise win rates, or evaluation accuracy, while \citet{divekar2026precise} considers ranking metrics such as Precision@\(K\).
These methods primarily improve the estimation of prespecified evaluation functionals.
DIAL instead estimates the latent item-level preference \(s_{\human}\), or under approximation its calibrated counterpart \(Mc_M^\star\), from which pairwise probabilities and rankings are induced.
As shown in \Cref{thm:population-risk-tradeoff}, restricting human estimation to the learned LLM preference space trades approximation error against the estimation error.

Finally, AtC \citep{xie2026atc} first aggregates human comparisons into a consensus ordering and then isotonic-calibrates predictive scores to respect that ordering.
It is the closest stage-matched alternative for our human-calibration stage, but takes a predictive score vector and a human-derived ordering as inputs, instead of jointly estimating position effects and multi-judge preference structure.
In \Cref{subsec:real-data}, we therefore instantiate \methAtC{} using pooled BTL rank aggregation and the same LLM-only position-debiased consensus score as DIAL.
DIAL instead learns the position-debiased LLM preference space \(\col(W)\) and calibrates a chosen \(\col(M)\subseteq\col(W)\) from limited human comparisons.
The weighted estimator in \Cref{sec:estimation} further allows the amount of LLM anchoring to vary with human-LLM agreement, interpolating between unrestricted human-only estimation and staged LLM-anchored calibration.

\clearpage
\section{Supporting Results for Position Debiasing}
\label{app:proofs}
\subsection[Proof of Proposition \ref{prop:ident-LLM}]{Proof of \Cref{prop:ident-LLM}}\label{app:proof:prop:ident-LLM}
\begin{proof}[Proof of \Cref{prop:ident-LLM}]
    Fix $k$ and let $\delta:=s_k-\widetilde s_k$ and $c:=b_k-\widetilde b_k$ for two centered parameter values inducing the same comparison probabilities.
    Since the logistic link is one-to-one,
    \[
    [D_k,z_k]
    \begin{pmatrix}
    \delta\\
    c
    \end{pmatrix}
    =0,
    \qquad
    \one^\top\delta=0.
    \]
    Because $D_k\one_N=0$, $(\one_N^\top,0)^\top$ always belongs to $\ker([D_k,z_k])$.
    Hence centered identification holds if and only if this is the entire null space, equivalently $\rank([D_k,z_k])=N$; if the rank is smaller, centering the first block of any additional null vector gives a nonzero centered perturbation with the same comparison probabilities.
    Finally, $\rank([D_k,z_k])=N$ if and only if $\rank(D_k)=N-1$ and $z_k\notin\col(D_k)$.
    For an incidence matrix, $\rank(D_k)=N-1$ is equivalent to connectivity of $\mathcal G_k$, which proves the equivalence.
    If a pair is observed under both orders, its two rows in $D_k$ coincide while the corresponding entries of $z_k$ are opposite, so $z_k\notin\col(D_k)$.
    Applying the argument judge-wise proves identification of $(S,b)$ under \Cref{cond:llm-design}.
\end{proof}

\begin{corollary}[Cycle criterion for the order effect]
\label[corollary]{cor:cycle-criterion}
    Suppose $\mathcal G_k$ is connected.
    Then $b_k$ is identifiable if and only if some cycle of the comparison multigraph $\{(i_t,j_t)\}_{t=1}^{m_k}$, traversed with signs, has nonzero signed sum of the order labels $a_t$.
\end{corollary}

\begin{proof}[Proof of \Cref{cor:cycle-criterion}]
    By \Cref{prop:ident-LLM} it suffices to show that $z_k\in\col(D_k)$ if and only if every signed cycle sum vanishes.
    An element of $\col(D_k)$ has the form $(x_{i_t}-x_{j_t})_{t=1}^{m_k}$ for some $x\in\RR^N$, that is, an edge function that is the difference of a potential, and a function on the oriented edges of a graph is a potential difference if and only if its signed sum around every cycle is zero.
    A pair observed under both orders contributes a cycle of length two with signed sum $2$, which recovers the sufficient condition stated after \Cref{prop:ident-LLM}.
\end{proof}

\begin{proposition}[Identification under the low-rank structure]
\label[proposition]{prop:ident-lowrank}
    Suppose \Cref{cond:llm-structure} holds.
    If $\col(W)$ is known, then writing $s_k=W\alpha_k$ the parameter $(\alpha_k,b_k)$ of judge $k$ is identifiable if and only if $\rank([D_kW,z_k])=r+2$.
    If $\col(W)$ is unknown, a necessary condition for identifiability of $(S,b)$ is the parameter count $\sum_{k=1}^Km_k\geq(r+1)(K+N-r-2)+K$, and it is sufficient that some subset $\mathcal K_0$ of judges satisfies $\rank([D_k,z_k])=N$ for every $k\in\mathcal K_0$ together with $\rank(S_{\mathcal K_0\largecdot})=r+1$, in which case each remaining judge needs only $\rank([D_kW,z_k])=r+2$.
\end{proposition}

\begin{proof}[Proof of \Cref{prop:ident-lowrank}]
    With $\col(W)$ known the linear predictor is $D_kW\alpha_k+z_kb_k$, which determines $(\alpha_k,b_k)$ if and only if $[D_kW,z_k]$ has full column rank $r+2$; no centering constraint is required because $\one_N^\top W=\zero_{r+1}$.
    For unknown $\col(W)$, removing the calibrated score from \eqref{eq:dim-Mr} leaves $\dim\{(S,b)\}=(r+1)(K+N-r-2)+K$, and identification requires at least as many comparison cells.
    For sufficiency, \Cref{prop:ident-LLM} identifies $s_k$ for every $k\in\mathcal K_0$, the rank condition gives $\col(W)=\operatorname{row}(S_{\mathcal K_0\largecdot})$, and the first part applies to each remaining judge.
\end{proof}

A judge with too few comparisons to identify its own score vector can still be identified through the shared space, so the low-rankness in \Cref{cond:llm-structure} contributes to both identification and efficiency.

\subsection{Debiased versus Order-Averaged Probabilities}
\label{app:debias-vs-average}

Throughout this section, fix a judge $k$ and suppress it from the notation: write $s_i$, $b$, $\mathcal G$, and $p^{(a)}_{ij}$ for $s_{ki}$, $b_k$, $\mathcal G_k$, and $p^{(a)}_{kij}$, and let $\Delta_{ij}:=s_i-s_j$.
We distinguish two mathematically related but conceptually different quantities.

\begin{enumerate}[(i)]
    \item \textbf{Position-debiased latent preference}, which defines the comparison probability by removing the order term from the fitted linear predictor and retaining only the latent score difference:
    \begin{equation*}
        p^{\mathrm{deb}}_{ij}
        :=
        \sigma(\Delta_{ij})
        =
        \sigma\!\left(s_i-s_j\right).
    \end{equation*}
    $s$ represents the latent preference component separated from the nuisance order effect, so that the position-debiased ranking of the judge is obtained directly from $s$,
    \(i \succ j\) iff \(s_i>s_j\) iff \(p^{\mathrm{deb}}_{ij}>\frac12\).

    \item \textbf{Order-averaged protocol probability}, which averages the two order-specific probabilities:
    \begin{equation*}
        p^{\mathrm{avg}}_{ij}
        :=
        \frac12\left\{
            p^{(1)}_{ij}+p^{(-1)}_{ij}
        \right\}
        =
        \frac12\left\{
            \sigma(\Delta_{ij}+b)
            +
            \sigma(\Delta_{ij}-b)
        \right\}.
    \end{equation*}
    Unlike $p^{\mathrm{deb}}_{ij}$, it continues to depend on the magnitude of the judge's order effect.
\end{enumerate}

DIAL uses \eqref{eq:binomial-model-llm}, equivalently $p^{\mathrm{deb}}_{ij}$, as the position-debiased latent preference of the judge, instead of $p^{\mathrm{avg}}_{ij}$.
Proposition~\ref{prop:debiasing-relation} characterizes the connection between the two quantities.

\begin{proposition}[Debiased versus order-averaged probabilities]
\label[proposition]{prop:debiasing-relation}
    For a fixed pair $i<j$, the following hold:
    \begin{enumerate}[(i), ref=\roman*]
        \item \label{eq:deb-logodds} The position-debiased log-odds equal the average of the two order-specific log-odds:
        \[
            \logit\!\left(p^{\mathrm{deb}}_{ij}\right)
            =
            \Delta_{ij}
            =
            \frac12\left\{
                \logit(p^{(1)}_{ij})
                +
                \logit(p^{(-1)}_{ij})
            \right\},
        \quad
            b
            =
            \frac12\left\{
                \logit(p^{(1)}_{ij})
                -
                \logit(p^{(-1)}_{ij})
            \right\}.
        \]
        Thus $p^{\mathrm{deb}}_{ij}$ averages the two fitted comparisons on the log-odds scale, whereas $p^{\mathrm{avg}}_{ij}$ averages them on the probability scale.

        \item \label{eq:deb-attenuation}
        For finite $\Delta_{ij}$, if $b\neq0$ and $\Delta_{ij}\neq0$,
        \[
            \left|
            p^{\mathrm{avg}}_{ij}-\frac12
            \right|
            <
            \left|
            p^{\mathrm{deb}}_{ij}-\frac12
            \right|.
        \]
        Thus, probability-scale order averaging attenuates preferences toward $1/2$ whenever a nonzero order effect is present.

        \item \label{eq:deb-monotone} For every finite $b$, $p^{\mathrm{avg}}_{ij}$ is a strictly increasing function of $\Delta_{ij}$ and satisfies
        \(
            p^{\mathrm{avg}}_{ij}=\frac12
            \ \Longleftrightarrow\ 
            \Delta_{ij}=0.
        \)
        Consequently,
        \[
            p^{\mathrm{avg}}_{ij}>\frac12
            \quad\Longleftrightarrow\quad
            p^{\mathrm{deb}}_{ij}>\frac12
            \quad\Longleftrightarrow\quad
            s_i>s_j.
        \]
    \end{enumerate}
\end{proposition}

\begin{proof}[Proof of \Cref{prop:debiasing-relation}]
    Let \(u=\tanh(\Delta_{ij}/2)\) and \(t=\tanh(b/2)\).
    For \eqref{eq:deb-logodds}, the claim follows immediately from
    $\logit(p^{(1)}_{ij})=\Delta_{ij}+b$ and
    $\logit(p^{(-1)}_{ij})=\Delta_{ij}-b$.
    For \eqref{eq:deb-attenuation}, using
    $\sigma(x)=\{1+\tanh(x/2)\}/2$ and the addition formula for
    $\tanh$ gives
    \begin{align}
        p^{\mathrm{avg}}_{ij}-\frac12
        =
        \frac{u}{2}
        \frac{1-t^2}{1-u^2t^2}
        =
        \left(p^{\mathrm{deb}}_{ij}-\frac12\right)
        \frac{1-t^2}{1-u^2t^2}.\label{eq:p-avg-p-deb}
    \end{align}
    For finite $\Delta_{ij}$ and nonzero $b$, we have
    $0<t^2<1$ and $u^2<1$, so
    \[
        0<
        \frac{1-t^2}{1-u^2t^2}
        <1.
    \]
    Finally, for \eqref{eq:deb-monotone},
    \(
        \frac{\partial p^{\mathrm{avg}}_{ij}}{\partial\Delta}
        =
        \frac12\left\{
            \sigma'(\Delta+b)+\sigma'(\Delta-b)
        \right\}>0
    \)
    and symmetry gives $p^{\mathrm{avg}}_{ij}=1/2$ at $\Delta_{ij}=0$.
\end{proof}

As shown in \Cref{prop:debiasing-relation}, although $p^{\mathrm{avg}}_{ij}$ and $p^{\mathrm{deb}}_{ij}$ induce the same pairwise ordering under the model, they generally differ in probability magnitude and calibration.
The attenuation factor $(1-t^2)/(1-u^2t^2)$ in \eqref{eq:p-avg-p-deb} depends on the pair through $u=\tanh(\Delta_{ij}/2)$, so the matrix of order-averaged log-odds need not inherit the rank-$(r+1)$ structure that \Cref{cond:llm-structure} imposes on $S$.
By contrast $p^{\mathrm{deb}}_{ij}$ has a clean reading: its odds $\exp({\Delta_{ij}})$ are the geometric mean of the two order-specific odds $\exp({\Delta_{ij}+b})$ and $\exp({\Delta_{ij}-b})$.
Although order averaging preserves pairwise ordering, its log-odds are generally not exact score differences, so fitting a BTL model to the averaged probabilities is generically misspecified.

\begin{proposition}[Order averaging is not BTL]
\label[proposition]{prop:avg-nonadditive}
    If $b\neq0$ and $\mathcal G$ contains a triangle, then for $s$ outside a Lebesgue-null set there is no $\widetilde s\in\RR^N$ with $\logit(p^{\mathrm{avg}}_{ij})=\widetilde s_i-\widetilde s_j$ on every compared pair.
\end{proposition}

\begin{proof}[Proof of \Cref{prop:avg-nonadditive}]
    Let
    \[
    g_b(\Delta)
    :=
    \logit
    \!\left[
    \tfrac12\{\sigma(\Delta+b)+\sigma(\Delta-b)\}
    \right].
    \]
    By \Cref{prop:debiasing-relation}\eqref{eq:deb-monotone} and symmetry, $g_b$ is odd, strictly increasing, and real analytic.
    For $b\neq0$,
    \[
    g_b'(0)=4\sigma(b)\{1-\sigma(b)\}<1,
    \qquad
    \lim_{\Delta\to\infty}g_b'(\Delta)=1,
    \]
    so $g_b$ is nonlinear.
    By \Cref{cor:cycle-criterion}, the edge values $g_b(\Delta_{ij})$ are score differences only if their signed sum vanishes on every cycle.
    On a triangle, this requires
    \[
    F(x,y):=g_b(x+y)-g_b(x)-g_b(y)=0.
    \]
    Since $F$ is real analytic and is not identically zero because $g_b$ is nonlinear, its zero set has Lebesgue measure zero.
\end{proof}

Proposition~\ref{prop:avg-nonadditive} is an exact representability result and does not imply that the best BTL approximation must have large error.
Indeed, for small $b$,
\[
g_b(\Delta)
=
\Delta-\frac{b^2}{2}\tanh(\Delta/2)+O(b^4),
\]
while locally in $\Delta$, $g_b(\Delta)=\operatorname{sech}^2(b/2)\Delta+\cO(\Delta^3)$.
Thus, balanced order averaging can behave approximately as a rescaling of the latent score over the range of contrasts encountered in a dataset, even though no single score vector reproduces the averaged probabilities exactly.

\subsection{Estimation of the Position-Debiased Representation}
\label{app:stage-one}

The LLM-only endpoint $\widehat\theta_\infty$ of \eqref{eq:staged-calibrated-estimator} is a maximum-likelihood estimator in the structured model of \Cref{cond:llm-structure}.
Its limit law supplies both the stability used by \Cref{cor:staged-oracle-equivalence} and the judge diagnostics for the calibrated preference.
In particular, \Cref{prop:stage-one-normality} yields Wald statistics for $H_0:b_k=0$, testing whether judge $k$ shows no order effect, and for $H_0:b_k=b_{k'}$, testing whether two judges share one, together with intervals for the consensus loadings $\gamma_k$; all are computed from the empirical inverse information at $\widehat\xi_\infty$ at the $\sqrt{n_{\llm}}$ scale.
Both hypotheses are basis-invariant: $b$ is identified by \Cref{prop:ident-LLM}, and $\gamma=N^{-1}S\mu$ is identified by the normalization in \Cref{cond:llm-structure}.
The argument parallels the fixed-rank likelihood analysis of \citet[Theorem~3]{yu2026heterogeneous}, with the judge-wise comparison design augmented by the order indicator.

\begin{proposition}[Stage-one estimation]
\label[proposition]{prop:stage-one-normality}
    Suppose that $N$, $K$, and $r$ are fixed, that the LLM comparison model \eqref{eq:binomial-model-llm} and \Cref{cond:llm-structure} hold, that $n_k/n_{\llm}\to\kappa_k>0$ and $n_{kij}^{(a)}/n_k\to\rho_{kij}^{(a)}$ with each limiting within-judge support satisfying \Cref{cond:llm-design}, and that every limiting cell probability lies in $(0,1)$.
    Then, with probability tending to one, $\widehat\theta_\infty$ is a regular minimizer of $\ell_{\llm}$, and in any smooth coordinate $\xi$ for the regular structured model
    \[
    \sqrt{n_{\llm}}\left(\widehat\xi_\infty-\xi_0\right)
    \dto
    \cN\left(0,\mathcal I_\xi^{-1}\right),
    \qquad
    \mathcal I_\xi:=\nabla_\xi^2R_{\llm}\{S(\xi),b(\xi)\}\big|_{\xi=\xi_0},
    \]
    with $\mathcal I_\xi\succ0$ by \Cref{cond:llm-design}.
    Consequently, $\widehat b_\infty$, $\widehat\gamma_\infty$, and $\widehat S_\infty$ are $\sqrt{n_{\llm}}$-consistent and asymptotically normal, and the estimated item-side space satisfies
    \begin{equation}
    \left\|P_{\col(\widehat W_\infty)}-P_{\col(W)}\right\|_{\mathrm F}
    =
    \Op(n_{\llm}^{-1/2}),
    \label{eq:subspace-rate}
    \end{equation}
    where $P_{\col(\cdot)}$ denotes orthogonal projection onto the indicated space.
\end{proposition}

\begin{proof}[Proof of \Cref{prop:stage-one-normality}]
    By \Cref{prop:ident-LLM} and \Cref{cond:llm-design}, equality of the limiting LLM linear predictors implies equality of the centered parameter $(S,b)$.
    Hence, the limiting risk $R_{\llm}$ is uniquely minimized at the truth, and in any smooth local coordinate $\xi$ on the regular structured model, its Fisher information $\mathcal I_\xi$ is positive definite.
    The LLM-side sublevel-compactness argument in the proof of \Cref{lem:weighted-consistency}, with the human component omitted, gives with probability tending to one a finite regular minimizer satisfying $\widehat\xi_\infty\pto\xi_0$.
    The Lindeberg-Feller score CLT, information identity, and Taylor expansion of the likelihood equation then give
    \[
    \sqrt{n_{\llm}}\left(\widehat\xi_\infty-\xi_0\right)
    \dto
    \cN\left(0,\mathcal I_\xi^{-1}\right).
    \]
    The asymptotic normality of $\widehat b_\infty$, $\widehat\gamma_\infty$, and $\widehat S_\infty$ follows by the delta method.
    Finally, \Cref{cond:llm-structure} gives $\col(W)=\row(S)$, and the map $S\mapsto P_{\row(S)}$ is smooth on the rank-$(r+1)$ stratum, so another application of the delta method yields \eqref{eq:subspace-rate}.
\end{proof}

\clearpage
\section{Supporting Theory for Human Preference Calibration}
\label{app:population-calibration-proofs}

This section provides proofs and supporting results for human-preference calibration.

\subsection{Identification and Fixed-Space Risk Preliminaries}
\label{app:proof-prop:human_identification}
\label{app:population-alignment}

\begin{proof}[Proof of \Cref{prop:human_identification}]
    Under \Cref{cond:score-normalization,cond:llm-design}, \Cref{prop:ident-LLM} identifies $(S,b)$, and \Cref{cond:llm-structure} identifies the calibration space $\col(M)$ with $\rank(M)=d$ and $\one_N^\top M=0$.
    If $c_{\human}$ and $\widetilde c_{\human}$ induce the same human probabilities, then
    \[
    D_{\human}M(c_{\human}-\widetilde c_{\human})=0,
    \]
    so \Cref{cond:human-design} gives $c_{\human}=\widetilde c_{\human}$.
    Conversely, failure of \Cref{cond:human-design} gives a nonzero coordinate perturbation with the same linear predictor.
    Since $M$ has full column rank,
    \[
    \rank(D_{\human}M)=d
    \quad\Longleftrightarrow\quad
    \ker(D_{\human})\cap\col(M)=\{0\}.
    \]
    If $\cG_{\human}$ is connected, $\ker(D_{\human})=\operatorname{span}(\one_N)$, so the condition follows from $\one_N^\top M=0$.
    Finally, $Mc_{\human}$ is unchanged by the admissible basis transformation when $M=W$, while $M=\mu$ has no factor-basis ambiguity.
\end{proof}

For any centered score vector $s\in\RR^N$, define the population human negative log-likelihood
\begin{equation}
R_{\human}(s)
:=
\sum_{(i,j)\in\Omega_{\human}^{\infty}}
\rho_{ij}
\left[
\log\{1+\exp(s_i-s_j)\}
-
p_{\human ij}(s_i-s_j)
\right].
\label{eq:population-human-risk}
\end{equation}

For comparison with the full LLM preference space, define
\begin{equation}
\Delta_\mu
:=
\inf_{\alpha\in\RR}
\left\{
R_{\human}(\alpha\mu)-R_{\human}(s_{\human})
\right\},
\label{eq:population-alignment-gap-mu}
\end{equation}
and let $\Delta_W$ be \eqref{eq:population-alignment-gaps} at $M=W$; write $\delta_\mu,\delta_W$ for their local limits.
For every centered $s\in\RR^N$,
\begin{equation}
R_{\human}(s)-R_{\human}(s_{\human})
=
\sum_{(i,j)\in\Omega_{\human}^{\infty}}
\rho_{ij}
\operatorname{KL}
\left\{
\operatorname{Bernoulli}(p_{\human ij})
\;\middle\|\;
\operatorname{Bernoulli}\!\left(\sigma(s_i-s_j)\right)
\right\}.
\label{eq:population-risk-kl}
\end{equation}
Hence, if $\cG_{\human}^{\infty}$ is connected,
$0\le\Delta_W\le\Delta_\mu$,
$\Delta_W=0$ if and only if $s_{\human}\in\col(W)$, and
$\Delta_\mu=0$ if and only if $s_{\human}\in\operatorname{span}(\mu)$.

\paragraph{A risk expansion on a fixed score space.}
Let $B\in\RR^{N\times q}$ have full column rank and satisfy $\one_N^\top B=0$.
The corresponding centered score space is $\col(B)$.
We say that the population comparison support identifies $\col(B)$ if
\begin{equation}
(e_i-e_j)^\top Bu=0
\ \text{ for all }(i,j)\in\Omega_{\human}^\infty
\quad\Longrightarrow\quad
u=0.
\label{eq:fixed-space-identification}
\end{equation}
For $c\in\RR^q$, define $R_B(c):=R_{\human}(Bc)$.
Under \eqref{eq:fixed-space-identification}, the finite-support logistic risk is coercive and strictly convex, so it has a unique minimizer $c_B^\star:=\argmin_{c\in\RR^q}R_B(c)$.
Let a generic human comparison sample $(I,J)\in\Omega_{\human}^{\infty}$ with probability $\rho_{ij}$ and then sample $Y\sim\operatorname{Bernoulli}(p_{\human IJ})$ conditional on $(I,J)$.
Define
\[
X_B:=B^\top(e_I-e_J),
\qquad
\pi_B:=\sigma(X_B^\top c_B^\star),
\qquad
\psi_B:=X_B(\pi_B-Y).
\]
The population Hessian and score-variance matrices are
\begin{equation}
H_B:=\EE\!\left[\pi_B(1-\pi_B)X_BX_B^\top\right],
\qquad
J_B:=\EE\!\left[\psi_B\psi_B^\top\right].
\label{eq:generic-HJ}
\end{equation}
The matrices $H_B$ and $J_B$ are the usual Hessian and score-variance quantities for maximum-likelihood estimation under possible model misspecification \citep{white1982maximum}.
Moreover, \eqref{eq:fixed-space-identification} implies $H_B\succ0$.

\paragraph{The comparison Laplacian.}
A single matrix governs all of the population human quantities used below.
Define the Fisher-weighted comparison Laplacian
\begin{equation}
L
:=
\sum_{(i,j)\in\Omega_{\human}^\infty}
\rho_{ij}\,p_{\human ij}(1-p_{\human ij})\,(e_i-e_j)(e_i-e_j)^\top
=
\nabla^2R_{\human}(s_{\human}),
\label{eq:comparison-laplacian}
\end{equation}
the Hessian of \eqref{eq:population-human-risk} at the truth.
It is a weighted graph Laplacian for $\cG_{\human}^\infty$: $L\succeq0$, $L\one_N=\zero_N$, and $\rank(L)=N-1$ if and only if $\cG_{\human}^\infty$ is connected.
Write $L^+$ for the Moore-Penrose inverse of $L$, $L^{1/2}$ for its positive semidefinite square root and $L^{+/2}:=(L^{1/2})^+$, $P:=I_N-N^{-1}\one_N\one_N^\top$, and $P_A:=A(A^\top A)^{-1}A^\top$ for $A$ of full column rank.
For a linear subspace $\mathcal V\subseteq\RR^N$, define
\[
\operatorname{dist}_L^2\{x,\mathcal V\}
:=
\inf_{v\in\mathcal V}(x-v)^\top L(x-v).
\]

\begin{lemma}[Laplacian form of the fixed-space quantities]
\label[lemma]{lem:comparison-laplacian}
    Let $B\in\RR^{N\times q}$ have full column rank with $\one_N^\top B=0$ and satisfy \eqref{eq:fixed-space-identification}.
    \begin{enumerate}[(i), ref=\roman*]
    \item \label{eq:laplacian-trace} If the restricted model is correctly specified, then $H_B=J_B=B^\top LB$ in \eqref{eq:generic-HJ}, and $\operatorname{tr}(H_B^{-1}J_B)=q$.

    \item \label{eq:laplacian-projection} Suppose in addition that $\cG_{\human}^\infty$ is connected.
    If $\col(B_{\human})$ is the whole centered score space, then $B_{\human}(B_{\human}^\top LB_{\human})^{-1}B_{\human}^\top=L^+$, while
    \[
    B(B^\top LB)^{-1}B^\top
    =
    L^{+/2}P_{L^{1/2}B}L^{+/2}
    \preceq L^+,
    \]
    with equality if and only if $q=N-1$.
    The corresponding projection ranks are $\rank(P)=N-1$ and $\rank(P_{L^{1/2}B})=q$.

    \item \label{eq:laplacian-risk-expansion}
    Suppose in addition that $\cG_{\human}^{\infty}$ is connected, and let
    $\epsilon_B:=\inf_{c\in\RR^q}\|Bc-s_{\human}\|_2$.
    Along any local sequence with fixed $\rho$ and comparison support,
    $p_{\human ij}\in[\varepsilon,1-\varepsilon]$ for fixed $\varepsilon\in(0,1/2)$, and
    $\epsilon_B\to0$,
    \[
    \inf_{c\in\RR^q}
    \{R_{\human}(Bc)-R_{\human}(s_{\human})\}
    =
    \frac12
    \operatorname{dist}_L^2\{s_{\human},\col(B)\}
    +
    O(\epsilon_B^3).
    \]
    \end{enumerate}
\end{lemma}

\begin{proof}[Proof of \Cref{lem:comparison-laplacian}]
    For \eqref{eq:laplacian-trace}, correct specification gives $\pi_B=p_{\human IJ}$ almost surely, so $H_B=\EE[p_{\human IJ}(1-p_{\human IJ})X_BX_B^\top]=B^\top LB$ by \eqref{eq:comparison-laplacian}, while $\EE[(\pi_B-Y)^2\mid I,J]=p_{\human IJ}(1-p_{\human IJ})$ gives $J_B=H_B$ and hence $\operatorname{tr}(H_B^{-1}J_B)=q$.
    
    For \eqref{eq:laplacian-projection}, note first that $B(B^\top LB)^{-1}B^\top$ is unchanged when $B$ is replaced by $BA$ for invertible $A$, so we may take $B_{\human}$ with orthonormal columns spanning the centered space.
    Then $B_{\human}B_{\human}^\top=P$, and $L=PLP$ gives $L=B_{\human}(B_{\human}^\top LB_{\human})B_{\human}^\top$ with $B_{\human}^\top LB_{\human}$ invertible by connectivity; the matrix $B_{\human}(B_{\human}^\top LB_{\human})^{-1}B_{\human}^\top$ then satisfies the four Moore-Penrose conditions for $L$, so it equals $L^+$.
    Next, $P_{L^{1/2}B}=L^{1/2}B(B^\top LB)^{-1}B^\top L^{1/2}$, and $L^{+/2}L^{1/2}=P$ with $PB=B$, so $L^{+/2}P_{L^{1/2}B}L^{+/2}=B(B^\top LB)^{-1}B^\top$; the same computation with $B$ replaced by $B_{\human}$ gives $L^+=L^{+/2}PL^{+/2}$.
    Since $\col(L^{1/2}B)\subseteq\col(P)$ we have $P-P_{L^{1/2}B}\succeq0$, and conjugating by $L^{+/2}$ preserves this, with equality exactly when $\col(L^{1/2}B)=\col(P)$, that is when $q=N-1$.
    
    For \eqref{eq:laplacian-risk-expansion}, connectivity and the uniform interior bound make $L$ uniformly positive definite on the centered subspace.
    Let $\bar s_B$ minimize $(s-s_{\human})^\top L(s-s_{\human})$ over $s\in\col(B)$.
    Norm equivalence gives $\|\bar s_B-s_{\human}\|_2=O(\epsilon_B)$.
    Since $\nabla R_{\human}(s_{\human})=0_N$, $\nabla^2R_{\human}(s_{\human})=L$, and the third derivatives are uniformly bounded,
    \[
    R_{\human}(\bar s_B)-R_{\human}(s_{\human})
    =
    \frac12\operatorname{dist}_L^2\{s_{\human},\col(B)\}
    +
    O(\epsilon_B^3).
    \]
    The restricted risk minimizer has no larger risk, and local strong convexity gives distance $O(\epsilon_B)$ from $s_{\human}$.
    Applying the same Taylor expansion there and using optimality of $\bar s_B$ for the quadratic term gives the reverse bound.
\end{proof}

Taking $B=M$ in \Cref{lem:comparison-laplacian} gives the calibration-space variance reduction and the corresponding local expansion for $\Delta_M$.

For the following fixed-space proofs, write $\ell_{\human,B}(c):=\ell_{\human}(Bc)$, with all derivatives taken with respect to $c$.
The estimation regularity that the risk expansion requires is not an assumption but a consequence of the finite design.

\begin{lemma}[Existence, consistency, and moment bounds on a fixed score space]
\label[lemma]{lem:fixed-space-existence}
    Suppose \eqref{eq:fixed-space-identification} holds for $B$, that $\Omega_{\human}^\infty$ is finite with $\rho_{ij}>0$, and that $p_{\human ij}\in[\varepsilon,1-\varepsilon]$ for a fixed $\varepsilon\in(0,1/2)$.
    Let $\cE$ be the event that every cell of $\Omega_{\human}^\infty$ is observed at least once with each outcome, and take $\widehat c_B$ to be the minimizer of $\ell_{\human,B}$ on $\cE$ and $\widehat c_B:=0$ off $\cE$.
    Then there are constants $c,C>0$, depending on $p_{\human}$ only through $\varepsilon$, such that for every $n_{\human}$:
    \begin{enumerate}[(i), ref=\roman*]
    \item \label{eq:fse-event} $\Pr(\cE^c)\le Ce^{-cn_{\human}}$, and on $\cE$ the criterion $\ell_{\human,B}$ has a unique finite minimizer, which satisfies $\|\widehat c_B\|_2\le Cn_{\human}$;

    \item \label{eq:fse-tail} $\Pr\left\{\sqrt{n_{\human}}\,\|\widehat c_B-c_B^\star\|_2>t\right\}\le Ce^{-ct^2}$ for every $0\le t\le c\sqrt{n_{\human}}$;

    \item \label{eq:fse-ui} consequently $\widehat c_B\pto c_B^\star$, and $n_{\human}\|\widehat c_B-c_B^\star\|_2^2$ is uniformly integrable.
    \end{enumerate}
    Because the constants depend on $p_{\human}$ only through $\varepsilon$, all three statements hold uniformly over triangular arrays $\{p_{\human ij}^{(n_{\human})}\}$ obeying the same bound.
\end{lemma}

\begin{proof}[Proof of \Cref{lem:fixed-space-existence}]
    Write $x_{ij}:=B^\top(e_i-e_j)$.
    Call $\cA\subseteq\Omega_{\human}^\infty$ identifying if $x_{ij}^\top u=0$ for all $(i,j)\in\cA$ forces $u=0$; for such an $\cA$ the map $u\mapsto\max_{(i,j)\in\cA}|x_{ij}^\top u|$ is a norm on $\RR^q$, and we let $\kappa>0$ be the smallest of its minima over the unit sphere, taken over the finitely many identifying subsets.

    For \eqref{eq:fse-event}, a cell observed with both outcomes contributes at least $\log(1+e^{v})+\log(1+e^{-v})\ge|v|$ to $n_{\human}\ell_{\human,B}(c)$, where $v:=x_{ij}^\top c$, because $\log(1+e^{v})\ge\max(0,v)$.
    On $\cE$ this holds for every cell, so $n_{\human}\ell_{\human,B}(c)\ge\kappa\|c\|_2$, and the criterion is coercive; \eqref{eq:fixed-space-identification} makes it strictly convex, so the minimizer exists and is unique.
    Since $\ell_{\human,B}(\widehat c_B)\le\ell_{\human,B}(0)=\log2$, the same bound gives $\|\widehat c_B\|_2\le n_{\human}\log2/\kappa$.
    For the probability, $n_{\human ij}\sim\operatorname{Binomial}(n_{\human},\rho_{ij})$ and, conditionally on $n_{\human ij}$, the chance that a cell shows a single outcome is at most $2(1-\varepsilon)^{n_{\human ij}}$, so
    \[
    \Pr(\cE^c)
    \le
    \sum_{(i,j)\in\Omega_{\human}^\infty}
    \left\{
    \Pr(n_{\human ij}<\rho_{ij}n_{\human}/2)
    +
    2(1-\varepsilon)^{\rho_{ij}n_{\human}/2}
    \right\}
    \le
    Ce^{-cn_{\human}},
    \]
    since $\Omega_{\human}^\infty$ is finite and each $\rho_{ij}>0$.

    For \eqref{eq:fse-tail}, $\nabla\ell_{\human,B}(c_B^\star)$ is an average of $n_{\human}$ independent, mean-zero, uniformly bounded vectors, so Hoeffding's inequality gives $\Pr\{\sqrt{n_{\human}}\|\nabla\ell_{\human,B}(c_B^\star)\|_2>t\}\le Ce^{-ct^2}$.
    On a fixed ball $\|c-c_B^\star\|_2\le\delta$ the logistic weights are bounded below, and the empirical cell frequencies concentrate, so $\nabla^2\ell_{\human,B}\succeq\kappa'I_q$ there for some $\kappa'>0$ with probability at least $1-Ce^{-cn_{\human}}$.
    Convexity then gives $\|\widehat c_B-c_B^\star\|_2\le\kappa'^{-1}\|\nabla\ell_{\human,B}(c_B^\star)\|_2$ whenever the right-hand side is at most $\kappa'\delta$, which is the stated range of $t$.

    For \eqref{eq:fse-ui}, consistency follows from \eqref{eq:fse-tail}.
For uniform integrability, first note that $c_B^\star$ is uniformly bounded under the stated interior-probability condition. Indeed, for every $p\in[\varepsilon,1-\varepsilon]$ and $v\in\RR$,

$$
    \log(1+e^v)-pv\ge \varepsilon |v|.
$$

Hence, writing $\rho_{\min}:=\min_{(i,j)\in\Omega_{\human}^\infty}\rho_{ij}>0$ and using the definition of $\kappa$ above,

$$
\begin{aligned}
    R_B(c)
    &\ge
    \varepsilon
    \sum_{(i,j)\in\Omega_{\human}^\infty}
    \rho_{ij}|x_{ij}^\top c| \\
    &\ge
    \varepsilon\rho_{\min}\kappa\|c\|_2.
\end{aligned}
$$

Since $R_B(c_B^\star)\le R_B(0)=\log2$, it follows that

$$
    \|c_B^\star\|_2
    \le
    \frac{\log2}{\varepsilon\rho_{\min}\kappa}.
$$

Now set
\[
    Z_{n_{\human}}
    :=
    \sqrt{n_{\human}}\,
    \|\widehat c_B-c_B^\star\|_2,
    \qquad
    X_{n_{\human}}
    :=
    Z_{n_{\human}}^2
    =
    n_{\human}\|\widehat c_B-c_B^\star\|_2^2.
\]
We prove uniform integrability by showing that
$\{X_{n_{\human}}\}$ has uniformly bounded second moments.
Choose a fixed $a>0$ small enough that \eqref{eq:fse-tail} applies for
every $0\le t\le a\sqrt{n_{\human}}$.
Using the tail-integral representation,
\[
\begin{aligned}
    \EE\!\left[
    Z_{n_{\human}}^4
    \ind_{\{Z_{n_{\human}}\le a\sqrt{n_{\human}}\}}
    \right]
    &\le
    4\int_0^{a\sqrt{n_{\human}}}
    t^3\Pr(Z_{n_{\human}}>t)\,dt \\
    &\le
    C\int_0^\infty
    t^3e^{-ct^2}\,dt
    \le C.
\end{aligned}
\]

It remains to control the far tail.
On $\cE$, part \eqref{eq:fse-event} and the uniform bound on
$c_B^\star$ give
\[
    Z_{n_{\human}}
    \le
    Cn_{\human}^{3/2}.
\]
Hence, applying \eqref{eq:fse-tail} at
$t=a\sqrt{n_{\human}}$,
\[
\begin{aligned}
    \EE\!\left[
    Z_{n_{\human}}^4
    \ind_{\{Z_{n_{\human}}>a\sqrt{n_{\human}}\}}
    \ind_{\cE}
    \right]
    &\le
    Cn_{\human}^6
    \Pr\!\left(
    Z_{n_{\human}}>a\sqrt{n_{\human}}
    \right) \\
    &\le
    Cn_{\human}^6e^{-cn_{\human}}
    \le C.
\end{aligned}
\]
Off $\cE$, the convention $\widehat c_B=0$ and the uniform bound on
$c_B^\star$ imply
\[
    Z_{n_{\human}}^4
    =
    n_{\human}^2\|c_B^\star\|_2^4
    \le
    Cn_{\human}^2,
\]
and therefore
\[
    \EE\!\left[
    Z_{n_{\human}}^4
    \ind_{\cE^c}
    \right]
    \le
    Cn_{\human}^2e^{-cn_{\human}}
    \le C.
\]
Consequently,
\[
    \sup_{n_{\human}}
    \EE\!\left[
    X_{n_{\human}}^2
    \right]
    =
    \sup_{n_{\human}}
    \EE\!\left[
    Z_{n_{\human}}^4
    \right]
    <\infty.
\]
It follows that, for every $A>0$,
\[
\begin{aligned}
    \sup_{n_{\human}}
    \EE\!\left[
    X_{n_{\human}}
    \ind_{\{X_{n_{\human}}>A\}}
    \right]
    &\le
    \frac{1}{A}
    \sup_{n_{\human}}
    \EE[X_{n_{\human}}^2]
    \longrightarrow0
    \qquad (A\to\infty).
\end{aligned}
\]
Thus
$n_{\human}\|\widehat c_B-c_B^\star\|_2^2$
is uniformly integrable.

\end{proof}

The convention in \Cref{lem:fixed-space-existence} modifies the empirical minimizer only on an event of probability at most $Ce^{-cn_{\human}}$, so it affects neither the limit distribution nor the $o(n_{\human}^{-1})$ risk expansion below.

\begin{lemma}[Risk expansion on a fixed score space]
\label[lemma]{lem:fixed-space-risk}
    Suppose that the conditions of \Cref{lem:fixed-space-existence} hold.
    Then
    \[
    \sqrt{n_{\human}}
    (\widehat c_B-c_B^\star)
    \dto
    \cN
    \left(
    0,
    H_B^{-1}J_BH_B^{-1}
    \right),
    \]
    and
    \begin{equation}
    \EE
    \left[
    R_{\human}(B\widehat c_B)
    -
    R_{\human}(Bc_B^\star)
    \right]
    =
    \frac{1}{2n_{\human}}
    \operatorname{tr}(H_B^{-1}J_B)
    +
    o(n_{\human}^{-1}).
    \label{eq:generic-risk-expansion}
    \end{equation}
    If the restricted model is correctly specified, then $H_B=J_B$ and the leading term in \eqref{eq:generic-risk-expansion} equals $q/(2n_{\human})$.
    The remainder is uniform over triangular arrays $\{p_{\human ij}^{(n_{\human})}\}$ with $p_{\human ij}^{(n_{\human})}\in[\varepsilon,1-\varepsilon]$ and $H_B\succeq\kappa I_q$ for fixed $\varepsilon,\kappa>0$.
\end{lemma}

\begin{proof}[Proof of \Cref{lem:fixed-space-risk}]
    Recall the local shorthand $\ell_{\human,B}(c)=\ell_{\human}(Bc)$, differentiated with respect to $c$.
    At the population minimizer,
    \[
    \EE\{\nabla\ell_{\human,B}(c_B^\star)\}
    =
    \nabla R_B(c_B^\star)
    =
    0.
    \]
    Because the pair support is finite and the outcomes are Bernoulli, the multivariate central limit theorem gives
    \[
    \sqrt{n_{\human}}
    \nabla\ell_{\human,B}(c_B^\star)
    \dto
    \cN(0,J_B).
    \]
    Let $\mathcal{A}_{n_{\human}}$ denote the event that a finite empirical minimizer exists.
    On $\mathcal{A}_{n_{\human}}$, a Taylor expansion of the empirical score around $c_B^\star$ gives
    \[
    0
    =
    \nabla\ell_{\human,B}(\widehat c_B)
    =
    \nabla\ell_{\human,B}(c_B^\star)
    +
    \left\{
    H_B+\op(1)
    \right\}
    (\widehat c_B-c_B^\star),
    \]
    where consistency and continuity of the logistic Hessian imply convergence of the empirical Hessian.
    Since $\Pr(\mathcal{A}_{n_{\human}})\to1$,
    \[
    \sqrt{n_{\human}}
    (\widehat c_B-c_B^\star)
    =
    -
    H_B^{-1}
    \sqrt{n_{\human}}
    \nabla\ell_{\human,B}(c_B^\star)
    +
    \op(1),
    \]
    which proves the stated asymptotic normality.

    Since $c_B^\star$ minimizes $R_B$, $\nabla R_B(c_B^\star)=0$.
    Let $u_B:=\widehat c_B-c_B^\star$ and define
    \[
    \overline H_B
    :=
    2\int_0^1
    (1-t)
    \nabla^2 R_B(c_B^\star+tu_B)
    \,dt.
    \]
    The integral form of Taylor's theorem gives
    \[
    R_B(\widehat c_B)-R_B(c_B^\star)
    =
    \frac12
    u_B^\top
    \overline H_B
    u_B.
    \]
    Because $\widehat c_B\pto c_B^\star$ and $\nabla^2R_B$ is continuous,
    \[
    \overline H_B\pto H_B.
    \]
    Moreover, the pair support is finite and $\sigma(x)\{1-\sigma(x)\}\leq 1/4$, so
    \[
    \sup_{c\in\RR^q}
    \left\|
    \nabla^2R_B(c)
    \right\|_{\mathrm{op}}
    <
    \infty.
    \]
    Hence, there exists some constant $C>0$ such that
    \[
    0
    \leq
    n_{\human}
    \left\{
    R_B(\widehat c_B)-R_B(c_B^\star)
    \right\}
    \leq
    C n_{\human}
    \|\widehat c_B-c_B^\star\|_2^2.
    \]
    By \Cref{lem:fixed-space-existence}\eqref{eq:fse-ui}, the right-hand side is uniformly integrable.
    Together with the stated asymptotic normality and $\overline H_B\pto H_B$, this yields
    \[
    n_{\human}
    \left\{
    R_B(\widehat c_B)-R_B(c_B^\star)
    \right\}
    \dto
    \frac12 Z^\top H_BZ,
    \qquad
    Z\sim
    \cN(0,H_B^{-1}J_BH_B^{-1}),
    \]
    and uniform integrability therefore gives
    \begin{align*}
    \lim_{n_{\human}\to\infty}
    n_{\human}
    \EE
    \left[
    R_B(\widehat c_B)
    -
    R_B(c_B^\star)
    \right]
    &=
    \frac12
    \EE(Z^\top H_BZ)
    \\
    &=
    \frac12
    \operatorname{tr}(H_B^{-1}J_B).
    \end{align*}
    This proves \eqref{eq:generic-risk-expansion}.

    If the restricted model is correctly specified, then $\pi_B=p_{\human IJ}$ almost surely.
    The information identity gives $J_B=H_B$, and therefore $\operatorname{tr}(H_B^{-1}J_B)=q$.
\end{proof}

\subsection[Proof of Theorem \ref{thm:population-risk-tradeoff}]{Proof of \Cref{thm:population-risk-tradeoff}}
\label{app:proof-thm:population-risk-tradeoff}

Throughout this subsection, fixed-space empirical minimizers use the measurable extension in \Cref{lem:fixed-space-existence}.
\begin{theorem}[Risk on a fixed score space]
\label[theorem]{thm:fixed-space-tradeoff}
    Let $N$ be fixed and let $\mathcal V$ be a fixed $q$-dimensional subspace of the centered score space on which the population comparison support is identifying in the sense of \eqref{eq:fixed-space-identification}.
    Let $H_{\mathcal V}$ and $J_{\mathcal V}$ denote \eqref{eq:generic-HJ} in any basis of $\mathcal V$.
    Write $\widetilde s_{\mathcal V}$ for the minimizer of $\ell_{\human}$ over $\mathcal V$ and $\Delta_{\mathcal V}:=\inf_{s\in\mathcal V}R_{\human}(s)-R_{\human}(s_{\human})$.
    Under the sampling scheme of \Cref{subsec:population-calibration},
    \begin{equation}
    \EE\left[R_{\human}(\widetilde s_{\mathcal V})-R_{\human}(s_{\human})\right]
    =
    \Delta_{\mathcal V}
    +
    \frac{1}{2n_{\human}}\operatorname{tr}(H_{\mathcal V}^{-1}J_{\mathcal V})
    +
    o(n_{\human}^{-1}).
    \label{eq:population-risk-general}
    \end{equation}
    Under correct specification the estimation term is $q/(2n_{\human})+o(n_{\human}^{-1})$.
    The whole centered score space gives $q=N-1$, while $\mathcal V=\col(M)$ gives $q=d$.
    If, along a local-mismatch sequence with fixed $\rho$, $p_{\human ij}^{(n_{\human})}\in[\varepsilon,1-\varepsilon]$ on $\Omega_{\human}^{\infty}$ for fixed $\varepsilon\in(0,1/2)$ and $n_{\human}\Delta_{\mathcal V}\to\delta_{\mathcal V}<\infty$, then
    \[
    n_{\human}\EE\{R_{\human}(\widetilde s_{\mathcal V})-R_{\human}(s_{\human})\}
    \to
    \delta_{\mathcal V}+\frac q2,
    \]
    so anchoring to $\mathcal V$ improves first-order risk exactly when $\delta_{\mathcal V}<(N-1-q)/2$.
\end{theorem}

\begin{proof}[Proof of \Cref{thm:fixed-space-tradeoff,thm:population-risk-tradeoff}]
    Let $B$ be any basis of $\mathcal V$, so that $\one_N^\top B=0$ and $B$ has full column rank $q$.
    Both $\widetilde s_{\mathcal V}=B\widehat c_B$ and $\inf_{s\in\mathcal V}R_{\human}(s)=R_B(c_B^\star)$ are invariant to replacing $B$ by $BA$ for invertible $A$.
    By the definition of $\Delta_{\mathcal V}$,
    \[
    \EE\left[R_{\human}(\widetilde s_{\mathcal V})-R_{\human}(s_{\human})\right]
    =
    \Delta_{\mathcal V}
    +
    \EE\left[R_B(\widehat c_B)-R_B(c_B^\star)\right],
    \]
    and \Cref{lem:fixed-space-risk} applied to $B$ gives \eqref{eq:population-risk-general}.
    If $s_{\human}\in\mathcal V$, then $\Delta_{\mathcal V}=0$ because $s_{\human}$ minimizes $R_{\human}$ globally, the restricted model is correctly specified, and $\operatorname{tr}(H_{\mathcal V}^{-1}J_{\mathcal V})=q$ by \Cref{lem:comparison-laplacian}\eqref{eq:laplacian-trace}.

    For the unrestricted endpoint, take $B_{\human}\in\RR^{N\times(N-1)}$ spanning the centered score space; connectivity of $\cG_{\human}^\infty$ gives \eqref{eq:fixed-space-identification}, correct specification, and $q=N-1$.
    For the anchored endpoint, take $B=M$ and $q=d$; \eqref{eq:fixed-space-identification} is equivalent to $\rank(D_{\human}^\infty M)=d$, as relaxed further in \Cref{thm:incomplete-human-coverage}.

    \paragraph{Local-mismatch regime.}
    Let $\pi_{ij}^\star:=\sigma\{(e_i-e_j)^\top Bc_B^\star\}$.
    By the KL identity \eqref{eq:population-risk-kl}, which does not require connectivity,
    \[
    \Delta_{\mathcal V}
    =
    \sum_{(i,j)\in\Omega_{\human}^{\infty}}
    \rho_{ij}
    \operatorname{KL}
    \left\{
    \operatorname{Bernoulli}(p_{\human ij})
    \;\middle\|\;
    \operatorname{Bernoulli}(\pi^\star_{ij})
    \right\}
    \ge
    0 .
    \]
    Hence $n_{\human}\Delta_{\mathcal V}\to\delta_{\mathcal V}<\infty$ implies $\Delta_{\mathcal V}\to0$ and therefore $\pi^\star_{ij}-p_{\human ij}\to0$ for every $(i,j)\in\Omega_{\human}^{\infty}$.
    Specializing \eqref{eq:generic-HJ} gives $J_{\mathcal V}-H_{\mathcal V}\to0$, while the uniform interior bound keeps $H_{\mathcal V}$ uniformly positive definite, so $\operatorname{tr}(H_{\mathcal V}^{-1}J_{\mathcal V})\to q$.
    The remainder in \eqref{eq:population-risk-general} is uniform along the sequence by the triangular-array clause of \Cref{lem:fixed-space-risk}, so multiplying by $n_{\human}$ gives
    \[
    n_{\human}\EE\left[R_{\human}(\widetilde s_{\mathcal V})-R_{\human}(s_{\human})\right]
    \longrightarrow
    \delta_{\mathcal V}+\frac{q}{2},
    \]
    and the same argument with $\mathcal V$ the whole centered score space gives the limit $(N-1)/2$.
    The former is smaller if and only if $\delta_{\mathcal V}<(N-1-q)/2$.
    Taking $\mathcal V=\col(M)$ and $q=d$ gives \Cref{thm:population-risk-tradeoff}.
\end{proof}

\subsection{Calibration under Incomplete Human Comparison Coverage}
\label{app:incomplete-human-coverage}

Under \eqref{eq:human-subspace-model}, connectivity of $\cG_{\human}^\infty$ can be replaced by identification of the $d$ directions in $\col(M)$.
Enumerate $\Omega_{\human}^\infty=\{(i_t,j_t)\}_{t=1}^{m_{\human}^\infty}$, let $D_{\human}^\infty$ have rows $(e_{i_t}-e_{j_t})^\top$, and write $C_{\human}$ for the number of connected components of $\cG_{\human}^\infty$, including isolated items, so $q_{\human}:=\rank(D_{\human}^\infty)=N-C_{\human}$.

\begin{theorem}[Calibration with incomplete human coverage]
\label[theorem]{thm:incomplete-human-coverage}
    Suppose \eqref{eq:human-subspace-model} holds and $\rank(D_{\human}^\infty M)=d$.
    Then:
    \begin{enumerate}[(i), ref=\roman*]
    \item \label{eq:ihc-identification} $c_{\human}$ and the basis-invariant score $s_{\human}=Mc_{\human}$ are identifiable from the population human probabilities, whereas an unrestricted centered score is not identifiable when $C_{\human}>1$.
    
    \item \label{eq:ihc-risk}
    Treating $M$ as known,
    \[
    \EE\!\left[R_{\human}(\widetilde s_\infty)-R_{\human}(s_{\human})\right]
    =
    \frac{d}{2n_{\human}}+o(n_{\human}^{-1}).
    \]
    Let $M_{\human}\in\RR^{N\times q_{\human}}$ span $\col\{(D_{\human}^\infty)^\top\}$ and set $\widehat s_0^\dagger:=M_{\human}\widehat c_{M_{\human}}$.
    Then
    \[
    \EE\!\left[R_{\human}(\widehat s_0^\dagger)-R_{\human}(s_{\human})\right]
    =
    \frac{q_{\human}}{2n_{\human}}+o(n_{\human}^{-1}).
    \]
    \item \label{eq:ihc-inference}
    Under the remaining conditions of \Cref{thm:weighted-asymptotic-normality}, with connectivity of $\cG_{\human}^\infty$ replaced by $\rank(D_{\human}^\infty M)=d$, the conclusion \eqref{eq:calibrated-asymptotic-normality} and the consistency of the sandwich estimator \eqref{eq:weighted-empirical-sandwich} continue to hold.
    \end{enumerate}
\end{theorem}

\begin{proof}[Proof of \Cref{thm:incomplete-human-coverage}]
    For \eqref{eq:ihc-identification}, equality of population probabilities gives
    $D_{\human}^\infty M(c_{\human}-\widetilde c_{\human})=0$, so the rank condition gives $c_{\human}=\widetilde c_{\human}$.
    If $C_{\human}>1$, then $\dim\ker(D_{\human}^\infty)=C_{\human}$, which contains a nonzero centered direction.
    
    For \eqref{eq:ihc-risk}, $\rank(D_{\human}^\infty M)=d$ is exactly \eqref{eq:fixed-space-identification} for $M$, so exact calibration and \Cref{lem:fixed-space-risk} give the first expansion.
    For $M_{\human}$, $D_{\human}^\infty$ is injective on $\col(M_{\human})$ and the restricted model reproduces $D_{\human}^\infty s_{\human}$; applying \Cref{lem:fixed-space-risk} with dimension $q_{\human}$ gives the second expansion.
    
    For \eqref{eq:ihc-inference}, \Cref{lem:h-lambda-pd} gives $\mathcal H_\lambda\succ0$ under $\rank(D_{\human}^\infty M)=d$, and \Cref{lem:weighted-consistency} then allows the proof of \Cref{thm:weighted-asymptotic-normality} to proceed unchanged.
\end{proof}

If comparisons cover only $\mathcal C\subset[N]$ and its induced graph is connected, then $\rank(D_{\human}^\infty M)=d$ is equivalent to $\rank(P_{\mathcal C}M_{\mathcal C})=d$, where $P_{\mathcal C}:=I_{|\mathcal C|}-|\mathcal C|^{-1}\one\one^\top$.
Hence $s_{\human}$ can remain identifiable on all $N$ items even when items outside $\mathcal C$ receive no human comparisons.

\subsection{Additional Results for Human Calibration}
\label{app:human-calibration-additional}

\paragraph{Effect of estimating the calibration space.}
The oracle analysis of \Cref{thm:population-risk-tradeoff} treats $M$ as fixed, whereas the staged endpoint \eqref{eq:staged-calibrated-estimator} estimates it from the LLM comparisons.
Under exact calibration, write $H_M:=M^\top LM$.

\begin{corollary}[Oracle equivalence of the staged endpoint]
\label[corollary]{cor:staged-oracle-equivalence}
    Suppose the exact-calibration case of \Cref{thm:population-risk-tradeoff} holds, \Cref{prop:stage-one-normality} applies, and $n_{\human},n_{\llm}\to\infty$ with $n_{\human}/n_{\llm}\to0$.
    Then
    \[
    \|\widehat s_\infty-\widetilde s_\infty\|_2
    =
    \Op(n_{\llm}^{-1/2}),
    \qquad
    \sqrt{n_{\human}}(\widehat s_\infty-s_{\human})
    \dto
    \cN(0,MH_M^{-1}M^\top),
    \]
    and
    \[
    R_{\human}(\widehat s_\infty)-R_{\human}(\widetilde s_\infty)
    =
    o_P(n_{\human}^{-1}).
    \]
\end{corollary}

\begin{proof}[Proof of \Cref{cor:staged-oracle-equivalence}]
    By \Cref{prop:stage-one-normality} and smoothness of $S\mapsto P_{\col(M_\theta)}$,
    \[
    \|P_{\col(\widehat M_\infty)}-P_{\col(M)}\|_{\mathrm F}
    =
    \Op(n_{\llm}^{-1/2}).
    \]
    Local Lipschitz continuity of the restricted human minimizer therefore gives
    $\|\widehat s_\infty-\widetilde s_\infty\|_2=\Op(n_{\llm}^{-1/2})$.
    Since $n_{\human}/n_{\llm}\to0$,
    $\sqrt{n_{\human}}(\widehat s_\infty-\widetilde s_\infty)\pto0$,
    while \Cref{lem:fixed-space-risk} gives the stated limit law for $\widetilde s_\infty$.
    
    Finally, \Cref{lem:fixed-space-risk} gives
    $\|\widetilde s_\infty-s_{\human}\|_2=\Op(n_{\human}^{-1/2})$.
    Since $\nabla R_{\human}(s_{\human})=0$ and $\nabla^2R_{\human}$ is bounded,
    \[
    |R_{\human}(\widehat s_\infty)-R_{\human}(\widetilde s_\infty)|
    =
    \Op\!\left((n_{\human}n_{\llm})^{-1/2}+n_{\llm}^{-1}\right)
    =
    o_P(n_{\human}^{-1}).
    \]
    The same limit is reached by $\widehat s_\lambda$ along any $\lambda\to\infty$ by \Cref{cor:uniform-lambda}.
\end{proof}

\paragraph{Ranking implication.}
Let $c_M^\star\in\argmin_c R_{\human}(Mc)$ and $e:=Mc_M^\star-s_{\human}$.
If $\min_{i\neq j}|s_{\human i}-s_{\human j}|>2\|e\|_\infty$, then $Mc_M^\star$ and $s_{\human}$ induce the same ranking; under \eqref{eq:human-subspace-model}, $e=\zero_N$.

\clearpage
\section{Estimation, Inference, and Adaptive Tuning}
\label{app:gacv}

\subsection{Reduced score-space representation and regularity}
\label{app:gacv-reduced}

The criterion in \eqref{eq:weighted-loss}, restricted to the space $\mathcal M_r$ in \eqref{eq:reduced-parameter-space}, is
\begin{equation}
Q_\lambda(s,S,b)
:=
\ell_{\human}(s)
+
\lambda\ell_{\llm}(S,b),
\qquad
(s,S,b)\in\mathcal M_r.
\label{eq:reduced-weighted-loss}
\end{equation}

Under \Cref{cond:llm-structure},
\[
S=[\gamma,U][\mu,V]^\top,\qquad
\rank(S)=r+1,\qquad
\operatorname{row}(S)=\col(W),
\]
with
\[
S\one_N=\zero_K,\qquad
S^\top\one_K=(\one_K^\top\gamma)\mu\neq\zero_N.
\]

Conversely, let $S\one_N=\zero_K$, $\rank(S)=r+1$, and $S^\top\one_K\neq\zero_N$, and define
\[
\mu:=\sqrt N\,\frac{S^\top\one_K}{\|S^\top\one_K\|_2},
\qquad
\gamma:=N^{-1}S\mu,
\qquad
R:=S-\gamma\mu^\top.
\]
Then $\one_N^\top\mu=0$, $\one_K^\top\gamma>0$, $\one_K^\top R=\zero_N^\top$, $R\mu=\zero_K$, and $\rank(R)=r$.
If $R=P\Sigma Q^\top$ is a compact singular-value decomposition, setting $V:=\sqrt N\,Q$ and $U:=N^{-1/2}P\Sigma$ yields a factorization satisfying \Cref{cond:llm-structure}.
Thus every admissible $S$ has a regular factorization, unique only up to orthogonal rotation of $(U,V)$.

\begin{lemma}[Factorized and reduced criterion equivalence]
\label[lemma]{lem:factor-reduced-equivalence}
    For every $\lambda>0$, the map
    \[
    \mathcal T_M(c,\theta)
    :=
    (M_\theta c,S_\theta,b),
    \qquad
    (c,\theta)\in\RR^d\times\Theta_r,
    \]
    is onto $\mathcal M_r$ and satisfies
    \[
    Q_\lambda(M_\theta c,S_\theta,b)
    =
    Q_\lambda\{\mathcal T_M(c,\theta)\}.
    \]
    Consequently,
    \[
    \inf_{\substack{c\in\RR^d\\ \theta\in\Theta_r}}
    Q_\lambda(M_\theta c,S_\theta,b)
    =
    \inf_{(s,S,b)\in\mathcal M_r}
    Q_\lambda(s,S,b).
    \]
    If $(\widehat c_\lambda,\widehat\theta_\lambda)$ minimizes \eqref{eq:weighted-joint-estimator}, then
    \[
    \bigl(
    M_{\widehat\theta_\lambda}\widehat c_\lambda,
    S_{\widehat\theta_\lambda},
    \widehat b_\lambda
    \bigr)
    \]
    is a minimizer of \eqref{eq:reduced-weighted-loss}.
    Conversely, every minimizer of \eqref{eq:reduced-weighted-loss} admits a regular factorization with the same calibrated score and objective value.
\end{lemma}

\begin{proof}[Proof of \Cref{lem:factor-reduced-equivalence}]
    Fix $\theta\in\Theta_r$.
    The constraints defining $\Theta_r$ give $S_\theta\one_N=\zero_K$, $S_\theta=[\gamma,U]W_\theta^\top$, and, since $\one_K^\top U=0^\top$, $S_\theta^\top\one_K=(\one_K^\top\gamma)\mu\neq\zero_N$ because $\one_K^\top\gamma>0$ and $N^{-1}\|\mu\|_2^2=1$.
    Moreover $V$ has rank $r$ and $\mu^\top V=0$, so $\rank(W_\theta)=r+1$ and $\operatorname{row}(S_\theta)=\col(W_\theta)$, both spaces having dimension $r+1$.
    Hence $\mathcal C_M(S_\theta)=\col(M_\theta)$ and $\rank(M_\theta)=d$ in either case, so $\mathcal T_M$ maps into $\mathcal M_r$.
    Conversely, for $(s,S,b)\in\mathcal M_r$, the reconstruction above gives $\theta\in\Theta_r$ with $S_\theta=S$ and $\col(M_\theta)=\mathcal C_M(S)$, so $s=M_\theta c$ for some $c\in\RR^d$ and $\mathcal T_M$ is onto.
    By \eqref{eq:weighted-loss},
    \[
    Q_\lambda(M_\theta c,S_\theta,b)
    =
    \ell_{\human}(M_\theta c)+\lambda\ell_{\llm}(S_\theta,b)
    =
    Q_\lambda\{\mathcal T_M(c,\theta)\},
    \]
    so the two infima agree by surjectivity.
    Comparison with every factorization shows that any global minimizer of \eqref{eq:weighted-joint-estimator} maps to a minimizer of \eqref{eq:reduced-weighted-loss}, and the converse follows from surjectivity and objective preservation.
\end{proof}

The set $\mathcal M_r$ is locally a smooth manifold: after imposing $S\one_N=\zero_K$, the exact-rank constraint is locally the rank-$(r+1)$ manifold in $\RR^{K\times(N-1)}$, while $S^\top\one_K\neq\zero_N$ is open.

\paragraph{Explicit coordinates.}
For analysis, use the over-parameterization
\begin{equation}
\begin{gathered}
\varphi=(A,B,c,b)
\longmapsto
(s,S,b)=\left(s_M(A,B,c),AB^\top,b\right),
\\
s_M(A,B,c)
:=
\begin{cases}
\displaystyle
c\sqrt N\,\frac{BA^\top\one_K}{\|BA^\top\one_K\|_2},
& M=\mu,
\\[1ex]
Bc,
& M=W,
\end{cases}
\\
A\in\RR^{K\times(r+1)},
\qquad
B\in\RR^{N\times(r+1)},
\qquad
B^\top\one_N=\zero_{r+1},
\end{gathered}
\label{eq:overparam}
\end{equation}
with $A$ and $B$ full column rank, $A^\top\one_K\neq\zero_{r+1}$, $c\in\RR^d$, and $b\in\RR^K$.
Its image is $\mathcal M_r$, and its fibres are exactly the gauge orbits
\begin{equation}
(A,B,c)
\longmapsto
\begin{cases}
(AG^{-\top},BG,c), & M=\mu,\\
(AG^{-\top},BG,G^{-1}c), & M=W,
\end{cases}
\qquad
G\in GL(r+1).
\label{eq:gauge}
\end{equation}
Therefore
\begin{equation}
\dim\mathcal M_r
=
(r+1)(K+N-r-2)+d+K.
\label{eq:dim-Mr}
\end{equation}
Thus the regular fibre of \eqref{eq:overparam} has dimension $(r+1)^2$, the dimension of $GL(r+1)$.
For computation, we use the restriction
\[
A^\top\one_K=K e_1,
\]
which remains a surjective local parameterization of $\mathcal M_r$ and restricts the gauge in \eqref{eq:gauge} to $\{G\in GL(r+1):Ge_1=e_1\}$, of dimension $r(r+1)$.
For $M=\mu$, its scalar calibration coordinate is equivalently written as $\alpha$ with $s=\alpha Be_1$.

\begin{lemma}[Gauge-invariant sandwich and trace]
\label[lemma]{lem:overparam-invariance}
    Fix a local minimizer of $q_\lambda$ on $\mathcal M_r$, and let $\zeta$ be any smooth chart around it in which the Hessian $H_\zeta:=\nabla_\zeta^2q_\lambda$ is nonsingular.
    Let $J_\zeta$ be the covariance of the individual human scores in that chart.
    Let $\varphi$ be coordinates in \eqref{eq:overparam}, or in any smooth restriction that remains a submersion onto $\mathcal M_r$, including the computational restriction $A^\top\one_K=K e_1$, and set $H_\varphi:=\nabla_\varphi^2q_\lambda(\varphi)$, $J_\varphi$ the same score covariance, and $\dot s_\varphi:=\partial s/\partial\varphi^\top$.
    Then, with Moore-Penrose inverses,
    \[
    \operatorname{tr}\!\left(H_\varphi^+J_\varphi\right)
    =
    \operatorname{tr}\!\left(H_\zeta^{-1}J_\zeta\right),
    \qquad
    \dot s_\varphi H_\varphi^+J_\varphi H_\varphi^+\dot s_\varphi^\top
    =
    \dot s_\zeta H_\zeta^{-1}J_\zeta H_\zeta^{-1}\dot s_\zeta^\top.
    \]
    Consequently the GACV penalty in \eqref{eq:gacv} and the sandwich \eqref{eq:weighted-empirical-sandwich} can be evaluated in either over-parameterized chart, and neither depends on the chart.
\end{lemma}

\begin{proof}[Proof of \Cref{lem:overparam-invariance}]
    Let $T:=\partial\zeta/\partial\varphi^\top$ at $\varphi$.
    Because the chosen parameterization is a submersion onto $\mathcal M_r$, $T$ has full row rank.
    Every quantity involved is a function of $(s,S,b)$ alone, so $\nabla_\varphi=T^\top\nabla_\zeta$; since $\nabla_\zeta q_\lambda=0$ at the point, the term involving the second derivative of $\zeta$ drops and
    \[
    H_\varphi=T^\top H_\zeta T,
    \qquad
    J_\varphi=T^\top J_\zeta T,
    \qquad
    \dot s_\varphi=\dot s_\zeta T.
    \]
    Take a thin QR factorization $T^\top=QR$ with $Q^\top Q=I$ and $R$ invertible, and set $G:=RH_\zeta R^\top\succ0$.
    Then $H_\varphi=QGQ^\top$, so $H_\varphi^+=QG^{-1}Q^\top$, while $J_\varphi=Q(RJ_\zeta R^\top)Q^\top$.
    Hence
    \[
    \operatorname{tr}\!\left(H_\varphi^+J_\varphi\right)
    =
    \operatorname{tr}\!\left(G^{-1}RJ_\zeta R^\top\right)
    =
    \operatorname{tr}\!\left(R^{-\top}H_\zeta^{-1}J_\zeta R^\top\right)
    =
    \operatorname{tr}\!\left(H_\zeta^{-1}J_\zeta\right).
    \]
    Since $TQ=R^\top$, the same substitution gives $\dot s_\varphi H_\varphi^+J_\varphi H_\varphi^+\dot s_\varphi^\top=\dot s_\zeta H_\zeta^{-1}J_\zeta H_\zeta^{-1}\dot s_\zeta^\top$.
\end{proof}

\paragraph{Regularity and endpoints.}
Neither $\Theta_r$ nor $\mathcal M_r$ is closed because rank may drop and $S^\top\one_K\neq\zero_N$ is an open restriction.
The item-side normalization in \Cref{cond:llm-structure} removes scale escape; under a finite identifying design with limiting probabilities in $(0,1)$, the remaining asymptotic failure modes are rank drop and separation.
At $\lambda=0$, the score-level problem reduces to the unrestricted centered human BTL estimator $\widehat s_0$, while its factor coordinates are unidentified.
For finite $\lambda$, the human likelihood can perturb $(S,b)$, so we report judge diagnostics from the LLM-only fit $\widehat\theta_\infty$ and reserve $\widehat s_\lambda$ for the human target.

For tuning and inference, \Cref{lem:factor-reduced-equivalence} lets us work on $\mathcal M_r$, while \Cref{lem:h-lambda-pd,lem:weighted-consistency} give the required local regularity.

\paragraph{Weighted information and consistency.}
\label{app:weighted-information}

\begin{lemma}[When the weighted information is nonsingular]
\label[lemma]{lem:h-lambda-pd}
    Suppose the limiting LLM comparison design satisfies \Cref{cond:llm-design}.
    Then $\ker(\mathcal I_{\llm})$ is exactly the $d$-dimensional fibre along which $s$ moves inside $\mathcal C_M(S)=\col(M)$ with $(S,b)$ fixed, and $\mathcal I_{\human}$ restricted to that fibre is congruent to $M^\top LM$.
    Consequently, for every $\lambda>0$,
    \[
    \mathcal H_\lambda\succ0
    \quad\Longleftrightarrow\quad
    \rank(D_{\human}^\infty M)=d.
    \]
\end{lemma}

\begin{proof}[Proof of \Cref{lem:h-lambda-pd}]
    For a tangent direction $v$, write $(\dot s_v,\dot S_v,\dot b_v)$ for its score-level differential.
    Since $\mathcal I_{\llm}$ is the Hessian of the limiting LLM risk, $v^\top\mathcal I_{\llm}v=0$ if and only if the LLM linear predictor is unchanged, so
    \[
    D_k\dot S_{v,k\cdot}^\top+z_k\dot b_{v,k}=0
    \]
    on the limiting support for every $k$.
    By \Cref{cond:llm-design} and row centering, this implies $\dot S_v=0$ and $\dot b_v=0$.
    With $(S,b)$ fixed, the remaining directions are $\dot s_v=M\dot c_v$, giving a $d$-dimensional fibre.
    If $\Pi$ is a basis of this fibre, then $\dot s_0\Pi=MC$ for some invertible $C$, and
    \[
    \Pi^\top\mathcal I_{\human}\Pi
    =
    C^\top M^\top LM C.
    \]
    Hence $\mathcal H_\lambda\succ0$ if and only if $M^\top LM\succ0$.
    By \eqref{eq:comparison-laplacian},
    \[
    u^\top M^\top LMu
    =
    \sum_{(i,j)\in\Omega_{\human}^\infty}
    \rho_{ij}p_{\human ij}(1-p_{\human ij})
    \{(e_i-e_j)^\top Mu\}^2,
    \]
    which is positive for every $u\neq0$ exactly when $\rank(D_{\human}^\infty M)=d$.
\end{proof}

The condition \Cref{cond:llm-design} is sufficient but not necessary here: if it fails, a direction invisible to the LLM likelihood may still move $\mathcal C_M(S)$ and hence be detected by the human likelihood.

\paragraph{Effective dimension.}
For $A,B\succeq0$ with $A+\lambda B\succ0$, define
\[
\mathrm{df}(\lambda)
:=
\operatorname{tr}\{(A+\lambda B)^{-1}A\}.
\]
Writing $G_\lambda=(A+\lambda B)^{-1}$ gives
\[
\mathrm{df}'(\lambda)
=
-\operatorname{tr}(G_\lambda BG_\lambda A)
\leq0,
\]
with $\mathrm{df}(0^+)=\rank(A)$ and $\mathrm{df}(\infty)=\dim\ker(B)$.
Taking $A=\mathcal I_{\human}$ and $B=\mathcal I_{\llm}$, \Cref{lem:h-lambda-pd} gives $\dim\ker(\mathcal I_{\llm})=d$, so the effective dimension decreases from $\rank(\mathcal I_{\human})$ to $d$, and from $N-1$ to $d$ under the connected-design setting of \Cref{thm:weighted-asymptotic-normality} when $\rank(\dot s_0)=N-1$.

\begin{lemma}[Regularity and consistency of the weighted estimator]
\label[lemma]{lem:weighted-consistency}
    Fix $\lambda\in(0,\infty)$ and $N,K,r$.
    Under model \eqref{eq:binomial-model-llm}, \eqref{eq:binomial-model-human}, and \eqref{eq:human-subspace-model}, \Cref{cond:score-normalization,cond:llm-structure}, let the human comparisons follow the independent-draw scheme of \Cref{subsec:population-calibration}, let $n_{\human},n_{\llm}\to\infty$, and suppose
    $n_k/n_{\llm}\to\kappa_k>0$ and $n_{kij}^{(a)}/n_k\to\rho_{kij}^{(a)}$, with each limiting within-judge support satisfying \Cref{cond:llm-design}.
    Assume $\rank(D_{\human}^\infty M)=d$ and that all active limiting comparison probabilities lie in $(0,1)$.
    Then the population criterion
    \[
    R_{\human}(s)+\lambda R_{\llm}(S,b)
    \]
    has the unique minimizer
    $(s_{\human},S,b)$ over $\mathcal M_r$, and its value is strictly
    separated from the minimum along any sequence approaching the boundary of $\mathcal M_r$.
    With probability tending to one, the empirical reduced criterion has a finite global minimizer in $\mathcal M_r$, every measurable sequence of such minimizers is consistent, and its Hessian in any smooth local coordinate on $\mathcal M_r$ is positive definite.
    Moreover, the score-level image of every global minimizer $(\widehat c_\lambda,\widehat\theta_\lambda)$ of \eqref{eq:weighted-joint-estimator} is a global minimizer of the reduced criterion.
\end{lemma}

\begin{proof}[Proof of \Cref{lem:weighted-consistency}]
    By correct specification,
    \[
    R_{\human}(s)-R_{\human}(s_{\human})\geq0,
    \qquad
    R_{\llm}(S,b)-R_{\llm}(S_0,b_0)\geq0.
    \]
    Under \Cref{cond:llm-design}, equality in the second display requires $(S,b)=(S_0,b_0)$.
    At $(S_0,b_0)$, exact calibration gives $s,s_{\human}\in\mathcal C_M(S_0)=\col(M)$, while equality in the human risk implies
    $D_{\human}^\infty(s-s_{\human})=0$.
    Hence $\rank(D_{\human}^\infty M)=d$ gives $s=s_{\human}$, proving uniqueness of the population minimizer.

    Because the comparison support is finite and all active probabilities are interior, with probability tending to one every active cell contains both outcomes.
    The empirical logistic losses are then coercive in their active linear predictors.
    Together with \Cref{cond:llm-design}, this bounds $(S,b)$ on every relevant empirical sublevel.
    Moreover, the unrestricted minimum human loss converges to $R_{\human}(s_{\human})$.
    Since the population LLM risk is separated from its minimum outside every neighborhood of $(S_0,b_0)$, uniform convergence on bounded predictor sets implies that every empirical global minimizer has $(S,b)$ in an arbitrarily small neighborhood of $(S_0,b_0)$ with probability tending to one.

    By continuity of $\mathcal C_M(S)$ and $\rank(D_{\human}^\infty M)=d$, on such a neighborhood
    $D_{\human}^\infty$ is uniformly injective on $\mathcal C_M(S)$.
    The human loss is therefore coercive in the calibration coordinates there, so the relevant empirical sublevel is compact.
    Uniform convergence and uniqueness of the population minimizer give consistency and existence of a finite minimizer in $\mathcal M_r$.

    Finally, \Cref{lem:h-lambda-pd} gives $\mathcal H_\lambda\succ0$ at the truth.
    Consistency and local uniform convergence of the empirical Hessian then imply positive definiteness of the Hessian in any smooth local coordinate with probability tending to one.
    The final claim follows from \Cref{lem:factor-reduced-equivalence}.
\end{proof}

\subsection[Proof of Theorem \ref{thm:weighted-asymptotic-normality}]{Proof of \Cref{thm:weighted-asymptotic-normality}}
\label{app:weighted-asymptotic-normality}

We work in the local coordinate $\zeta$ introduced in \Cref{subsec:asymptotic-normality}.
Let $\zeta_0$ denote the coordinate of the true reduced parameter $(s_{\human},S,b)$.
Let $\widehat\zeta_\lambda$ denote the local coordinate of the score-level image $(\widehat s_\lambda,\widehat S_\lambda,\widehat b_\lambda)$ of the factorized estimator in \eqref{eq:weighted-joint-estimator}.
By \Cref{lem:weighted-consistency,lem:factor-reduced-equivalence}, with probability tending to one this image is a regular global minimizer of the reduced criterion \eqref{eq:reduced-weighted-loss} and $\widehat\zeta_\lambda\pto\zeta_0$.
Fix a relatively compact smooth chart neighborhood of $\zeta_0$; the finite-support logistic criterion has uniformly bounded derivatives there.
Throughout this subsection, consider any sequence satisfying $n_{\human},n_{\llm}\to\infty$ and $n_{\human}/n_{\llm}\to\tau\in[0,\infty)$.

For the LLM sample, expand each binomial count into its individual Bernoulli comparisons.
For every comparison cell in the limiting support,
\[
\frac{n_{kij}^{(a)}}{n_{\llm}}
=
\frac{n_k}{n_{\llm}}\cdot\frac{n_{kij}^{(a)}}{n_k}
\longrightarrow
\kappa_k\rho_{kij}^{(a)},
\]
where $\sum_k\kappa_k=1$ and $\sum_{(i,j,a)}\rho_{kij}^{(a)}=1$ for each $k$, so the cell weights $\kappa_k\rho_{kij}^{(a)}$ sum to one, and each limiting within-judge support satisfies \Cref{cond:llm-design}.
Writing
\[
\eta_{kij}^{(a)}(S,b)
=
S_{ki}-S_{kj}+ab_k,
\]
the limiting population LLM loss is
\begin{align}
    R_{\llm}(S,b)
    =
    \sum_{k=1}^K
    \sum_{i<j}
    \sum_{a\in\{-1,1\}}
    \kappa_k\rho_{kij}^{(a)}
    \left[
    \log\left\{1+\exp\left(\eta_{kij}^{(a)}(S,b)\right)\right\}
    -
    p_{kij}^{(a)}
    \eta_{kij}^{(a)}(S,b)
    \right], \label{eq:R-llm}
\end{align}
where terms with $\kappa_k\rho_{kij}^{(a)}=0$ are omitted.

\begin{proof}[Proof of \Cref{thm:weighted-asymptotic-normality}]
    For the $t$-th human comparison between items $(i_t,j_t)$,
    define its negative log-likelihood contribution as
    \begin{equation}
    h_t(\zeta)
    :=
    \log\left\{
    1+\exp\left(
    \eta_{\human,i_tj_t}(\zeta)
    \right)
    \right\}
    -
    Y_t\eta_{\human,i_tj_t}(\zeta).
    \end{equation}

    For the $u$-th LLM comparison, corresponding to judge $k_u$,
    items $(i_u,j_u)$, and order $a_u\in\{-1,1\}$, define
    \begin{equation}
    l_u(\zeta)
    :=
    \log\left\{
    1+\exp\left(
    \eta_{k_u i_u j_u}^{(a_u)}(\zeta)
    \right)
    \right\}
    -
    Y_u\eta_{k_u i_u j_u}^{(a_u)}(\zeta).
    \end{equation}
    Both are expressed through the reduced local coordinate $\zeta$.
    Then
    \[
    q_\lambda(\zeta)
    =
    \frac{1}{n_{\human}}
    \sum_{t=1}^{n_{\human}}
    h_t(\zeta)
    +
    \frac{\lambda}{n_{\llm}}
    \sum_{u=1}^{n_{\llm}}
    l_u(\zeta).
    \]

    \paragraph{Part 1: Asymptotic normality.}
    Under the human BTL model, the LLM order-effect model, and the calibration restriction \eqref{eq:human-subspace-model}, $\zeta_0$ correctly specifies both likelihood components, so the two population scores vanish separately at $\zeta_0$.
    The information identity gives
    \[
    \Var\left\{\nabla_\zeta h_t(\zeta_0)\right\}
    =
    \mathcal I_{\human}
    \]
    for the human observations, while the limiting LLM design gives
    \[
    \frac{1}{n_{\llm}}
    \sum_{u=1}^{n_{\llm}}
    \Var\left\{\nabla_\zeta l_u(\zeta_0)\right\}
    \longrightarrow
    \mathcal I_{\llm}.
    \]

    Both $\mathcal I_{\human}$ and $\mathcal I_{\llm}$ are positive semidefinite, and $\mathcal H_\lambda\succ0$ by \Cref{lem:h-lambda-pd}, since connectivity of $\cG_{\human}^\infty$ makes $L$ positive definite on the centered space and hence gives $\rank(D_{\human}^\infty M)=d$.

    Because the human and LLM samples are independent,
    \begin{align*}
    \sqrt{n_{\human}}\nabla_\zeta q_\lambda(\zeta_0)
    &=
    \frac{1}{\sqrt{n_{\human}}}
    \sum_{t=1}^{n_{\human}}
    \nabla_\zeta h_t(\zeta_0)
    +
    \lambda
    \sqrt{\frac{n_{\human}}{n_{\llm}}}
    \frac{1}{\sqrt{n_{\llm}}}
    \sum_{u=1}^{n_{\llm}}
    \nabla_\zeta l_u(\zeta_0).
    \end{align*}
    The human term converges to $\mathcal N(0,\mathcal I_{\human})$.
    For the LLM triangular array, the individual score vectors are uniformly bounded and their average covariance converges to $\mathcal I_{\llm}$ under the limiting cell proportions, so the Lindeberg-Feller CLT gives
    \[
    \frac{1}{\sqrt{n_{\llm}}}\sum_{u=1}^{n_{\llm}}\nabla_\zeta l_u(\zeta_0)
    \dto
    \mathcal N(0,\mathcal I_{\llm}).
    \]
    Therefore, if $\tau>0$, independence and the joint central limit theorem give
    \[
    \sqrt{n_{\human}}\nabla_\zeta q_\lambda(\zeta_0)
    \dto
    \mathcal N
    \left(
    0,\,
    \mathcal I_{\human}+\lambda^2\tau\mathcal I_{\llm}
    \right).
    \]
    If $\tau=0$, then
    \[
    \lambda
    \sqrt{\frac{n_{\human}}{n_{\llm}}}
    \frac{1}{\sqrt{n_{\llm}}}
    \sum_{u=1}^{n_{\llm}}
    \nabla_\zeta l_u(\zeta_0)
    =
    \op(1),
    \]
    so the same conclusion holds with $\tau=0$.
    Hence, for any sequence with $n_{\human}/n_{\llm}\to\tau\in[0,\infty)$,
    \[
    \sqrt{n_{\human}}\nabla_\zeta q_\lambda(\zeta_0)
    \dto
    \mathcal N
    \left(
    0,\mathcal J_{\lambda,\tau}
    \right).
    \]

    Because the comparison support is finite and $\widehat\zeta_\lambda\pto\zeta_0$, the local uniform law of large numbers gives
    \[
    \nabla_\zeta^2 q_\lambda(\widetilde\zeta_\lambda)
    \pto
    \mathcal I_{\human}+\lambda\mathcal I_{\llm}
    =
    \mathcal H_\lambda
    \]
    for any intermediate point $\widetilde\zeta_\lambda$ between $\widehat\zeta_\lambda$ and $\zeta_0$.
    A Taylor expansion of the first-order condition therefore gives
    \[
    0
    =
    \nabla_\zeta q_\lambda(\zeta_0)
    +
    \left\{
    \mathcal H_\lambda+\op(1)
    \right\}
    (\widehat\zeta_\lambda-\zeta_0),
    \]
    and hence
    \[
    \sqrt{n_{\human}}
    (\widehat\zeta_\lambda-\zeta_0)
    =
    -
    \mathcal H_\lambda^{-1}
    \sqrt{n_{\human}}\nabla_\zeta q_\lambda(\zeta_0)
    +
    \op(1).
    \]
    It follows that
    \[
    \sqrt{n_{\human}}
    (\widehat\zeta_\lambda-\zeta_0)
    \dto
    \mathcal N
    \left(
    0,\,
    \mathcal H_\lambda^{-1}
    \mathcal J_{\lambda,\tau}
    \mathcal H_\lambda^{-1}
    \right).
    \]

    Applying the delta method to $\zeta\mapsto s(\zeta)$ gives
    \[
    \sqrt{n_{\human}}
    \left(
    \widehat s_\lambda-s_{\human}
    \right)
    \dto
    \mathcal N
    \left(
    0,
    \dot s_0
    \mathcal H_\lambda^{-1}
    \mathcal J_{\lambda,\tau}
    \mathcal H_\lambda^{-1}
    \dot s_0^\top
    \right).
    \]
    Under \eqref{eq:human-subspace-model}, $s_{\human}=s(\zeta_0)=Mc_{\human}$, which proves \eqref{eq:calibrated-asymptotic-normality}.

    \paragraph{Part 2: Variance estimator.}
    Under a smooth nonsingular change of local coordinates, $\mathcal H_\lambda$, $\mathcal J_{\lambda,\tau}$, and $\dot s_0$ transform covariantly, so the displayed covariance is unchanged.
    When $\tau=0$, $\mathcal J_{\lambda,0}=\mathcal I_{\human}$, giving the abundant-LLM expression in \Cref{thm:weighted-asymptotic-normality}.
    For completeness, define the empirical sandwich covariance explicitly.
    At $\widehat\zeta_\lambda$, let
    \[
    g_{\human t}
    :=
    \nabla_\zeta h_t(\widehat\zeta_\lambda),
    \qquad
    \overline g_{\human}
    :=
    \frac{1}{n_{\human}}\sum_{t=1}^{n_{\human}}g_{\human t},
    \]
    and
    \[
    g_{\llm u}
    :=
    \nabla_\zeta l_u(\widehat\zeta_\lambda),
    \qquad
    \overline g_{\llm}
    :=
    \frac{1}{n_{\llm}}\sum_{u=1}^{n_{\llm}}g_{\llm u}.
    \]
    Define
    \[
    \widehat{\mathcal V}_{\human}
    :=
    \frac{1}{n_{\human}}
    \sum_{t=1}^{n_{\human}}
    (g_{\human t}-\overline g_{\human})
    (g_{\human t}-\overline g_{\human})^\top,
    \]
    \[
    \widehat{\mathcal V}_{\llm}
    :=
    \frac{1}{n_{\llm}}
    \sum_{u=1}^{n_{\llm}}
    (g_{\llm u}-\overline g_{\llm})
    (g_{\llm u}-\overline g_{\llm})^\top,
    \]
    and
    \[
    \widehat{\mathcal H}_\lambda
    :=
    \nabla_\zeta^2q_\lambda(\widehat\zeta_\lambda),
    \qquad
    \widehat{\mathcal J}_\lambda
    :=
    \widehat{\mathcal V}_{\human}
    +
    \lambda^2\frac{n_{\human}}{n_{\llm}}
    \widehat{\mathcal V}_{\llm}.
    \]
    Writing
    \[
    \widehat{\dot s}_\lambda
    :=
    \nabla_\zeta s(\widehat\zeta_\lambda),
    \]
    define
    \begin{equation}
    \widehat\Sigma_{\lambda,s}
    :=
    \widehat{\dot s}_\lambda
    \widehat{\mathcal H}_\lambda^{-1}
    \widehat{\mathcal J}_\lambda
    \widehat{\mathcal H}_\lambda^{-1}
    \widehat{\dot s}_\lambda^\top.
    \label{eq:weighted-empirical-sandwich}
    \end{equation}
    By \Cref{lem:weighted-consistency}, finite comparison support, and the limiting-design condition,
    \[
    \widehat{\mathcal H}_\lambda\pto\mathcal H_\lambda,
    \qquad
    \widehat{\mathcal V}_{\human}\pto\mathcal I_{\human},
    \qquad
    \widehat{\mathcal V}_{\llm}\pto\mathcal I_{\llm},
    \qquad
    \widehat{\dot s}_\lambda\pto\dot s_0.
    \]
    Since $n_{\human}/n_{\llm}\to\tau$ and $\mathcal H_\lambda$ is positive definite,
    \[
    \widehat\Sigma_{\lambda,s}
    \pto
    \Sigma_{\lambda,s,\tau}
    \]
    by the continuous mapping theorem.
    Thus $n_{\human}^{-1}\widehat\Sigma_{\lambda,s}$ consistently estimates the first-order covariance of $\widehat s_\lambda$.
    The same change-of-coordinate calculation as above shows that \eqref{eq:weighted-empirical-sandwich} is invariant to the chosen smooth local coordinate.

    \paragraph{Part 3: Uniformity and infinity.}
    The uniform linearization, sandwich consistency, and $\lambda\to\infty$ limit follow from \Cref{cor:uniform-lambda}.
\end{proof}

\begin{corollary}[Uniformity in the calibration weight]
\label[corollary]{cor:uniform-lambda}
    Under the conditions of \Cref{thm:weighted-asymptotic-normality} with $n_{\human}/n_{\llm}\to0$, for every $\lambda_0>0$,
    \[
    \sup_{\lambda\ge\lambda_0}
    \Bigl\|
    \sqrt{n_{\human}}(\widehat s_\lambda-s_{\human})
    +
    \dot s_0\mathcal H_\lambda^{-1}\sqrt{n_{\human}}\,\nabla_\zeta\ell_{\human}\{s(\zeta_0)\}
    \Bigr\|_2
    =\op(1),
    \qquad
    \sup_{\lambda\ge\lambda_0}
    \bigl\|\widehat\Sigma_{\lambda,s}-\Sigma_{\lambda,s,0}\bigr\|
    =\op(1),
    \]
    and $\Sigma_{\lambda,s,0}\to MH_M^{-1}M^\top$ as $\lambda\to\infty$.
    Hence, if $\lambda_n\to\lambda_\star\in[\lambda_0,\infty)$, \eqref{eq:calibrated-asymptotic-normality} holds with covariance $\Sigma_{\lambda_\star,s,0}$; if $\lambda_n\to\infty$, its limit is that of \Cref{cor:staged-oracle-equivalence}.
\end{corollary}

\begin{proof}[Proof of \Cref{cor:uniform-lambda}]
    By coordinate invariance, use the product chart $\zeta=(\xi,c)$, where $\xi$ is a smooth chart of the regular structured $(S,b)$ model as in \Cref{prop:stage-one-normality} and $s=M(\xi)c$ for a smooth local basis $M(\xi)$ of $\mathcal C_M\{S(\xi)\}$ with $M(\xi_0)=M$.
    In this chart the LLM loss depends on $\xi$ alone, so its score and Hessian have zero $c$-blocks at every $\zeta$, $\mathcal I_{\llm}=\operatorname{diag}(\mathcal I_\xi,0)$ with $\mathcal I_\xi\succ0$, and $\partial s/\partial c=M$ at $\zeta_0$.
    Write $\mathcal I_{\human}=\begin{pmatrix}A&B\\B^\top&C\end{pmatrix}$ with $C=M^\top LM\succ0$ (\Cref{lem:h-lambda-pd}).
    Block inversion of $\mathcal H_\lambda$ gives, for a constant $c_0$ depending on $\lambda_0$ only,
    \begin{equation}
    \mathcal H_\lambda^{-1}
    =
    \begin{pmatrix}0&0\\0&C^{-1}\end{pmatrix}
    +O(\lambda^{-1}),
    \qquad
    \sup_{\lambda\ge\lambda_0}
    \bigl\|\lambda\,\mathcal H_\lambda^{-1}\operatorname{diag}(E,0)\bigr\|
    \le c_0\|E\|
    \quad\text{for every }E.
    \label{eq:uniform-block-inverse}
    \end{equation}
    In the proof of \Cref{lem:weighted-consistency}, the inequality $\lambda\{\ell_{\llm}(\widehat\xi_\lambda)-\ell_{\llm}(\xi_0)\}\le\ell_{\human}(s_{\human})-\inf\ell_{\human}=\op(1)$ holds with $\lambda$ replaced by $\lambda_0$ for all $\lambda\ge\lambda_0$, and the remaining steps do not involve $\lambda$, so $\sup_{\lambda\ge\lambda_0}\|\widehat\zeta_\lambda-\zeta_0\|=\op(1)$.
    The first-order condition reads
    $0=\nabla q_\lambda(\zeta_0)+\{\mathcal H_\lambda+E_{\human}+\lambda\operatorname{diag}(E_\xi,0)\}(\widehat\zeta_\lambda-\zeta_0)$,
    where $E_{\human}=\op(1)$ uniformly, and $E_\xi=o(1)+O(\|\widehat\zeta_\lambda-\zeta_0\|)$ because the logistic Hessian of $\ell_{\llm}$ is free of the outcomes.
    Multiplying by $\mathcal H_\lambda^{-1}$ and using \eqref{eq:uniform-block-inverse},
    $\sqrt{n_{\human}}(\widehat\zeta_\lambda-\zeta_0)=-\{I+\op(1)\}\mathcal H_\lambda^{-1}\sqrt{n_{\human}}\nabla q_\lambda(\zeta_0)$ uniformly in $\lambda\ge\lambda_0$.
    Here $\sqrt{n_{\human}}\nabla q_\lambda(\zeta_0)=\sqrt{n_{\human}}\nabla\ell_{\human}\{s(\zeta_0)\}+\lambda\sqrt{n_{\human}/n_{\llm}}\,(Z_{\llm}^\top,0)^\top$ with $Z_{\llm}=\Op(1)$, and \eqref{eq:uniform-block-inverse} gives $\mathcal H_\lambda^{-1}(\lambda Z_{\llm}^\top,0)^\top=\Op(1)$ uniformly, so the LLM term is $\Op(\sqrt{n_{\human}/n_{\llm}})=\op(1)$; the delta method with $\widehat{\dot s}_\lambda\pto\dot s_0$ uniformly proves the first display.
    For the second, $\widehat{\mathcal H}_\lambda^{-1}-\mathcal H_\lambda^{-1}=\op(1)$ uniformly by the same expansion of $\widehat{\mathcal H}_\lambda=\mathcal H_\lambda+E_{\human}+\lambda\operatorname{diag}(E_\xi,0)$, and the LLM part of the meat contributes $\widehat{\mathcal H}_\lambda^{-1}\lambda^2(n_{\human}/n_{\llm})\operatorname{diag}(\widehat{\mathcal V}_\xi,0)\widehat{\mathcal H}_\lambda^{-1}=\Op(n_{\human}/n_{\llm})$ uniformly by \eqref{eq:uniform-block-inverse}.
    Finally, the first part of \eqref{eq:uniform-block-inverse} gives $\Sigma_{\lambda,s,0}\to\dot s_0\operatorname{diag}(0,C^{-1})\dot s_0^\top=MC^{-1}M^\top$.
\end{proof}

\paragraph{Local departures from exact calibration.}
\Cref{thm:weighted-asymptotic-normality} assumes \eqref{eq:human-subspace-model} exactly.
The following proposition gives the first-order limit under an $n_{\human}^{-1/2}$ departure.

\begin{proposition}[Inference under local misalignment]
\label[proposition]{prop:local-misalignment}
    Assume the conditions of \Cref{thm:weighted-asymptotic-normality} and $c_{\human}\neq0$, with \eqref{eq:human-subspace-model} replaced by
    \begin{equation}
    s_{\human}^{(n_{\human})}
    =
    Mc_{\human}
    +
    n_{\human}^{-1/2}h,
    \qquad
    h\in\RR^N,\quad \one_N^\top h=0.
    \label{eq:local-misalignment}
    \end{equation}
    Then
    \begin{equation}
    \sqrt{n_{\human}}
    \left(
    \widehat s_\lambda-s_{\human}^{(n_{\human})}
    \right)
    \dto
    \mathcal N
    \left(
    -(I_N-A_\lambda)h,\ \Sigma_{\lambda,s,\tau}
    \right),
    \qquad
    A_\lambda
    :=
    \dot s_0\mathcal H_\lambda^{-1}\dot s_0^\top L,
    \label{eq:local-asymptotic-normality}
    \end{equation}
    with $L$ the comparison Laplacian \eqref{eq:comparison-laplacian}.
    On the centered subspace, $A_\lambda$ is self-adjoint in the $L$-inner product, its ordered eigenvalues lie in $[0,1]$ and are nonincreasing in $\lambda$, and $A_\lambda h=h$ for $h\in\col(M)$.
    As $\lambda\to0^+$ it tends to the identity on centered vectors, while as $\lambda\to\infty$ it tends to $P_M^L:=M(M^\top LM)^{-1}M^\top L$.
\end{proposition}

\begin{proof}[Proof of \Cref{prop:local-misalignment}]
    Keep $\zeta_0$ at the coordinate of $(Mc_{\human},S,b)$, which under \eqref{eq:local-misalignment} is no longer the true human score.
    Only the human component of the criterion changes.
    Writing $p_{\human ij}^{(n_{\human})}=\sigma\{(Mc_{\human})_i-(Mc_{\human})_j+n_{\human}^{-1/2}(h_i-h_j)\}$ and expanding the logistic function about $n_{\human}^{-1/2}=0$,
    \begin{align*}
    \nabla_sR_{\human}^{(n_{\human})}(Mc_{\human})
    &=
    -\sum_{(i,j)\in\Omega_{\human}^\infty}
    \rho_{ij}
    \left\{
    p_{\human ij}^{(n_{\human})}
    -
    \sigma\left((Mc_{\human})_i-(Mc_{\human})_j\right)
    \right\}
    (e_i-e_j)
    \\
    &=
    -n_{\human}^{-1/2}Lh
    +
    O(n_{\human}^{-1}),
    \end{align*}
    by the definition \eqref{eq:comparison-laplacian} of $L$.
    The LLM component is unaffected, so the chain rule gives $\sqrt{n_{\human}}\,\EE\{\nabla_\zeta q_\lambda(\zeta_0)\}\to-\dot s_0^\top Lh$, while the score covariance is unchanged at first order because $p_{\human ij}^{(n_{\human})}\to p_{\human ij}$.
    Hence $\sqrt{n_{\human}}\nabla_\zeta q_\lambda(\zeta_0)\dto\cN(-\dot s_0^\top Lh,\mathcal J_{\lambda,\tau})$, and repeating the Taylor expansion above,
    \[
    \sqrt{n_{\human}}(\widehat\zeta_\lambda-\zeta_0)
    \dto
    \cN\left(
    \mathcal H_\lambda^{-1}\dot s_0^\top Lh,\
    \mathcal H_\lambda^{-1}\mathcal J_{\lambda,\tau}\mathcal H_\lambda^{-1}
    \right).
    \]
    The delta method gives $\sqrt{n_{\human}}(\widehat s_\lambda-Mc_{\human})\dto\cN(A_\lambda h,\Sigma_{\lambda,s,\tau})$, and subtracting the deterministic $\sqrt{n_{\human}}(s_{\human}^{(n_{\human})}-Mc_{\human})=h$ proves \eqref{eq:local-asymptotic-normality}.
    Since $c_{\human}\neq0$, the tangent map $\dot s_0$ spans the centered score space:
    for $M=\mu$, variation of $c$ spans $\mu$ and variation of $\mu$ spans its centered orthogonal complement; for $M=W$, variation of $c$ spans $\col(W)$ and subspace variation of $Wc_{\human}$ spans its centered complement.
    Hence $\rank(\dot s_0)=N-1$.
    
    \paragraph{The two endpoints.}
    Set
    \[
    B_\lambda
    :=
    L^{1/2}\dot s_0\mathcal H_\lambda^{-1}\dot s_0^\top L^{1/2}.
    \]
    Since $\mathcal I_{\human}=\dot s_0^\top L\dot s_0$ and $\rank(\dot s_0)=N-1$, $B_\lambda\to P$ as $\lambda\to0^+$.
    By \Cref{lem:h-lambda-pd}, $\ker(\mathcal I_{\llm})$ is the $d$-dimensional fibre mapped by $\dot s_0$ onto $\col(M)$, so $B_\lambda\to P_{L^{1/2}M}$ and $A_\lambda\to P_M^L$ as $\lambda\to\infty$.
    For $h\in\col(M)$, choose $v\in\ker(\mathcal I_{\llm})$ with $\dot s_0v=h$.
    Then
    \[
    \mathcal H_\lambda v
    =
    \mathcal I_{\human}v
    =
    \dot s_0^\top Lh,
    \]
    so $A_\lambda h=h$.
    
    Moreover,
    \[
    \frac{\partial B_\lambda}{\partial\lambda}
    =
    -
    L^{1/2}\dot s_0
    \mathcal H_\lambda^{-1}
    \mathcal I_{\llm}
    \mathcal H_\lambda^{-1}
    \dot s_0^\top L^{1/2}
    \preceq0.
    \]
    Since $B_\lambda\succeq0$ and $B_\lambda\to P$ as $\lambda\to0^+$,
    \[
    0\preceq B_\lambda\preceq P.
    \]
    Thus, on the centered subspace, the ordered eigenvalues of $B_\lambda$ lie in $[0,1]$ and are nonincreasing in $\lambda$.
    Because
    \[
    B_\lambda
    =
    L^{1/2}A_\lambda L^{+/2}
    \]
    on that subspace, $A_\lambda$ has the same eigenvalues.
    Finally,
    \[
    A_\lambda^\top L
    =
    LA_\lambda,
    \]
    so $A_\lambda$ is self-adjoint in the $L$-inner product.
    Together with $B_\lambda\to P_{L^{1/2}M}$, exactly $d$ centered directions remain unshrunk as $\lambda\to\infty$.
\end{proof}

\paragraph{Interpretation.}
By \Cref{prop:local-misalignment}, the local bias vanishes as $\lambda\to0^+$ and approaches $-(I_N-P_M^L)h$ as $\lambda\to\infty$, while $\Sigma_{\lambda,s,\tau}$ gives the corresponding first-order variance.

\begin{corollary}[Local risk at a finite calibration weight]
\label[corollary]{cor:finite-weight-risk}
Under \Cref{prop:local-misalignment} with $\tau=0$, let $R_{\human}^{(n_{\human})}$ denote the human risk under \eqref{eq:local-misalignment}.
For fixed $\lambda\in(0,\infty)$,
\[
n_{\human}\!\left\{R_{\human}^{(n_{\human})}(\widehat s_\lambda)
-R_{\human}^{(n_{\human})}(s_{\human}^{(n_{\human})})\right\}
\dto
\frac12\left\|(I_N-A_\lambda)h-G_\lambda\right\|_L^2,
\qquad
G_\lambda\sim\mathcal N(0,\Sigma_{\lambda,s,0}).
\]
The limiting variable has mean
\[
\frac12\|(I_N-A_\lambda)h\|_L^2+\frac12\tr(B_\lambda^2),
\qquad
B_\lambda:=L^{1/2}\dot s_0\mathcal H_\lambda^{-1}\dot s_0^\top L^{1/2}.
\]
Moreover,
$\tr(B_\lambda^2)\le \tr(B_\lambda)=\df(\lambda)$ and
$d\le\df(\lambda)\le N-1$.
As $\lambda\to0^+$ and $\lambda\to\infty$, the limiting means approach
$(N-1)/2$ and $\delta_M+d/2$, respectively, along \eqref{eq:local-misalignment}.
\end{corollary}

\begin{proof}[Proof of \Cref{cor:finite-weight-risk}]
    By \Cref{prop:local-misalignment} and a second-order expansion of $R_{\human}^{(n_{\human})}$ at $s_{\human}^{(n_{\human})}$, whose Hessian converges to $L$, the scaled excess risk converges to the stated quadratic form.
    For $\tau=0$, $\tr(L\Sigma_{\lambda,s,0})=\tr(B_\lambda^2)$, while the effective-dimension calculation gives $\tr(B_\lambda)=\df(\lambda)$ and $0\preceq B_\lambda\preceq P$.
    The endpoint claims follow from \Cref{prop:local-misalignment} and \Cref{lem:comparison-laplacian}.
\end{proof}

\subsection[Proof of Theorem \ref{thm:gacv}]{Proof of \Cref{thm:gacv}}\label{app:gacv-proof}

\begin{assumption}[GACV regularity]
\label[assumption]{ass:weighted-local-regularity}
    Condition on the complete LLM comparison sample and let $n_{\human}\to\infty$.
    For each $\lambda\in\overline\Lambda$, the corresponding conditional population criterion has a unique interior minimizer $\zeta_\lambda^\star$ with nonsingular Hessian and is well separated: for some neighborhood $\mathcal U_\lambda$ of $\zeta_\lambda^\star$ and $\epsilon_\lambda>0$,
    \[
    \inf_{\zeta\notin\mathcal U_\lambda}
    q_\lambda^\star(\zeta)
    \ge
    q_\lambda^\star(\zeta_\lambda^\star)+\epsilon_\lambda,
    \]
    where $q_\lambda^\star$ is the conditional expectation of the human part of the criterion plus its fixed LLM part.
    The full-sample fit converges in probability to $\zeta_\lambda^\star$.
    Here the score space is $\mathcal M_r$ for finite $\lambda$, the centered score space for $\lambda=0$, and $\mathcal C_M(\widehat S_\infty)$ for $\lambda=\infty$.
\end{assumption}

\begin{lemma}[Verification of GACV regularity]
\label[lemma]{lem:gacv-regularity}
Suppose the conditions of \Cref{lem:weighted-consistency} hold, $\cG_{\human}^\infty$ is connected, and $\overline\Lambda$ is fixed and finite.
As $n_{\llm}\to\infty$, \Cref{ass:weighted-local-regularity} holds with probability tending to one over the LLM sample, simultaneously for $\lambda\in\overline\Lambda$.
\end{lemma}

\begin{proof}[Proof of \Cref{lem:gacv-regularity}]
For finite $\lambda>0$, uniform $C^2$ convergence of $\ell_{\llm}$ to $R_{\llm}$ together with \Cref{lem:weighted-consistency,lem:h-lambda-pd} gives a unique separated regular conditional minimizer, and finite-support M-estimation gives conditional consistency as $n_{\human}\to\infty$.
For $\lambda=0$, connectivity gives the same properties on the centered score space.
For $\lambda=\infty$, \Cref{prop:stage-one-normality} gives $\widehat M_\infty\pto M$, hence $\rank(D_{\human}^\infty\widehat M_\infty)=d$ with probability tending to one, and the fixed-space argument applies.
Finiteness of $\overline\Lambda$ gives simultaneity.
\end{proof}

\begin{proof}[Proof of \Cref{thm:gacv}]
    Condition on the complete LLM comparison sample throughout.
    For finite $\lambda\in\Lambda$, use \Cref{lem:factor-reduced-equivalence} to identify each regular factorized full-sample fit with its score-level minimizer in $\mathcal M_r$; for $\lambda=0$ use the centered score space, and for $\lambda=\infty$ use $\mathcal C_M(\widehat S_\infty)$.
    Let $\zeta$ denote smooth local coordinates on the neighborhood $\mathcal U_\lambda$ in \Cref{ass:weighted-local-regularity}.
    By that assumption and finiteness of $\overline\Lambda$, with conditional probability tending to one all full-sample fits lie in these neighborhoods, their Hessians are uniformly nonsingular, and the required derivatives are uniformly bounded.
    
    For an individual human comparison $t$, let $(I_t,J_t)$ denote its compared pair and let $Z_t\in\{0,1\}$ indicate preference for item $I_t$.
    Its negative log-likelihood contribution is
    \[
    h_t(s)
    =
    \log\{1+\exp(s_{I_t}-s_{J_t})\}
    -
    Z_t(s_{I_t}-s_{J_t}).
    \]
    In local coordinates, write $h_t(\zeta):=h_t\{s(\zeta)\}$.
    For each candidate $\lambda$, the full-sample criterion can be written locally as
    \[
    q_\lambda(\zeta)
    =
    \frac{1}{n_{\human}}
    \sum_{t=1}^{n_{\human}}
    h_t(\zeta)
    +
    a_\lambda(\zeta),
    \]
    where $a_\lambda(\zeta)=\lambda\ell_{\llm}\{S(\zeta),b(\zeta)\}$ for finite $\lambda$, $a_0=0$, and at $\lambda=\infty$ the LLM-only fit is absorbed into the fixed reduced score space.
    In every case, $a_\lambda$ is fixed when a human comparison is deleted.
    
    Let $\widehat\zeta_\lambda$ minimize $q_\lambda$ and define
    \[
    g_{t,\lambda}
    :=
    \nabla h_t(\widehat\zeta_\lambda),
    \qquad
    \overline g_\lambda
    :=
    \frac{1}{n_{\human}}
    \sum_{t=1}^{n_{\human}}
    g_{t,\lambda},
    \qquad
    \widehat H_\lambda
    :=
    \nabla^2 q_\lambda(\widehat\zeta_\lambda).
    \]
    The first-order condition gives
    \[
    \overline g_\lambda
    +
    \nabla a_\lambda(\widehat\zeta_\lambda)
    =
    0.
    \]
    
    After deleting human comparison $t$, we renormalize the human loss by $n_{\human}-1$ and obtain
    \[
    q_{\lambda,-t}(\zeta)
    =
    \frac{1}{n_{\human}-1}
    \sum_{u\neq t}
    h_u(\zeta)
    +
    a_\lambda(\zeta).
    \]
    Evaluating its score at the full-sample estimator and using the preceding first-order condition gives
    \begin{align*}
    \nabla q_{\lambda,-t}(\widehat\zeta_\lambda)
    &=
    \frac{n_{\human}\overline g_\lambda-g_{t,\lambda}}{n_{\human}-1}
    +
    \nabla a_\lambda(\widehat\zeta_\lambda) \\
    &=
    -\frac{g_{t,\lambda}-\overline g_\lambda}{n_{\human}-1}.
    \end{align*}
    The individual score vectors are uniformly bounded because the comparison support is finite and the chart derivatives are bounded on the fixed local neighborhoods above, so
    \[
    \max_{\substack{1\leq t\leq n_{\human}\\ \lambda\in\overline\Lambda}}
    \left\|
    \nabla q_{\lambda,-t}(\widehat\zeta_\lambda)
    \right\|_2
    =
    \Op(n_{\human}^{-1}).
    \]
    Deleting one human comparison perturbs the criterion by $\Op(n_{\human}^{-1})$ in $C^2$ norm on these neighborhoods.
    By the finite-support uniform law of large numbers and the separation in \Cref{ass:weighted-local-regularity}, with conditional probability tending to one every leave-one-out global minimizer lies in $\mathcal U_\lambda$, uniformly over $t$ and $\lambda\in\overline\Lambda$.
    The Hessian remains nonsingular there, and the implicit function theorem gives a unique stationary point near $\widehat\zeta_\lambda$; hence the leave-one-out global minimizer $\widehat\zeta_\lambda^{(-t)}$ is this point.
    A mean-value expansion of the leave-one-out first-order condition, together with uniform Hessian nonsingularity, therefore first gives
    \[
    \max_{\substack{1\leq t\leq n_{\human}\\ \lambda\in\overline\Lambda}}
    \left\|
    \widehat\zeta_\lambda^{(-t)}
    -
    \widehat\zeta_\lambda
    \right\|_2
    =
    \Op(n_{\human}^{-1}).
    \]
    At $\widehat\zeta_\lambda$, deleting one human observation changes the Hessian by $\Op(n_{\human}^{-1})$ uniformly in $t$ and $\lambda$.
    The preceding perturbation rate and the uniform third-derivative bound further imply that the leave-one-out Hessian at any intermediate point between $\widehat\zeta_\lambda^{(-t)}$ and $\widehat\zeta_\lambda$ equals $\widehat H_\lambda+\Op(n_{\human}^{-1})$ uniformly.
    Consequently, expanding the leave-one-out first-order condition and using the inverse perturbation formula gives
    \begin{equation}
    \widehat\zeta_\lambda^{(-t)}
    -
    \widehat\zeta_\lambda
    =
    \frac{1}{n_{\human}-1}
    \widehat H_\lambda^{-1}
    (g_{t,\lambda}-\overline g_\lambda)
    +
    \Op(n_{\human}^{-2}),
    \label{eq:gacv-delete-one-expansion}
    \end{equation}
    uniformly over $t$ and $\lambda\in\overline\Lambda$.

    Because the Hessians of the individual human losses are uniformly bounded on these neighborhoods, \eqref{eq:gacv-delete-one-expansion} implies that a first-order Taylor expansion of the deleted human loss has quadratic remainder $\Op(n_{\human}^{-2})$ uniformly in $t$ and $\lambda$.
    Thus
    \begin{align*}
    h_t(\widehat\zeta_\lambda^{(-t)})
    &=
    h_t(\widehat\zeta_\lambda)
    +
    g_{t,\lambda}^\top
    (\widehat\zeta_\lambda^{(-t)}-\widehat\zeta_\lambda)
    +
    \Op(n_{\human}^{-2}) \\
    &=
    h_t(\widehat\zeta_\lambda)
    +
    \frac{1}{n_{\human}-1}
    g_{t,\lambda}^\top
    \widehat H_\lambda^{-1}
    (g_{t,\lambda}-\overline g_\lambda)
    +
    \Op(n_{\human}^{-2}),
    \end{align*}
    uniformly over $t$ and $\lambda$.
    Averaging over $t$ gives
    \begin{align*}
    \operatorname{CV}_{\mathrm{loo}}(\lambda)
    &=
    \ell_{\human}(\widehat s_\lambda)
    +
    \frac{1}{n_{\human}(n_{\human}-1)}
    \sum_{t=1}^{n_{\human}}
    g_{t,\lambda}^\top
    \widehat H_\lambda^{-1}
    (g_{t,\lambda}-\overline g_\lambda)
    +
    \Op(n_{\human}^{-2}) \\
    &=
    \ell_{\human}(\widehat s_\lambda)
    +
    \frac{1}{n_{\human}-1}
    \operatorname{tr}
    \left(
    \widehat H_\lambda^{-1}
    \widehat J_\lambda
    \right)
    +
    \Op(n_{\human}^{-2}),
    \end{align*}
    because $\sum_t(g_{t,\lambda}-\overline g_\lambda)=0$.
    This proves \eqref{eq:gacv-loo-equivalence}.
    
    Finally, for finite $\lambda$, coordinate invariance of \eqref{eq:gacv} follows from \Cref{lem:overparam-invariance}, which also permits evaluation in the over-parameterized coordinates \eqref{eq:overparam}.
    For $\lambda\in\{0,\infty\}$, the same trace identity follows directly under any smooth nonsingular reparameterization of the corresponding score space.
\end{proof}

\subsection{Additional results for adaptive calibration}
\label{app:gacv-additional}
\begin{corollary}[Near-LOO optimality of GACV]
\label[corollary]{cor:gacv-loo-optimality}
    Under the conditions of \Cref{thm:gacv}, conditionally on the LLM comparisons,
    \begin{equation}
    \operatorname{CV}_{\mathrm{loo}}(\widehat\lambda)
    \leq
    \min_{\lambda\in\overline\Lambda}
    \operatorname{CV}_{\mathrm{loo}}(\lambda)
    +
    \Op(n_{\human}^{-2}).
    \label{eq:gacv-loo-optimality}
    \end{equation}
\end{corollary}

\begin{proof}[Proof of \Cref{cor:gacv-loo-optimality}]
    Let
    \[
    \Delta_n
    :=
    \max_{\lambda\in\overline\Lambda}
    \left|
    \operatorname{GACV}(\lambda)
    -
    \operatorname{CV}_{\mathrm{loo}}(\lambda)
    \right|
    =
    \Op(n_{\human}^{-2})
    \]
    by \Cref{thm:gacv}.
    If $\lambda_{\mathrm{loo}}\in\argmin_{\lambda\in\overline\Lambda}\operatorname{CV}_{\mathrm{loo}}(\lambda)$, then optimality of $\widehat\lambda$ for GACV gives
    \[
    \operatorname{CV}_{\mathrm{loo}}(\widehat\lambda)
    \leq
    \operatorname{CV}_{\mathrm{loo}}(\lambda_{\mathrm{loo}})
    +
    2\Delta_n,
    \]
    which proves \eqref{eq:gacv-loo-optimality}.
\end{proof}

For each fixed candidate $\lambda$, conditionally on the LLM sample,
$\EE\{\operatorname{CV}_{\mathrm{loo}}(\lambda)\mid\mathcal D_{\llm}\}
=
\EE[R_{\human}(\widehat s_{\lambda,n_{\human}-1})\mid\mathcal D_{\llm}]$,
where $\widehat s_{\lambda,m}$ denotes the fit using $m$ independent human comparisons.
Thus GACV is a second-order approximation to an unbiased estimator of candidate predictive human risk, and GACV achieves near-LOO selection over the fixed grid without a human validation split or a rate restriction between $n_{\human}$ and $n_{\llm}$; it can likewise select $(r,\lambda)$ jointly as in \Cref{app:rank-selection}.

\paragraph{Calibration-space test and endpoint comparison.}
The restriction \eqref{eq:human-subspace-model} can be tested using the two endpoint fits already computed for GACV.

\begin{proposition}[Calibration-space likelihood-ratio test]
\label[proposition]{prop:calibration-test}
    Suppose that $N$, $K$, and $r$ are fixed, that $\cG_{\human}^\infty$ is connected, that \Cref{prop:stage-one-normality} applies, and that $n_{\human},n_{\llm}\to\infty$ with $n_{\human}/n_{\llm}\to0$.
    Define
    \begin{equation}
    T_{n_{\human}}
    :=
    2n_{\human}
    \left\{
    \ell_{\human}(\widehat s_\infty)
    -
    \ell_{\human}(\widehat s_0)
    \right\}.
    \label{eq:calibration-lrt}
    \end{equation}
    Under \eqref{eq:human-subspace-model},
    \[
    T_{n_{\human}}
    \dto
    \chi^2_{N-1-d}.
    \]
    Along \eqref{eq:local-misalignment},
    \[
    T_{n_{\human}}
    \dto
    \chi^2_{N-1-d}(2\delta_M),
    \qquad
    2\delta_M
    =
    \operatorname{dist}_L^2\{h,\col(M)\}.
    \]
    The AIC comparison between the anchored and unrestricted endpoints selects anchoring exactly when
    \[
    T_{n_{\human}}<2(N-1-d),
    \]
    whose local first-order expected boundary is $\delta_M<(N-1-d)/2$, matching \Cref{thm:population-risk-tradeoff}.
\end{proposition}

\begin{proof}[Proof of \Cref{prop:calibration-test}]
    Write $\widehat M_\infty:=M_{\widehat\theta_\infty}$.
    By \Cref{prop:stage-one-normality}, $\|P_{\col(\widehat M_\infty)}-P_{\col(M)}\|_{\mathrm F}=\Op(n_{\llm}^{-1/2})$, so $n_{\human}/n_{\llm}\to0$ and the argument of \Cref{cor:staged-oracle-equivalence} allow $\col(\widehat M_\infty)$ to be replaced by $\col(M)$ with $\op(1)$ change in $T_{n_{\human}}$.
    The resulting statistic is the likelihood-ratio test of the $d$-dimensional restriction $s\in\col(M)$ within the $(N-1)$-dimensional centered human BTL model, so Wilks's theorem gives the null limit.
    Under \eqref{eq:local-misalignment}, \Cref{lem:comparison-laplacian}\eqref{eq:laplacian-risk-expansion} gives noncentrality $\operatorname{dist}_L^2\{h,\col(M)\}=2\delta_M$.
    Finally, the anchored-minus-unrestricted AIC difference is $T_{n_{\human}}-2(N-1-d)$, and the mean of the local limiting statistic is $N-1-d+2\delta_M$.
\end{proof}

\clearpage
\section{Experimental and Implementation Details}
\label{app:experiments}

\subsection{DIAL Estimation Procedure}
\label{app:dial-algorithm}

\Cref{alg:dial-anc,alg:dial-ada} summarize the estimation procedures used throughout the experiments.
The main reported DIAL estimator uses the consensus calibration basis
$M_\theta=\mu$; replacing it by $M_\theta=W_\theta=[\mu,V]$ gives
\methDIALW{} without changing the fitting or tuning logic.
DIAL-Anc first fits the structured LLM model and then calibrates the resulting representation to the human comparisons, producing the staged LLM-anchored endpoint
$(\widehat c_\infty,\widehat\theta_\infty,\widehat s_\infty)$.
DIAL-Ada uses this endpoint to initialize a descending path over finite calibration weights and selects among the admissible candidates by GACV.
Finite candidate fits that are not numerically resolved within the optimization box, or that fail the numerical regularity checks in
\Cref{app:lambda-selection}, are excluded from GACV selection.
The staged endpoint and the human-only endpoint are included whenever they satisfy the corresponding endpoint regularity checks, with the staged endpoint providing the computational fallback whenever admissible.
Replications with no admissible candidate are reported as failures.
For GACV, $\widehat H_\lambda^{-1}$ is evaluated in a nonsingular reduced chart; equivalently, the Moore--Penrose inverse
$\widehat H_{\lambda,\varphi}^{+}$
may be used in the over-parameterized coordinates of
\eqref{eq:overparam}, by \Cref{lem:overparam-invariance}.
Diagnostics for position bias are always reported from the LLM-only component
$\widehat\theta_\infty$, because finite-$\lambda$ joint fitting allows the human likelihood to perturb the LLM parameters under misspecification.
For numerical robustness, optimization is performed over a large box with coordinate bounds $10$ in the centering-reduced chart.

\begin{algorithm}[!ht]
\caption{DIAL-Anc estimation}
\label{alg:dial-anc}
\begin{algorithmic}[1]

\Require
LLM cell counts $\{n_{kij}^{(a)},Y_{kij}^{(a)}\}$;
human comparisons $\{n_{\human ij},Y_{\human ij}\}$;
rank $r$

\Ensure
DIAL-Anc score $\widehat s_{\mathrm{Anc}}=\widehat s_\infty$;

\State Set the calibration basis
\[
M_\theta\gets\mu,
\qquad
\text{or }
M_\theta\gets W_\theta=[\mu,V]
\text{ for full-space calibration.}
\]
\label{line:dial-anc-basis}

\State Fit the structured order-effect LLM model
\[
\widehat\theta_\infty
\in
\argmin_{\theta\in\Theta_r}
\ell_{\llm}(S_\theta,b_\theta).
\]
\label{line:dial-anc-llm}

\State Fit the staged human calibration
\[
\widehat c_\infty
\in
\argmin_c
\ell_{\human}
\!\left(M_{\widehat\theta_\infty}c\right).
\]
\label{line:dial-anc-cal}

\State Set
\[
\widehat s_\infty
=
\widehat s_{\mathrm{Anc}}
=
M_{\widehat\theta_\infty}\widehat c_\infty.
\]
\label{line:dial-anc-score}
\Comment{Consensus calibration when $M_\theta=\mu$}

\end{algorithmic}
\end{algorithm}

\begin{algorithm}[!ht]
\caption{DIAL-Ada estimation}
\label{alg:dial-ada}
\begin{algorithmic}[1]

\Require
LLM cell counts $\{n_{kij}^{(a)},Y_{kij}^{(a)}\}$;
human comparisons $\{n_{\human ij},Y_{\human ij}\}$;
rank $r$;
candidate weights $\overline\Lambda$ containing $0$ and $\infty$

\Ensure
DIAL-Ada score $\widehat s_{\mathrm{Ada}}$
and selected weight $\widehat\lambda$

\State Initialize the admissible set
$\mathcal R\gets\varnothing$.

\If{the staged endpoint
$(\widehat c_\infty,\widehat\theta_\infty,\widehat s_\infty)$
from \Cref{alg:dial-anc},
Lines~\ref{line:dial-anc-llm}--\ref{line:dial-anc-score},
passes the endpoint regularity check}
    \State $\mathcal R\gets\mathcal R\cup\{\infty\}$.
\EndIf

\State Fit the centered human-only BTL estimator $\widehat s_0$.

\If{$\widehat s_0$ passes the endpoint regularity check}
    \State $\mathcal R\gets\mathcal R\cup\{0\}$.
\EndIf

\State Order
$\Lambda=\overline\Lambda\setminus\{0,\infty\}$
from largest to smallest and initialize the path at
$(\widehat c_\infty,\widehat\theta_\infty)$
from \Cref{alg:dial-anc},
Lines~\ref{line:dial-anc-llm}--\ref{line:dial-anc-cal}.

\For{$\lambda\in\Lambda$ in descending order}
    \State Fit
    \[
    (\widehat c_\lambda,\widehat\theta_\lambda)
    \in
    \argmin_{c,\theta\in\Theta_r}
    \left\{
    \ell_{\human}(M_\theta c)
    +
    \lambda\ell_{\llm}(S_\theta,b_\theta)
    \right\},
    \]
    using the most recent retained solution as a warm start.

    \State Set
    \[
    \widehat s_\lambda
    =
    M_{\widehat\theta_\lambda}\widehat c_\lambda.
    \]

    \If{the fit passes the numerical and regularity checks}
        \State $\mathcal R\gets\mathcal R\cup\{\lambda\}$.
        \State Retain the fit to warm-start the next weight.
    \EndIf
\EndFor

\If{$\mathcal R=\varnothing$}
    \State Report the replication as a numerical failure.
\Else
    \State Compute $\operatorname{GACV}(\lambda)$
    for all $\lambda\in\mathcal R$.

    \State Select
    \[
    \widehat\lambda
    \in
    \argmin_{\lambda\in\mathcal R}
    \operatorname{GACV}(\lambda),
    \qquad
    \widehat s_{\mathrm{Ada}}
    =
    \widehat s_{\widehat\lambda}.
    \]
\EndIf

\end{algorithmic}
\end{algorithm}

\subsection{Variance Estimation for the DIAL Intervals}
\label{app:hac}
All DIAL intervals in this paper are Wald intervals from the sandwich \eqref{eq:weighted-empirical-sandwich}, whose bread $\widehat{\mathcal H}_\lambda$ is the Hessian of the weighted criterion and whose meat $\widehat{\mathcal J}_\lambda$ estimates the variance of its score.
Dependence between comparisons changes only the meat, which we estimate in the heteroskedasticity- and autocorrelation-consistent (HAC) form of \citet{newey1987simple} and \citet{andrews1991heteroskedasticity}.
Write $\mathcal O$ for the set of all $n_{\human}+n_{\llm}$ comparisons and let $\widetilde\psi_v:=n_{\human}^{-1}(g_{\human t}-\overline g_{\human})$ if $v$ is the human comparison $t$ and $\widetilde\psi_v:=\lambda n_{\llm}^{-1}(g_{\llm u}-\overline g_{\llm})$ if $v$ is the LLM comparison $u$, the centered contribution of $v$ to the score of $q_\lambda$ at $\widehat\zeta_\lambda$.
For a symmetric matrix $\kappa=(\kappa_{vw})_{v,w\in\mathcal O}$ of kernel weights with unit diagonal, set
\begin{equation}
\widehat{\mathcal J}_\lambda^{\kappa}
:=
n_{\human}\sum_{v,w\in\mathcal O}\kappa_{vw}\,\widetilde\psi_v\widetilde\psi_w^\top,
\label{eq:hac-meat}
\end{equation}
and use it in place of $\widehat{\mathcal J}_\lambda$ in \eqref{eq:weighted-empirical-sandwich}; the result is positive semidefinite whenever $\kappa$ is.
The identity kernel returns \eqref{eq:weighted-empirical-sandwich}, which treats every comparison as independent.
\paragraph{The kernel we use.}
\label{app:clustered}
The binomial model \eqref{eq:binomial-model-llm} treats the comparisons in a cell $(k,i,j)$ as independent draws with a common probability, whereas the probability may vary with the pair beyond what the scores predict, so that repeated queries of one pair by one judge are correlated through a shared pair-level effect.
Our implementation therefore sets $\kappa_{vw}=1$ when $v$ and $w$ are LLM comparisons of the same judge on the same pair, in either display order, and $\kappa_{vw}=\ind\{v=w\}$ otherwise, the cluster-robust form of \citet{liang1986longitudinal} with judge-pair cells as clusters.
This gives
\begin{equation}
\widehat{\mathcal J}_\lambda^{\mathrm{cell}}
=
\widehat{\mathcal V}_{\human}+\lambda^2\frac{n_{\human}}{n_{\llm}}\widehat{\mathcal V}_{\llm}^{\mathrm{cell}},
\qquad
\widehat{\mathcal V}_{\llm}^{\mathrm{cell}}
:=
\frac{1}{n_{\llm}}\sum_{k=1}^{K}\sum_{i<j}G_{kij}G_{kij}^\top,
\label{eq:cell-clustered-meat}
\end{equation}
where $G_{kij}:=\sum_{a\in\{-1,1\}}\{n^{(a)}_{kij}\widehat p^{(a)}_{kij}-Y^{(a)}_{kij}\}\nabla_\zeta\eta^{(a)}_{kij}(\widehat\zeta_\lambda)-(n^{(1)}_{kij}+n^{(-1)}_{kij})\overline g_{\llm}$ is the summed centered score of the cell and $\widehat p^{(a)}_{kij}:=\sigma\{\eta^{(a)}_{kij}(\widehat\zeta_\lambda)\}$, so \eqref{eq:cell-clustered-meat} is computed from the cell counts alone.
It matches the dependence generated in the simulations of \Cref{subsec:synthetic}, where the human comparisons are independent, so only the LLM block is clustered, and it is used for every DIAL interval in the paper except the binomial comparison in \Cref{app:simulation}.
The judge diagnostic of \Cref{fig:real-diagnostics}(a) reports an LLM-side parameter and uses the same meat in its LLM-only specialization, $\widehat{\mathcal V}_{\llm}^{\mathrm{cell}}$ with the Hessian of $\ell_{\llm}/n_{\llm}$ as bread, which is the $\lambda\to\infty$ limit of \eqref{eq:cell-clustered-meat}.
On the real data a record additionally supplies the judgments of all judges and at most one human label on the same two responses; the kernel that joins all comparisons of a record adds a cross term between the human and LLM scores of each record, requires record identities that the aggregated cell counts discard, and is not used here.
Consistency of \eqref{eq:hac-meat} requires the dependence to lie within the support of $\kappa$ and the number of independent blocks to grow \citep{hansen2019asymptotic}; with $N$ and $K$ fixed the cell kernel averages over $K\binom N2$ cells, and under the pair-level effect it estimates the variance over both the binomial sampling and the draw of the cell effects, which is why it stays near nominal in \Cref{app:simulation} while the binomial version under-covers as $n_{\llm}$ grows.

\subsection{Hyperparameter Selection}
\label{app:hyperparameters}

\subsubsection[Rank r]{Rank $r$}
\label{app:rank-selection}

\Cref{thm:fixed-space-tradeoff} applies to any fixed centered subspace, in particular to the nested family $\col(W_{r'})$, $r'=0,1,\ldots$, obtained by retaining the leading $r'$ disagreement directions of the fitted score matrix.
Along this family the approximation error $\Delta_{W_{r'}}$ is nonincreasing in $r'$, while the first-order estimation term grows as $(r'+1)/(2n_{\human})$, so the population rule is to choose the $r'$ minimizing $\delta_{W_{r'}}+(r'+1)/2$.
We implement its empirical counterpart by evaluating \eqref{eq:gacv} over a grid of $(r,\lambda)$ pairs and selecting the pair with the smallest criterion value, which requires no separate rank-selection procedure.

\subsubsection[Weight lambda]{Weight $\lambda$}
\label{app:lambda-selection}

For each experiment, $\overline\Lambda$ is a finite logarithmic grid augmented with the endpoints $0$ and $\infty$, and $\widehat\lambda$ is chosen by \eqref{eq:gacv}.
\Cref{thm:gacv} treats this candidate set as fixed asymptotically; in the implementation we center the finite grid at the joint-likelihood weight $n_{\llm}/n_{\human}$, using seven log-spaced multiples between $0.01$ and $10$ (between $0.1$ and $10$ in the simulations), so that the ordinary joint MLE is always a candidate; the lower end matters on real data, where $n_{\llm}/n_{\human}$ is in the hundreds and the path leaves the staged fit only below $0.1\,n_{\llm}/n_{\human}$.
The regularity in \Cref{ass:weighted-local-regularity} is verified asymptotically by \Cref{lem:gacv-regularity}; in finite samples, a candidate fit can be nonfinite or have a singular criterion Hessian.
In particular, the human-only endpoint has no finite MLE when the human comparison graph is disconnected or separated, and the trace term in \eqref{eq:gacv} is then undefined.
We therefore drop a finite positive-$\lambda$ candidate from the GACV argmin when its fit does not converge, is not numerically resolved within the optimization box, has more zero Hessian eigenvalues than the $r(r+1)$ exact invariances of the computational parameterization in \Cref{app:gacv-reduced}, or has $\|\widehat s_\lambda\|_\infty>10$.
The human-only endpoint is retained only when its directed win graph is strongly connected \citep{ford1957solution}; at $\lambda=\infty$, the LLM-only fit is conditioned on, so only its calibration coordinates must be numerically resolved.
The implementation uses $\lambda=\infty$ as the computational fallback; if no candidate passes the guards, that replication is reported as a failure rather than averaged.
GACV at the endpoints uses the same formula with the chart of the corresponding estimator, the centered score space at $\lambda=0$ and the coordinate $c$ with $W$ fixed at the LLM-only fit at $\lambda=\infty$.
The path is fitted from the largest $\lambda$ downward, initializing the first fit at the staged estimator and each subsequent fit at the previous one, and every fit is finished by a quasi-Newton step on the reduced chart.

\clearpage
\section{Simulation Studies}\label{app:simulation}

This section gives the design and implementation details for \Cref{fig:simu-main} and three additional synthetic configurations.

\subsection{Design and implementation}
\label{app:simulation-design}

\paragraph{Data-generating process.}
We draw the consensus direction $\mu$ from a Gaussian distribution, center it, and scale it to $\|\mu\|_2^2=N$.
The columns of $V\in\RR^{N\times r}$ are Gaussian, projected orthogonally to $\one_N$ and $\mu$, and normalized so that $N^{-1}V^\top V=I_r$.
We take $\gamma_k=1+0.3z_k$ with Gaussian $z_k$ and rescale $\gamma$ to $\one_K^\top\gamma=K$, while $U$ is Gaussian with zero column means, giving $S=\gamma\mu^\top+UV^\top$ satisfying \Cref{cond:llm-structure}.
Position effects are $b_k=\beta_k\xi_k$, where $\beta_k\sim\mathrm{Uniform}(0.3,1.2)$ and $\Pr(\xi_k=1)=0.8$, with one judge fixed at $b_k=0.05$.
The main configuration uses $s_{\human}=\mu$, while the misaligned configuration in \Cref{fig:simu-misaligned} uses $s_{\human}=Wc_{\human}$ with $c_{\human}=(1,c_V)$ and $c_V\sim\cN(0,0.5^2I_r)$.
Each comparison samples its item pair uniformly, LLM comparisons additionally sample the judge uniformly, and the canonical item is displayed first with probability $0.75$.
For the bottom-row experiments, we replace the LLM linear predictor by $\eta_{kij}^{(a)}+\varepsilon_{kij}$ with $\varepsilon_{kij}\sim\cN(0,1)$ drawn once per judge-pair and shared across repeated queries and both display orders.
For \Cref{fig:simu-pos-heterogeneity}, we instead generate pair-varying position effects $b_{kij}=b_k+\delta_{kij}$ and use the linear predictor $s_{ki}-s_{kj}+ab_{kij}$, where $\delta_{kij}$ is drawn once per judge-pair, centered within judge over pairs, and has pairwise standard deviation $\sigma_{\mathrm{pos}}$; estimation retains the working model with a single $b_k$ per judge.

\paragraph{Configurations.}
\Cref{fig:simu-main} uses $(N,K,r)=(10,4,1)$, with $n_{\llm}=20{,}000$ and $n_{\human}\in\{100,200,400,800,1600\}$ in the top row, and $n_{\human}=800$ with $n_{\llm}\in\{400,800,\ldots,12{,}800\}$ in the bottom row.
\Cref{fig:simu-items20} uses $(N,K,r)=(20,6,2)$, with $n_{\llm}=60{,}000$ and $n_{\human}\in\{300,\ldots,4800\}$ in the top row, and $n_{\human}=2400$ with $n_{\llm}\in\{2400,\ldots,76{,}800\}$ in the bottom row.
\Cref{fig:simu-misaligned} uses the same budgets as \Cref{fig:simu-main} with the misaligned human target above.
\Cref{fig:simu-pos-heterogeneity} uses the main $(N,K,r)=(10,4,1)$ configuration with $s_{\human}=\mu$, $n_{\human}=800$, $n_{\llm}=20{,}000$, and $\sigma_{\mathrm{pos}}\in\{0,0.25,0.5,0.75,1\}$.
Every configuration uses $50$ independent replications.

\paragraph{Estimation and metrics.}
All structured estimators use the true working rank $r$ and are defined as in \Cref{sec:experiments}; finite-$\lambda$ paths are warm-started from the staged endpoint and use the GACV grid and regularity checks of \Cref{app:lambda-selection}.
We additionally evaluate the oracle candidate minimizing population human risk.
The reported metrics are excess human risk $R_{\human}(\widehat s)-R_{\human}(s_{\human})$, Kendall's $\tau$ between $\widehat s$ and $s_{\human}$, $\operatorname{RMSE}(\widehat b,b)$, and coverage of pairwise human-score contrasts.
Coverage uses Fisher information for \methHuman{} and the judge-pair clustered sandwich of \Cref{app:clustered} for DIAL estimators; intervals at the GACV-selected weight are post-selection.
Since \methPooled{} and \methDIALnoPos{} impose $\widehat b=0$, their order-effect errors coincide, with \methPooled{} drawn as a wider band beneath \methDIALnoPos{} in the corresponding panels.

\paragraph{Numerical summary of \Cref{fig:simu-main}.}
Under exact specification (top row), \methHuman{} tracks the $(N-1)/(2n_{\human})$ first-order risk benchmark of \Cref{thm:population-risk-tradeoff} and \methConsensus{} the $1/(2n_{\human})$ benchmark of its one-dimensional calibration space; at $n_{\human}\le200$, \methConsensus{} reduces excess risk by $8$--$14\times$ and raises Kendall's $\tau$ from $0.75$ to $0.98$ relative to \methHuman{}.
The median excess risk of \methDIALmu{} stays within $20\%$ of \methConsensus{}, whereas \methDIALW{} is $1.3$--$1.6\times$ higher because it estimates an unused disagreement direction.
Under pair-level LLM misspecification (bottom row), \methConsensus{} plateaus and is worse than \methHuman{} for $n_{\llm}\le6400$, while \methDIALmu{} is below both endpoints at every $n_{\llm}$ and within $0.001$ of the oracle; its Kendall's $\tau$ is $0.91$--$0.93$, against $0.89$ for \methHuman{} and $0.78$--$0.90$ for \methConsensus{}.
For order-aware fits, $\operatorname{RMSE}(\widehat b,b)$ falls from $0.37$ to $0.11$ across the bottom row and reaches $0.039$ at $n_{\llm}=20{,}000$, while methods imposing $\widehat b=0$ remain at $0.68$, and cell-clustered coverage is $0.93$--$0.96$.

\subsection{Additional configurations}
\label{app:simulation-additional}

\paragraph{Larger problem.}
\Cref{fig:simu-items20} repeats the experiment at $(N,K,r)=(20,6,2)$.
The main patterns persist: under exact specification \methDIALmu{} approaches \methConsensus{} as human supervision becomes scarce, while under LLM overdispersion it remains below both \methHuman{} and \methConsensus{} across the LLM-budget sweep.
Judge-pair clustered coverage remains between $0.92$ and $0.96$, and \methDIALW{} pays additional variance under the consensus-aligned target because its extra disagreement directions are unnecessary.

\begin{figure}[!ht]
    \centering
    \includegraphics[width=\linewidth]{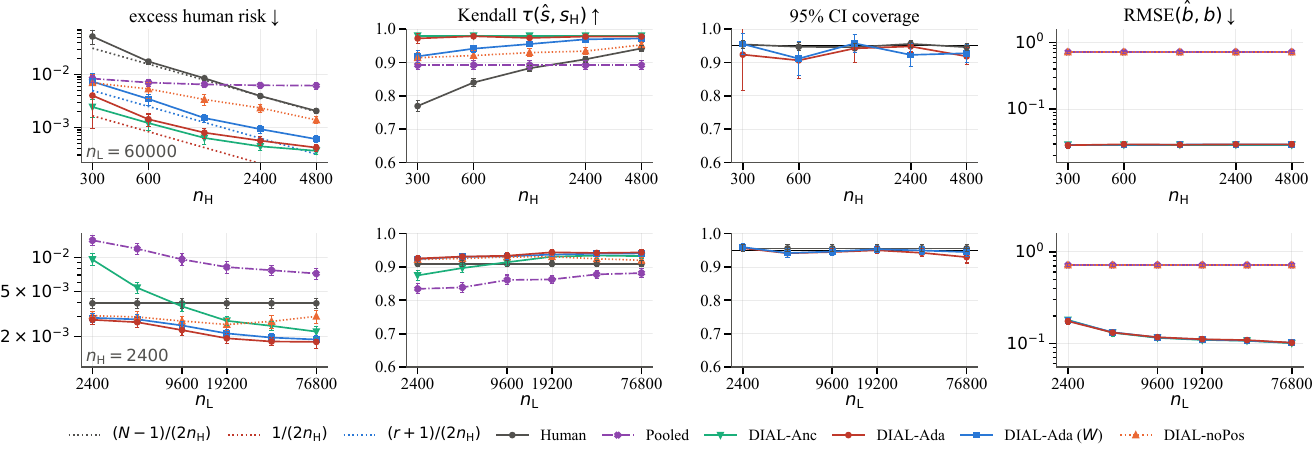}
    \caption{Larger-scale replicate of \Cref{fig:simu-main} with $N=20$, $K=6$, $r=2$, and $50$ replications; rows, columns, and estimators follow \Cref{fig:simu-main}.}
    \label{fig:simu-items20}
\end{figure}

\paragraph{Human target outside the consensus direction.}
\Cref{fig:simu-misaligned} sets $s_{\human}=Wc_{\human}$ with $c_V\sim\cN(0,0.5^2)$, so calibration within $\operatorname{col}(W)$ is correctly specified but consensus-only calibration has nonzero approximation error $\Delta_\mu$.
\methConsensus{} and \methPooled{} plateau at excess risks of about $0.040$ and $0.045$, while \methDIALmu{} shifts toward the human comparisons as $n_{\human}$ grows but cannot represent the missing disagreement direction.
In contrast, \methDIALW{} continues to decrease with the human budget, attaining Kendall's $\tau$ of $0.93$--$0.97$ versus $0.83$--$0.94$ for \methDIALmu{} and coverage of $0.94$--$0.96$ under the correctly specified top-row model.
Under bottom-row LLM overdispersion, \methDIALmu{}, \methDIALW{}, and \methHuman{} are within $0.001$ in excess risk, while \methConsensus{} remains biased.

\begin{figure}[!ht]
    \centering
    \includegraphics[width=\linewidth]{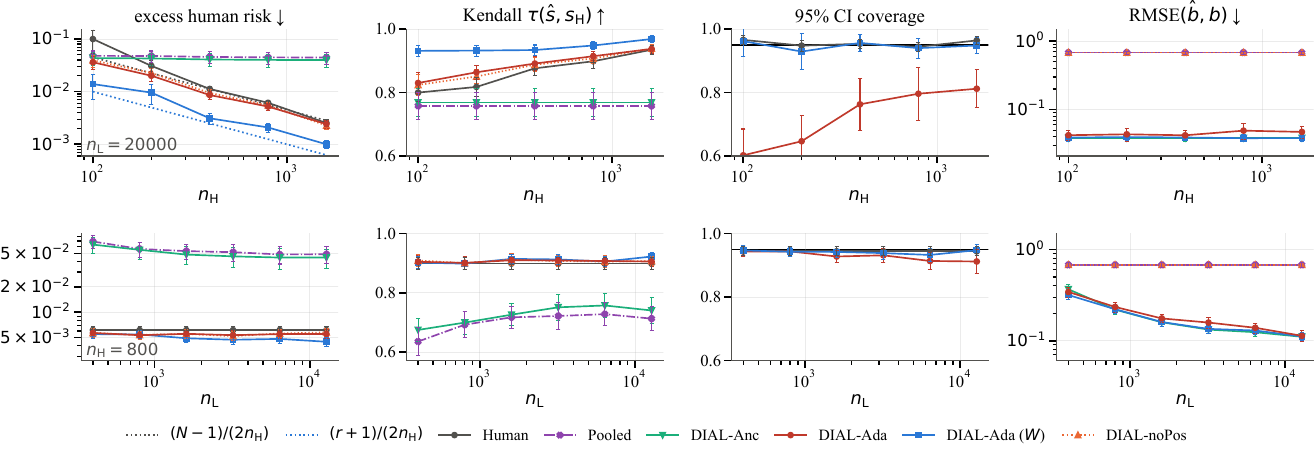}
    \caption{Misaligned-target simulation under the main $(N,K,r)=(10,4,1)$ configuration, with $s_{\human}=Wc_{\human}$ and $c_V\sim\cN(0,0.5^2)$.
    Rows and columns follow \Cref{fig:simu-main}; dotted lines show the first-order excess-risk benchmarks for \methHuman{} and calibration within $\operatorname{col}(W)$.}
    \label{fig:simu-misaligned}
\end{figure}

\paragraph{Pair-varying position effects.}
\Cref{fig:simu-pos-heterogeneity} violates the constant-position-effect working model by replacing $b_k$ with $b_k+\delta_{kij}$, while keeping the latent LLM score structure and human calibration model correctly specified.
As $\sigma_{\mathrm{pos}}$ increases from $0$ to $1$, the excess human risk of \methConsensus{} and \methDIALmu{} rises from $0.0011$ and $0.0013$ to $0.0030$ and $0.0027$ and their Kendall's $\tau$ falls from $0.978$ and $0.972$ to $0.932$ and $0.928$, yet both stay better than \methHuman{} ($0.0059$, $\tau=0.894$) and \methDIALnoPos{} at every $\sigma_{\mathrm{pos}}$.
The fitted $\widehat b_k$ summarizes the misspecified constant-effect component and is evaluated against the judge-level mean $b_k$; so its RMSE grows from $0.039$ to $0.114$ while remaining about six times smaller than the $0.679$ of the order-agnostic \methPooled{} and \methDIALnoPos{}, which impose $\widehat b=0$.

\begin{figure}[!t]
    \centering
    \includegraphics[width=\linewidth]{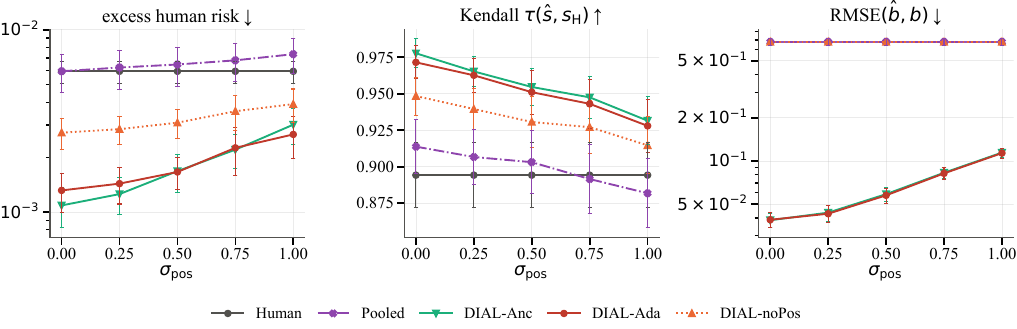}
    \caption{Robustness to heterogeneous position effects under the main $(N,K,r)=(10,4,1)$ configuration with $n_{\human}=800$ and $n_{\llm}=20{,}000$.
    Data are generated with $b_{kij}=b_k+\delta_{kij}$, where the within-judge pair effects have standard deviation $\sigma_{\mathrm{pos}}$, while all methods retain a constant $b_k$ working model.
    Panels report excess human risk, Kendall's $\tau$, and $\operatorname{RMSE}(\widehat b,b)$ over $50$ replications.}
    \label{fig:simu-pos-heterogeneity}
\end{figure}

\clearpage
\section{Real Data}\label{app:sec:real}

\subsection{Datasets and collection}\label{app:subsec:real-evaluation}

We use Chatbot Arena \citep{chiang2024chatbot}, MT-Bench \citep{zheng2023mtbench}, and PandaLM \citep{wang2024pandalm}, whose records contain a prompt, two responses, and a human preference label.
For every record we queried $22$ LLM judges in both display orders: $19$ open-weight models via Ollama and $3$ commercial APIs (\texttt{Claude Haiku 4.5}, \texttt{Gemini 3.1 Flash Lite}, and \texttt{GPT-4.1 nano}); decoding configurations are provided with the released collection kit.
We exclude \texttt{StableLM~2, 12B} because its tie rate exceeds $50\%$ on every dataset, leaving $K=21$ judges.

Because \eqref{eq:binomial-model-llm} and \eqref{eq:binomial-model-human} are binary, we exclude LLM ties and human ties or ``both bad'' labels; treating ties as half wins produced the same qualitative conclusions.
Because asymmetric tie rates can also carry position information, explicit tie modeling remains an extension.
\Cref{tab:real-data} reports the retained and discarded counts.

\begin{table}[h]
\centering
\footnotesize
\setlength{\tabcolsep}{4pt}
\caption{Dataset statistics after preprocessing.
``Kept'' refers to the $21$ retained judges; a cell is a (judge, canonical pair) combination.}
\label{tab:real-data}
\begin{tabular}{lrrr}
\toprule
 & Chatbot Arena & MT-Bench & PandaLM \\
\midrule
Records (prompts with a human label) & 7{,}738 & 1{,}199 & 945 \\
Candidate models $N$ & 20 & 6 & 5 \\
Observed pairs / possible & 189 / 190 & 15 / 15 & 10 / 10 \\
Human labels a / b / tie / both bad & 2{,}771 / 2{,}659 / 814 / 1{,}494 & 432 / 410 / 357 / 0 & 398 / 444 / 103 / 0 \\
Decisive human labels & 5{,}430 & 842 & 842 \\
LLM judgments, kept judges & 321{,}668 & 49{,}172 & 39{,}324 \\
\quad of which ties & 21{,}199 (6.6\%) & 3{,}289 (6.7\%) & 4{,}374 (11.1\%) \\
\quad tie-excluded $n_{\llm}$ & 300{,}469 & 45{,}883 & 34{,}950 \\
Display order $+1$ / $-1$ & 150{,}299 / 150{,}170 & 22{,}909 / 22{,}974 & 17{,}443 / 17{,}507 \\
Cells observed / possible & 3{,}965 / 3{,}990 & 315 / 315 & 210 / 210 \\
$n_{kij}$ per cell (min / mean / max) & 1 / 75.8 / 381 & 64 / 145.7 / 215 & 114 / 166.4 / 208 \\
Judge tie rate (min-max) & 0.002-0.175 & 0.004-0.145 & 0.004-0.247 \\
\bottomrule
\end{tabular}
\end{table}

\begin{table}[h]
\centering
\footnotesize
\setlength{\tabcolsep}{4pt}
\caption{Candidate models treated as items after preprocessing.}
\label{tab:real-items}
\begin{tabularx}{0.97\linewidth}{lP}
\toprule
Dataset & Candidate models \\
\midrule
Chatbot Arena ($N=20$) &
\texttt{RWKV-4-Raven-14B}, \texttt{alpaca-13b}, \texttt{chatglm-6b}, \texttt{claude-instant-v1}, \texttt{claude-v1}, \texttt{dolly-v2-12b}, \texttt{fastchat-t5-3b}, \texttt{gpt-3.5-turbo}, \texttt{gpt-4}, \texttt{gpt4all-13b-snoozy}, \texttt{guanaco-33b}, \texttt{koala-13b}, \texttt{llama-13b}, \texttt{mpt-7b-chat}, \texttt{oasst-pythia-12b}, \texttt{palm-2}, \texttt{stablelm-tuned-alpha-7b}, \texttt{vicuna-13b}, \texttt{vicuna-7b}, \texttt{wizardlm-13b} \\
\addlinespace[0.4ex]
MT-Bench ($N=6$) &
\texttt{alpaca-13b}, \texttt{claude-v1}, \texttt{gpt-3.5-turbo}, \texttt{gpt-4}, \texttt{llama-13b}, \texttt{vicuna-13b-v1.2} \\
\addlinespace[0.4ex]
PandaLM ($N=5$) &
\texttt{bloom-7b}, \texttt{cerebras-gpt-6.7B}, \texttt{llama-7b}, \texttt{opt-7b}, \texttt{pythia-6.9b} \\
\bottomrule
\end{tabularx}
\end{table}

Chatbot Arena is the largest and sparsest benchmark, whereas MT-Bench and PandaLM have complete pair coverage but only $6$ and $5$ candidate models (\Cref{tab:real-items}); PandaLM's original display order is one-sided.
Each record contributes up to $42$ LLM judgments but at most one human label, motivating adaptive weighting when within-record judgments are dependent.

\subsection{Protocol and estimators}\label{app:subsec:real-protocol}

We split records into train/test pools ($70/30\%$ for Chatbot Arena and $50/50\%$ otherwise), keeping both display orders of each record together.
Human calibration samples are drawn from decisive train labels, while LLM samples use tie-excluded train judgments subject to the specified judge panel and display design.
We report excess held-out human log loss relative to a test-pool BTL fit and Kendall's $\tau$ against its ranking.
All curves show means over up to $50$ record splits with $\pm1.96$ Monte Carlo standard errors.

We use the estimators in \Cref{sec:experiments} with $r=1$; \methDIALW{} is retained in \Cref{fig:real-main} but omitted from the remaining curve figures for readability.
Additional fixed-weight, oracle-weight, and rank-selection diagnostics use the same GACV grid and finite-fit rules as \Cref{app:lambda-selection}.
Judges with no retained judgments after subsampling are removed for that replication.
A bound-doubling sensitivity check gives essentially unchanged results except in the sparsest LLM regimes, where the LLM-only fit is numerically unresolved at either bound; at the smallest LLM budgets and judge panels we therefore read the curves as trends rather than as exact magnitudes or orderings.
Replications in which no candidate passes the guards of \Cref{app:lambda-selection} have no admissible fit and are reported as failures rather than averaged, which affects at most three of the $50$ splits in any reported cell.

\subsection{Position robustness evaluations and diagnostics}\label{app:subsec:real-position}

\begin{figure}[!t]
    \centering
    \includegraphics[width=\linewidth]{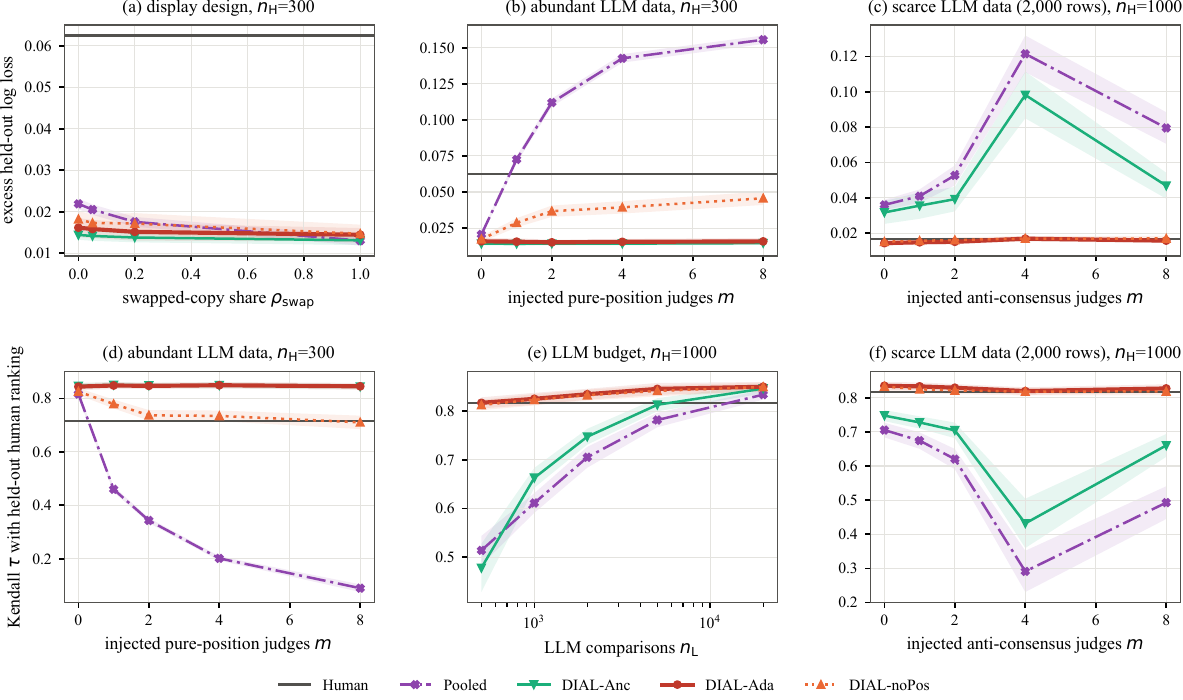}
    \caption{Real-data robustness sweeps on Chatbot Arena using five order-sensitive judges.
    Top: excess held-out log loss under (a) display imbalance, (b) injected pure-position judges, and (c) injected anti-consensus judges with scarce LLM data.
    Bottom: Kendall's $\tau$ under (d) the pure-position perturbation, (e) a balanced LLM-budget sweep, and (f) the anti-consensus perturbation.
    Means over up to $50$ record splits with $\pm1.96$ Monte Carlo standard errors.}
    \label{fig:real-curves}
\end{figure}

\begin{figure}[!t]
    \centering
    \includegraphics[width=0.8\linewidth]{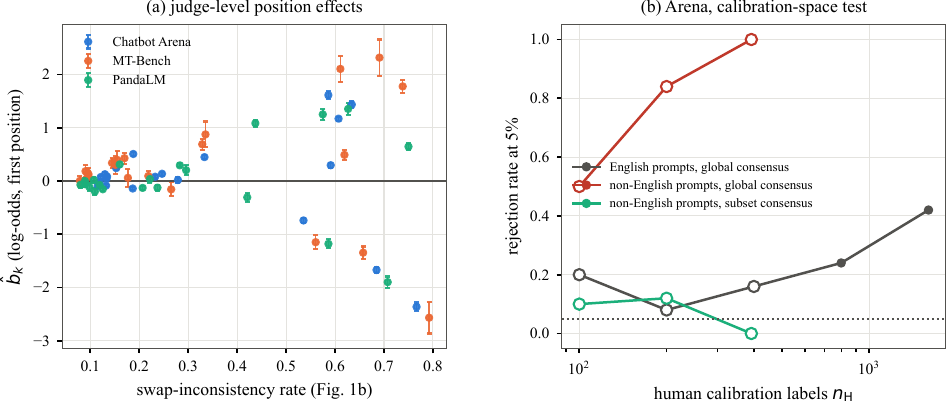}
    \caption{Additional real-data diagnostics.
    (a) LLM-only fitted order effects of the $21$ judges, with cell-clustered $95\%$ intervals, against swap-inconsistency rate.
    (b) Rejection rate of the calibration-space likelihood-ratio test of $s_{\human}=\alpha\mu$ on Chatbot Arena; hollow markers denote budgets at which the unrestricted human-only MLE failed in some replications, and at the two smallest budgets it failed in most, so those points are not interpreted.}
    \label{fig:real-diagnostics}
\end{figure}

\paragraph{Display imbalance and pure-position judges.}
The five-judge panel contains the judges with the largest swap-inconsistency on Chatbot Arena and is fixed across datasets, with $n_{\human}=300,80,60$.
We retain every canonical-first judgment and a fraction $\rho_{\mathrm{swap}}$ of swapped copies, with $\rho_{\mathrm{swap}}=1$ balanced and $\rho_{\mathrm{swap}}=0.05$ used for the one-sided design in \Cref{fig:real-main}; \Cref{fig:real-curves}(a) sweeps $\rho_{\mathrm{swap}}$ to zero.
For the injected perturbation, each pure-position judge selects the first-displayed response with probability $0.9$, independent of item identity.
Order-aware DIAL remains stable as display imbalance increases and pure-position judges are added, whereas the order-agnostic baselines deteriorate (\Cref{fig:real-main,fig:real-curves}(a,b,d)).

\paragraph{Robustness to imperfect LLM consensus.}
For the LLM-budget experiments, the balanced five-judge panel is subsampled to $n_{\llm}$ rows at fixed $n_{\human}=1000,80,60$ on Chatbot Arena, MT-Bench, and PandaLM, respectively.
Injected anti-consensus judges reproduce the flipped verdicts of a randomly selected panel judge; \Cref{fig:real-curves}(c,f) uses the scarce-LLM setting.
Under this corruption, \methDIALmu{} shifts toward the human evidence and is substantially more stable than fixed anchoring, although behavior at the strongest corruption is non-monotone.

\paragraph{Judge-specific position effects.}
\Cref{fig:real-diagnostics}(a) compares the LLM-only $\widehat b_k$ from the full balanced data with each judge's swap-inconsistency rate, using cell-clustered intervals from \Cref{app:clustered}.
The fitted effects vary substantially across judges in both magnitude and sign and broadly track empirical order sensitivity.

\subsection{Human-label efficiency and adaptive calibration}\label{app:subsec:real-efficiency}

\paragraph{Calibration-space diagnostic.}
Chatbot Arena records are split by recorded prompt language ($6{,}910$ English, $828$ other), with $30\%$ of each subset held out for testing.
We apply \Cref{prop:calibration-test} at the $5\%$ level using either the global rank-one consensus or a consensus refitted within the subset.
The statistic is defined only where the unrestricted human BTL fit is finite, so rejection rates are read at the budgets where that fit exists in all but a few splits: at the full non-English calibration pool it is finite in $48$ of the $50$ splits, and restricting to those splits leaves the rejection rates unchanged at $1.00$ for the global consensus and $0.00$ for the refitted one, whereas at $100$ and $200$ labels it is finite in only $4\%$ and $38\%$ of splits and those budgets are not interpreted.
\Cref{fig:real-diagnostics}(b) shows a pronounced mismatch of the global consensus for non-English prompts that largely disappears after subset-specific refitting, which also improves held-out ranking and log loss.

\paragraph{Human-label efficiency with abundant LLM data.}
With all $21$ judges, \methConsensus{} and \methDIALmu{} remain strong at small human-label budgets, where \methHuman{} often has no finite MLE (\Cref{fig:real-efficiency}, top).
When LLM evidence is abundant and balanced, adaptive updating offers little gain over anchoring; across ranks, the oracle calibration gains are negligible, providing little evidence that disagreement directions improve the pooled human target.

\paragraph{Adaptive calibration with scarce LLM data.}
When the five-judge panel is subsampled to small $n_{\llm}$, the LLM-only consensus becomes noisy and \methConsensus{} degrades, whereas \methDIALmu{} uses the human comparisons to update the consensus and remains competitive with the better endpoint across datasets (\Cref{fig:real-efficiency}, bottom).
As $n_{\llm}$ grows, the selected fit moves toward anchoring and the estimators converge.

\begin{figure}[!ht]
    \centering
    \includegraphics[width=\linewidth]{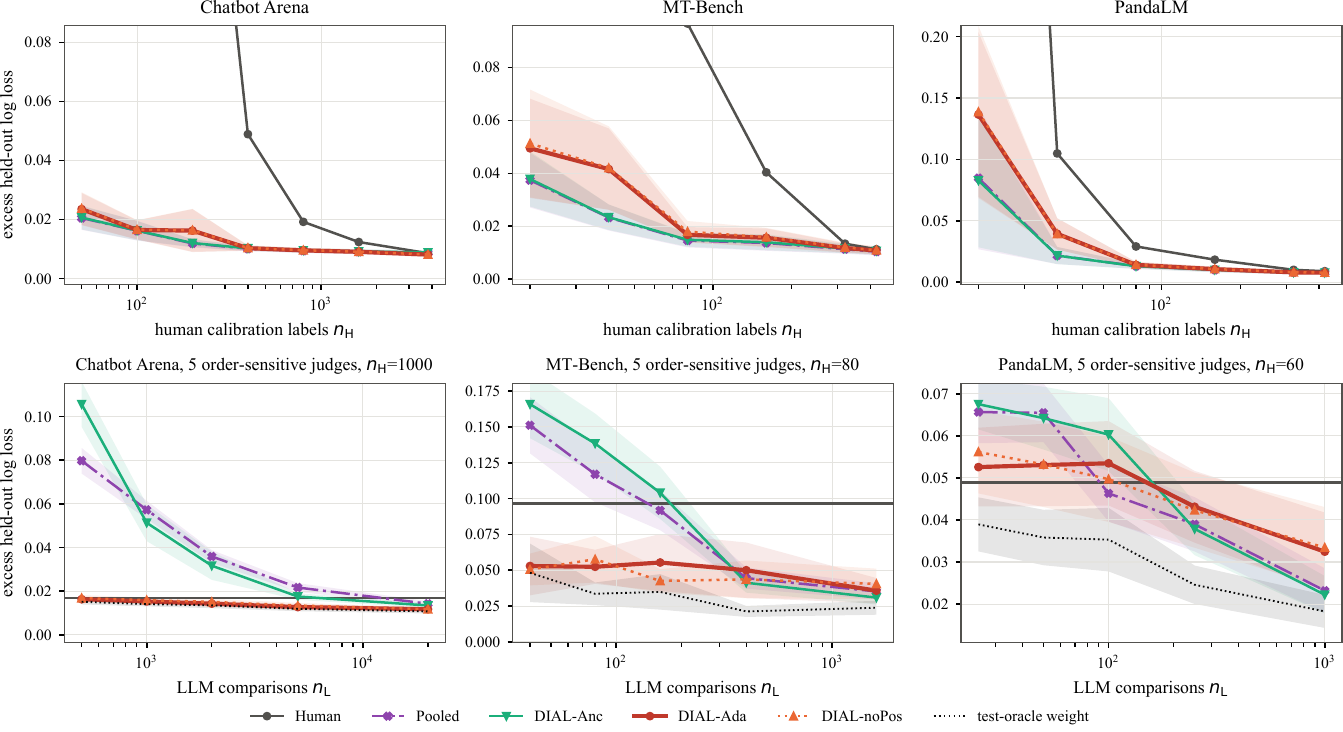}
    \caption{Human-label efficiency and LLM-budget sensitivity.
    Top: excess held-out human log loss against the human budget $n_{\human}$ with the as-collected balanced judgments of all $21$ judges; \methHuman{} is off the chart at the smallest budgets, where its MLE does not exist.
    Bottom: the five order-sensitive judges' balanced judgments subsampled to $n_{\llm}$ rows at a fixed calibration sample ($n_{\human}=1{,}000$, $80$, $60$); the dotted line is the candidate weight minimizing the test log loss.
    Means over up to $50$ record splits with $\pm1.96$ Monte Carlo standard errors.}
    \label{fig:real-efficiency}
\end{figure}

\paragraph{Judge-panel and budget sensitivity.}

\Cref{fig:panel-human-budget} compares six-judge open-weight panels with the full $21$-judge panel across human-label budgets.
The small panel remains close to the larger panels on Chatbot Arena and PandaLM, with MT-Bench the main exception; at the smallest human budgets, \methDIALmu{} is less stable than \methConsensus{} on the small panel.
Thus the human-label gains of DIAL do not require the full judge panel, although weaker panels leave less margin for adaptive updating.

\begin{figure}[!ht]
    \centering
    \includegraphics[width=\linewidth]{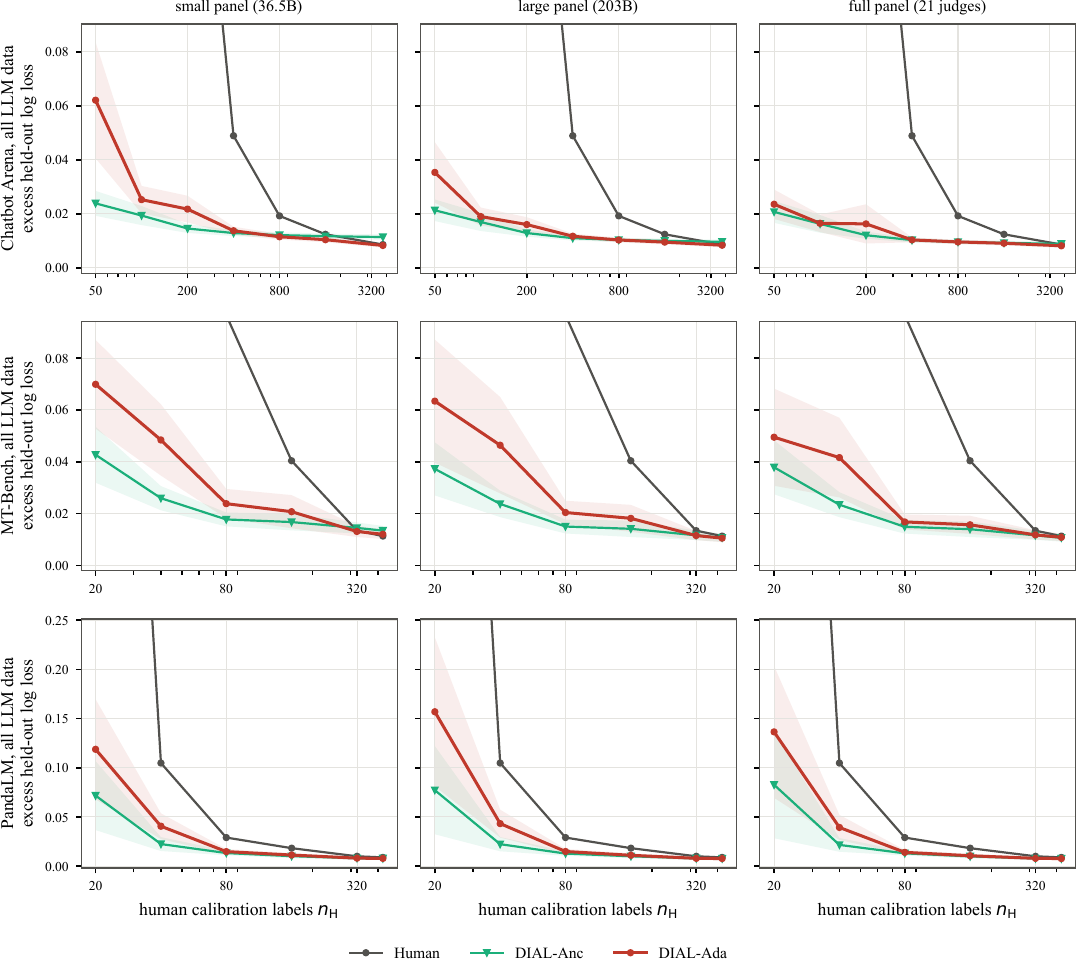}
    \caption{Excess held-out human log loss against the human budget $n_{\human}$, as in \Cref{fig:real-efficiency} (top), on the small, large, and full judge panels (columns) across Chatbot Arena, MT-Bench, and PandaLM (rows).
    The small and large panels each contain six open-weight judges with disclosed parameter counts $\le9$B and $\ge20$B per judge, respectively ($36.5$B versus $203$B total); full is the $21$-judge panel.
    All curves use up to the same $50$ record splits, with $\pm1.96$ Monte Carlo standard errors; the $y$-axis is shared within each row.}
    \label{fig:panel-human-budget}
\end{figure}

\begin{figure}[p]
    \centering
    \includegraphics[width=\linewidth]{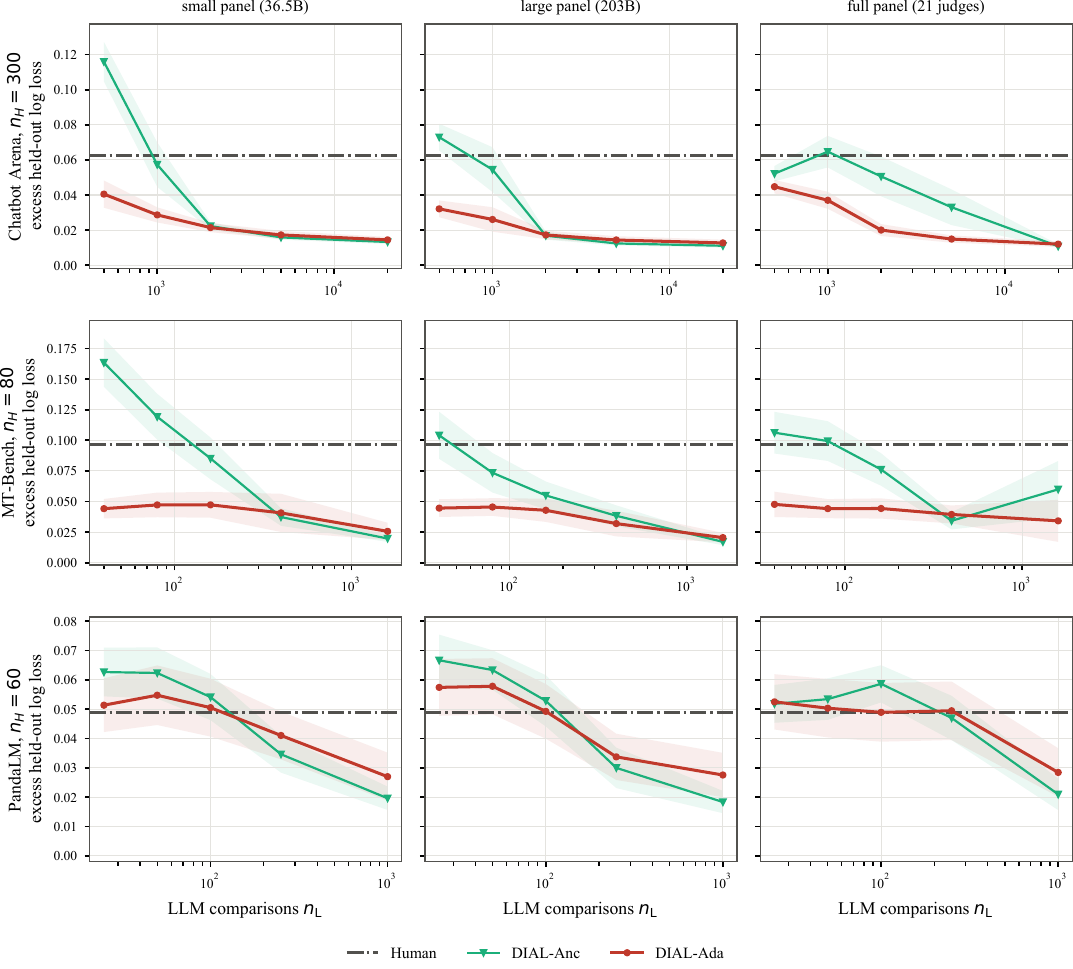}
    \caption{Complementary LLM-budget sweep for the small, large, and full judge panels of \Cref{fig:panel-human-budget}.
    The human calibration budget is fixed at $n_{\human}=300$, $80$, and $60$ on Chatbot Arena, MT-Bench, and PandaLM, respectively, while balanced LLM training judgments are subsampled uniformly to $n_{\llm}$ rows.
    Curves compare \methDIALmu{}, at the GACV-selected weight, with \methConsensus{} and \methHuman{}.
    All curves use up to the same $50$ record splits, with $\pm1.96$ Monte Carlo standard errors; the $y$-axis is shared within each row.}
    \label{fig:panel-llm-budget}
\end{figure}

\paragraph{Adaptive calibration across LLM budgets.}
With abundant LLM judgments and limited human supervision, \methConsensus{} is competitive because the LLM consensus is well estimated.
As the LLM budget decreases, \methDIALmu{} increasingly improves by using the human comparisons to update the noisier consensus estimate (\Cref{fig:panel-llm-budget}).
This pattern holds across judge panels and complements the human-budget sweep in \Cref{fig:panel-human-budget}.

\end{document}